\documentclass{article} 
\usepackage{beyond_calibration,times}

\usepackage[utf8]{inputenc} 
\usepackage[T1]{fontenc}    
\usepackage{booktabs}       
\usepackage{amsfonts}       
\usepackage{nicefrac}       
\usepackage{microtype}      
\usepackage{xcolor}         
\usepackage{amsthm} 
\usepackage{amsmath}
\usepackage{bbold}
\usepackage{dsfont} 
\usepackage{graphicx}
\graphicspath{ {./img/} }
\usepackage{float}
\usepackage{wrapfig}
\usepackage{subcaption}
\usepackage{thm-restate}        
\usepackage{enumitem}
\usepackage{stmaryrd}
\usepackage{listofitems}
\usepackage[toc,page,header]{appendix}
\usepackage{minitoc}
\def\mycitecolor{green!50!black}
\def\mylinkcolor{red!60!black}
\usepackage[colorlinks=true, linkcolor=\mylinkcolor, citecolor=\mycitecolor]{hyperref}       

\usepackage{url}            
\usepackage{soul}

\usepackage[textsize=tiny]{todonotes}
\newcommand{\RR}{\mathbb{R}}
\newcommand{\E}{\mathbb{E}}

\usepackage[mathscr]{eucal}
\newcommand{\Ccal}{\mathscr{R}}
\newcommand{\Ical}{\mathcal{I}}
\newcommand{\Lcal}{\mathcal{L}}

\newcommand{\Rcal}{\mathcal{R}}
\newcommand{\Scal}{\mathcal{S}}

\newcommand{\cond}{\;\middle\vert\;}

\newcommand{\br}[1]{\left(#1\right)}
\newcommand{\sqb}[1]{\left[#1\right]}
\newcommand{\cbr}[1]{\left\{#1\right\}}

\newcommand{\abs}[1]{\left\vert\, #1 \,\right\vert}

\newcommand{\pr}[1]{\left( #1 \right)}

\newcommand{\bra}[1]{\left[ #1 \right]}
\newcommand{\N}{\mathbb{N}}
\newcommand{\R}{\mathbb{R}}

\newcommand{\Bc}{\mathcal{B}}

\newcommand{\Xc}{\mathcal{X}}

\newcommand{\PPG}[2]{P\!\pr{#1\given#2}}
\newcommand{\1}[1]{\mathds{1}_{#1}}

\newcommand{\given}[1][]{#1\vert}

\newcommand{\eg}{\emph{e.g.}}
\newcommand{\ie}{\emph{i.e.}}

\newcommand{\inlinequote}[1]{``\textit{#1}''}

\newcommand{\GL}{\mathrm{GL}}
\newcommand{\glexp}{\GL_{\textit{explained}}}
\newcommand{\glexphat}{\widehat\GL_{\textit{explained}}}
\newcommand{\glres}{\GL_{\textit{residual}}}
\newcommand{\GLind}{\GL_{\textit{induced}}}
\newcommand{\GLindhat}{\widehat\GL_{\textit{induced}}}
\newcommand{\glbias}{\GL_{\textit{bias}}}
\newcommand{\glbiashat}{\widehat{\GL}_{\textit{bias}}}

\newcommand{\glpluginhat}{\widehat{\GL}_{\textit{plugin}}}
\newcommand{\gllb}{\GL_{\mathrm{LB}}}
\newcommand{\gllbhat}{\widehat{\GL}_{\mathrm{LB}}}

\newcommand{\CL}{\mathrm{CL}}
\newcommand{\CLind}{\CL_{\textit{induced}}}

\newcommand{\clhat}{\widehat{\CL}}

\newcommand{\esp}[1]{\mathbb{E}\!\bra{#1}}
\newcommand{\espk}[2]{\esp{#1\,\middle|\,#2}}
\newcommand{\KL}{D_{\mathrm{KL}}}
\newcommand{\kl}[2]{\KL\!\pr{#1\,\middle\Vert\,#2}}
\newcommand{\Var}[1][]{\mathbb{V}_{\!#1}}
\newcommand{\var}[2][]{\mathbb{V}_{\!#1}\!\bra{#2}}
\newcommand{\vark}[3][]{\var[#1]{#2\,\middle|\,#3}}
\newcommand{\covk}[3]{\mathrm{Cov}\!\bra{#1,#2\,\middle|\,#3}}
\newcommand{\Sc}{\mathcal{S}}

\theoremstyle{plain}
\newtheorem{definition}{Definition}[section]
\newtheorem{proposition}{Proposition}[section]
\newtheorem{theorem}{Theorem}[section]
\newtheorem{corollary}{Corollary}[section]

\newtheorem{lemma}{Lemma}[section]
\newtheorem{remark}{Remark}

\newsavebox{\overlongequation}

\usepackage{etoolbox}
\makeatletter
\patchcmd{\hyper@makecurrent}{%
    \ifx\Hy@param\Hy@chapterstring
        \let\Hy@param\Hy@chapapp
    \fi
}{%
    \iftoggle{inappendix}{
        \@checkappendixparam{chapter}%
        \@checkappendixparam{section}%
        \@checkappendixparam{subsection}%
        \@checkappendixparam{subsubsection}%
        \@checkappendixparam{paragraph}%
        \@checkappendixparam{subparagraph}%
    }{}%
}{}{\errmessage{failed to patch}}

\newcommand*{\@checkappendixparam}[1]{%
    \def\@checkappendixparamtmp{#1}%
    \ifx\Hy@param\@checkappendixparamtmp
        \let\Hy@param\Hy@appendixstring
    \fi
}
\makeatletter

\makeatletter
\newcommand*{\centerfloat}{%
  \parindent \z@
  \leftskip \z@ \@plus 1fil \@minus \textwidth
  \rightskip\leftskip
  \parfillskip \z@skip}
\makeatother

\newtoggle{inappendix}
\togglefalse{inappendix}

\apptocmd{\appendix}{\toggletrue{inappendix}}{}{\errmessage{failed to patch}}
\apptocmd{\subappendices}{\toggletrue{inappendix}}{}{\errmessage{failed to patch}}

\def\mytitle{Beyond calibration: estimating the
grouping loss of modern neural networks}
\title{\mytitle}

\author{Alexandre Perez-Lebel, \, Marine Le Morvan, \, Gaël Varoquaux\\
Soda project team, Inria Saclay, Palaiseau, France\\
\texttt{\{alexandre.perez,marine.le-morvan,gael.varoquaux\}@inria.fr}
}

\iclrfinalcopy 

\begin{document}

\doparttoc  
\faketableofcontents 

\maketitle

\begin{abstract}
The ability to ensure that a classifier gives reliable confidence scores is essential to ensure informed decision-making. To this end, recent work has focused on miscalibration, \ie{}, the over or under confidence of model scores. Yet calibration is not enough: even a perfectly calibrated classifier with the best possible accuracy can have confidence scores that are far from the true posterior probabilities. This is due to the grouping loss, created by samples with the same confidence scores but different true posterior probabilities. \emph{Proper scoring rule} theory shows that given the calibration loss, the missing piece to characterize individual errors is the \emph{grouping loss}. While there are many estimators of the calibration loss, none exists for the grouping loss in standard settings.
Here, we propose an estimator to approximate the grouping loss.
We show that modern neural network architectures in vision and NLP exhibit grouping loss, notably in distribution shifts settings, which highlights the importance of pre-production validation.
\end{abstract}

\section{Introduction}
\label{sec:introduction}


Validating the compliance of a model to a predefined set of specifications is important to control operational risks related to performance but also trustworthiness, fairness or robustness to varying operating conditions. It often requires that probability estimates capture the actual uncertainty of the prediction, \ie{} are close to the true posterior probabilities. Indeed, many situations call for probability estimates rather than just a discriminant classifier. Probability estimates are needed when the decision is left to a human decision maker, when the model needs to avoid making decisions if they are too uncertain, when the context of model deployment is unknown at training time, etc.

To evaluate probabilistic predictions, statistics and decision theory have put forward \emph{proper scoring rules} \citep{Dawid2014a,gneiting2007probabilistic}, such as the Brier or the log-loss. Strictly proper scoring rules are minimized when a model produces the true posterior probabilities, which make them a valuable tool for comparing models and selecting those with the best estimated probabilities \citep[\eg{}][]{Dawid2014}.
What they do not provide though is a means of validating whether the best estimated probabilities are good enough to be put into production, or whether further effort is needed to improve the model. Indeed, proper scores compound the \emph{irreducible loss} --due to the inherent randomness of a problem, \ie{} the aleatoric uncertainty-- and the \emph{epistemic loss} --which measures how far a model is from the best possible one. For example, a classifier with a Brier score of 0.15 could have optimal estimated probabilities (irreducible loss close to 0.15) or poor ones (irreducible loss close to 0).\looseness=-1

Calibration errors are another tool to evaluate probabilistic predictions, and measuring them is an active research topic in the machine learning community. \citep{ Kumar2019verified,minderer2021revisiting,roelofs2022mitigating}. The calibration error is in fact a component of proper scoring rules \citep{Brocker2009Reliability,Kull2015Novel}: it measures whether among all samples to which a calibrated classifier gave the same confidence score, on average, a fraction equal to the confidence score is positive. Importantly, the calibration error can be evaluated efficiently as it does not require access to the ground truth probabilities, but solely to their calibrated version. Calibration is however an incomplete characterization of predictive uncertainty. It measures an aggregated error that is blind to potential individual errors compensating each other. For example, among a group of individuals to which a calibrated cancer-risk classifier assigns a probability of 0.6, a fraction of 60\% actually has cancer. But a subgroup of them could be composed of 100\% cancer patients while another would only contain 20\% of cancer patients.\looseness=-1





%
In general, estimating the true posterior probabilities or obtaining individual guarantees is impossible \citep{Vovk2005a, Barber2020}. Recent works have thus attempted to refine guarantees on uncertainty estimates at an intermediary subgroup level. In particular, \citet{Hebert-Johnson2018} has introduced the notion of multicalibration, generalizing the notion of calibration within groups studied in fairness \citep{Kleinberg2016} to every efficiently-identifiable subgroup. \citet{FoygelBarber2019, Barber2020} defines  subgroups-based coverage guarantees which lie in between the coarse marginal coverage and the impossible conditional coverage guarantees.
In a similar vein, we study the remaining term measuring the discrepancy between the calibrated probabilities and the unknown true posterior probabilities
\citep{Kull2015Novel}, \ie{} the
grouping loss, for which no estimation procedure exists to date. In particular:
%
\begin{itemize}[topsep=0pt,itemsep=0ex,partopsep=0ex,parsep=0ex,leftmargin=2ex]
    \item We provide a new decomposition of the grouping loss into explained and residual components, together with a debiased estimator of the explained component as a lower bound (\autoref{sec:characterizing}).
    \item We demonstrate on simulations that the proposed estimator can provide tight lower-bounds on the grouping loss (\autoref{sec:exp_simu}).
    \item We evidence for the first time the presence of grouping loss on \emph{pre-trained} vision and language architectures, notably in distribution shifts settings (\autoref{sec:exp_nn}).
\end{itemize}

\section{Calibration is not enough}
\label{sec:cal_not_enough}

Calibration can be understood with a broad conceptual meaning of alignment of measures and statistical estimates \citep{osborne1991statistical}. However, in the context of decision theory or classifiers, the following definitions are used \citep{foster1998asymptotic,gneiting2007probabilistic,Kull2015Novel}:
\begin{description}[topsep=0pt,itemsep=0ex,partopsep=0ex,parsep=0ex]
  \item[\textbf{True posterior probabilities:}] $Q := \PPG{Y = 1}{X}$,
  \item[\textbf{Confidence scores:}] $S := f(X)$ score output by a classifier,
\item[\textbf{Calibrated scores:}] $C := \PPG{Y = 1}{S} = \espk{Q}{S}$, average true posterior probabilities for a score $S$.\looseness=-1
\end{description}

\paragraph{Confusion about calibration} A common confusion is to mistake confidence scores of a calibrated classifier with true posterior probabilities and think that a calibrated classifier outputs true posterior probabilities, which is false. We identified three main sources of confusion in the literature --see \autoref{sec:confusing_statements} for specific quotes. First, the vocabulary used sometimes leaves room for ambiguity, \eg{}, \emph{posterior probabilities} may refer to confidence scores or to the true posterior probabilities without further specifications. Second, plain-English definitions of calibration are sometimes incorrect, defining calibrated scores as the true posterior probabilities. Lastly, even when everything is correctly defined, it is sometimes implicitly supposed that true posterior probabilities are close to the calibrated scores. While it may be true in some cases, equating the two induces misconceptions.

\paragraph{Calibration with good accuracy does not imply good individual confidences}
\label{sec:counter-ex}
It is tempting to think that a calibrated classifier with optimal accuracy should provide confidence scores
\setlength\intextsep{0pt}
\begin{wrapfigure}{r}{0.35\textwidth}\vspace*{2mm}
  \centering
  \includegraphics[width=\linewidth]{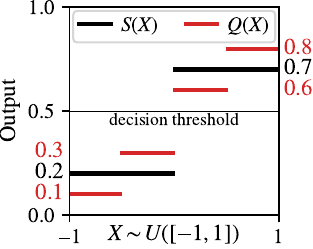}
  \vspace*{-5mm}
  \caption{A calibrated binary classifier with optimal accuracy and confidence scores $S(X)$ everywhere different from the true posterior probabilities $Q(X)$.
  \label{fig:example}}%
\end{wrapfigure}
close to the true posterior probabilities.
However, caution is necessary: \autoref{fig:example} shows a simple counterexample. The classifier presented gives an optimal accuracy as its confidence scores are always on the same side of the decision threshold as the true posterior probabilities. It is moreover calibrated, as for a given score $s$ (either 0.2 or 0.7 here), the expectation of $Q$ over the region where the confidence score is $s$ is actually equal to $s$. Yet, the confidence scores are not equal to $Q$ as $Q$ displays variance over regions of constant scores. This variance can be made as large as desired as long as both $Q$ and $S$ stay on the same side of the decision threshold to preserve accuracy. The flaws of a perfectly calibrated classifier that always predicts the same score are typically due to variations of the true posterior probabilities over constant confidence scores. As we formalize below, such variations are captured by the grouping loss \citep{Kull2015Novel}.
\autoref{sec:app:examples} provides a more realistic variant of this example based on the output of a neural network.

\section{Theoretical background}
\label{sec:theory}
\paragraph{Notations} Let $(X, Y) \in \mathcal X \times \mathcal Y$ be
jointly distributed random variables describing the features and labels
of a K-class classification task. Let $e_k$ be the one-hot vector of size
$K$ with its $k^{th}$ entry equal to one. The label space $\mathcal Y =
\cbr{e_1, \dots, e_K}$ is the set of all one-hot vectors of size $K$.
We assume that labels are drawn according to the true posterior distribution $Q = (Q_1, \dots, Q_K) \in \Delta_K$ where $Q_k:= P(Y=e_k|X)$ and $\Delta_K$ is the probability simplex $\Delta_K = \cbr{(p_1, \dots, p_K) \in \sqb{0, 1}^K: \sum_k p_k = 1}$. We consider a probabilistic classifier $f$ giving scores $S = f(X)$ with $S=\br{S_1, \dots, S_K} \in \Delta_K$. Note that $S$ and $Q$ are random vectors since they depend on $X$. This section introduces the formal definition of the grouping loss, which uses the concepts of calibrated scores as well as scoring rules.

\subsection{Calibration in a multi-class setting}%
In  multi-class settings various definitions of calibration give
different trade offs between control stringency and practical utility
\citep{Vaicenavicius2019Evaluating,Kull2019Beyond}.
The strongest definition controls the proportion of positives for groups
of samples with the same \emph{vector} of scores $S$.


\begin{definition}
    \label{def1}
    A probabilistic classifier giving scores $s = (s_1, \dots, s_k)$ is \textbf{jointly calibrated} if
    among all instances getting score $s$, the class probabilities are actually equal to $s$:
    \begin{flalign}
	\text{Calibration} &&
        \PPG{Y = e_k}{S=s} = s_k \quad \text{for k = 1, \dots, K}. &&
    \end{flalign}
\end{definition}

The score $S$ being a vector of size $K$ the number of classes, estimating the probability of $Y$ conditioned on $S$ is a difficult task that requires many samples. A weaker notion of multi-class calibration, introduced in \cite{Zadrozny2002Transforming}, requires calibration for each class marginally:

\begin{definition}
    \label{def2}
    A probabilistic classifier giving scores $s = (s_1, \dots, s_k)$ is \textbf{classwise-calibrated} if
    among all instances getting score $s_k$, the probability of class $k$ is actually equal to $s_k$:
    \begin{flalign}
	\text{Classwise calibration} &&
        \PPG{Y = e_k}{S_k=s_k} = s_k \quad \text{for k = 1, \dots, K}. &&
    \end{flalign}
\end{definition}

As classwise calibration can still be challenging to estimate when the
number of samples per class is too small, an even weaker definition is used in the machine learning community \citep{Guo2017Calibration}.

\begin{definition}
    \label{def3}
    A probabilistic classifier giving scores $s = (s_1, \dots, s_k)$ is \textbf{top-label-calibrated} if
    among all instances for which the confidence score of the predicted class is $s$, the probability that the predicted class is the correct one is $s$:
    \begin{flalign}
	\text{Top-label calibration} &&
        \PPG{Y = e_{\mathrm{argmax(s)}}}{\max(S) = s} = s. &&
    \end{flalign}
\end{definition}
Top-label calibration simplifies the problem by reducing it to a binary problem. However, it has an important limitation \citep{Vaicenavicius2019Evaluating}: as it only accounts for the confidence of the predicted class, it does not tell whether smaller probabilities are also calibrated.

\subsection{Proper scoring rules and their decomposition}


\paragraph{Scoring rules}
Scoring rules measure how well an estimated probability vector $S$ explains the observed labels $Y$. 
The two most widely used scoring rules are the log-loss and Brier score:
\begin{flalign}
\text{Log-loss}:\quad\phi^{\mathrm{LL}}(S, Y) :=- \sum_{k=1}^{K} Y_k \log S_k
&&
\text{Brier score}:\quad \phi^{\mathrm{BS}}(S, Y) :=\sum_{k=1}^{K}\left(S_{k}-Y_{k}\right)^{2}
\quad
\end{flalign}
Scoring rules are defined per sample, and the score over a dataset
is obtained by averaging over samples. More generally, the \emph{expected score} for rule $\phi$ of the estimated probability vector $S$ with regards to the class label $Y$ drawn according to $Q$ is given by $\mathrm{s_{\phi}}(S, Q)\,:=\,\mathbb{E}_{Y \sim Q} \sqb{\phi(S, Y)}$. Proper scoring rules decompositions have been introduced in terms of their divergences rather than their scores. The \emph{divergence} between probability vectors $S$ and $Q$ is then defined as $\mathrm{d_{\phi}}(S, Q):=\mathrm{s_{\phi}}(S, Q)-\mathrm{s_{\phi}}(Q, Q)$.
The divergences for the Brier score and the log-loss read:
\begin{flalign}
\text{Log-loss}:\quad d^{\mathrm{LL}}(S, Q) := \sum_{k=1}^{K} Q_k \log \frac{Q_k}{S_k}
&&
\text{Brier score}:\quad d^{\mathrm{BS}}(S, Q) :=\sum_{k=1}^{K}\left(S_{k}-Q_{k}\right)^{2}
\quad
\end{flalign}
Minimizing the Brier score \emph{in expectation} thus amounts to minimizing the mean squared error between $S$ and the unknown $Q$. A scoring rule is said
\emph{strictly proper} if its divergence is non-negative and
$\mathrm{d_{\phi}}(S, Q) = 0$ implies $S = Q$. Both the log-loss and Brier score are strictly proper.

\paragraph{Scoring rules decomposition} Let $C$ be the calibrated scores in the sense of \autoref{def1}, the strongest one \ie{}, $C_k = P(Y=e_k|S=s)$ for $k=1, \dots, K$. The divergence of strictly proper scoring rules can be decomposed as \citep{Kull2015Novel}:
\begin{equation}
    \label{eq:decomp}
    \E \sqb{d_{\phi}(S, Y)} \;=\; \underbrace{\E \sqb{d_{\phi}(S, C)}}_{\text{Calibration: CL}} \,+\, \underbrace{\E \sqb{d_{\phi}(C, Q)}}_{\text{Grouping: GL}} \,+\, \underbrace{\E \sqb{d_{\phi}(Q, Y)}}_{\text{Irreducible: IL}}
\end{equation}
where the expectation is taken over $Y \sim Q$ and $X$. CL is the
\emph{calibration loss}. IL is the \emph{irreducible loss} which stems from
the fact that one point may not have a deterministic label, making
perfect predictions impossible. GL is the \emph{grouping loss}.
Intuitively, while the calibration loss captures the
deviation of the expected score in a bin vs the expected posterior
probabilities, the grouping loss captures variations of the true
posterior probabilities around their expectation. Together calibration and grouping form the epistemic loss, capturing intrinsic the randomness of the best possible predictor. The scoring rule decomposition~(\ref{eq:decomp}) holds for top-label calibration (\autoref{def3}) as it can be reduced to a binary problem. In the case of classwise calibration, the extension is not straightforward in the general case but we prove in \autoref{prop:decomposition} that it holds for the Brier score and the log-loss.

\section{Characterization of the grouping loss}%

In this section, we focus for simplicity on all settings where the
calibrated scores can be expressed as $C_k =\E \sqb{Y_k | S}$, which
includes binary classification as well as the multi-class setting with
joint or top-label calibration. For classwise calibration,
\autoref{sec:classwise-extension} shows that all the results presented in this section also hold for the Brier score and log-loss.


\label{sec:characterizing}%

\subsection{Rewriting the grouping loss as a form of variance}%
To shed light on the grouping loss, we rewrite it using $f$-variances:

\begin{definition}[$f$-variance]
  Let $U,V : \Omega \to \R^d$ be two random variables defined on the same probability space, and function $f: \R^d \to \R$. Assuming the required expectations exist, the $f$-variance of $U$ given $V$ is:
  \setlength{\abovedisplayskip}{1mm}
  \setlength{\belowdisplayskip}{0mm}
  \begin{equation*}
    \vark[f]{U}{V} := \espk{f(U)}{V} - f(\espk{U}{V}).
  \end{equation*}
\end{definition}
The $f$-variance corresponds to the Jensen gap. It is positive by Jensen's
inequality when $f$ is convex. Beyond positivity, it can be seen as an
extension of the variance as using the square function for $f$ recovers
the traditional notion of variance. 

\begin{restatable}[The grouping loss as an $h$-variance]{lemma}{evariance}
   \label{prop:e-variance}
   Let $h$ be the negative entropy of the scoring rule $\phi$, \ie{} $h : p \mapsto -s_{\phi}(p, p)$.
   The grouping loss $\GL$ of the classifier $S$ with calibrated scores $C = \espk{Q}{S}$ and scoring rule $\phi$ writes:
   \begin{equation}
    \underbrace{\esp{d_{\phi}(C,Q)}}_{\GL(S)}\;=\;\esp{\vark[h]{Q}{S}}
    \label{eq:gl-var}
   \end{equation}
\end{restatable}
\vspace{-0.5em}
The proof is given in \autoref{prop:e-variance:proof}. In other words, the grouping loss associated to a scoring rule $\phi$ is an $h$-variance of the true posterior probability $Q$ around the average scores $C$ on groups of same level confidence $S$ (\autoref{eq:gl-var}). In particular for the Brier score, the $h$-variance is a classical variance. It measures discrepancy between $Q$ and $C$ with a squared norm: $\vark[h]{Q}{S} = \espk{\|Q - C\|^2}{S}$. For the log-loss, it is a Kullback-Leibler divergence: $\vark[h]{Q}{S} = \espk{\kl{Q}{C}}{S}$. These expressions highlight two challenges in estimating the grouping
loss. First, it relies on the true posterior probabilities $Q$, which we
do not have access to. Second, it involves a conditioning on the
confidence scores $S$, which are difficult to estimate for continuous scores. 

\subsection{Grouping loss decomposition and lower-bound}


As an $h$-variance of $Q$ given $S$, evaluating the grouping loss
requires access to $Q(X)$ for any point $X$. Unfortunately  $Q(X)$ is difficult to estimate, except in
special settings -- \eg{} multiple labels per sample as in
\citet{Mimori2021Diagnostic}. In fact, the scores $S$ of a classifier are
generally one's best estimate of $Q$, and the whole point of the grouping
loss is to quantify how far this best estimate is from the unknown oracle
$Q$.
We show that it is nevertheless possible to
estimate a lower bound on the grouping loss. On the level set where
a classifier score is $S$, it is indeed possible to estimate the average of $Q$ on regions of the feature space. Since by definition $Q$ is non-constant on the level set of a classifier with non-zero grouping loss, it allows to capture 
part of the grouping loss (\autoref{fig:clusters}). Intra-region variance remains uncaptured but can be reduced by choosing smarter and more numerous regions in the partition of the feature space. \autoref{th:GL_decomp} formalizes this intuition:

\begin{figure}[t]
  \setlength\intextsep{0pt}
  \vspace*{0.4cm}%
  \begin{minipage}{0.55\textwidth}
    \caption{\textbf{Intuition.}
    In the feature space $\mathcal{X}$, the level set of confidence $S=0.7$ displays $\espk{Q}{S} = 0.7$,
    which we expect from a calibrated classifier. However, a partition of the level set into 2 regions $\Ccal_1$ and $\Ccal_2$ reveals that $\espk{Q}{S,\Ccal_1}=0.6$ while $\espk{Q}{S,\Ccal_2}=0.8$, suggesting a high grouping loss. Intra-region variances $\vark[h]{Q}{S,\Ccal_1}$ and $\vark[h]{Q}{S,\Ccal_2}$ remain uncaptured.
    \label{fig:clusters}
    }
  \end{minipage}%
  \hfill%
  \begin{minipage}{0.38\textwidth}
    \vspace*{-0.3cm}
    \includegraphics[width=\linewidth]{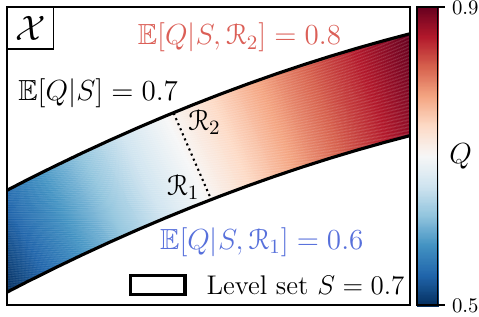}%
  \end{minipage}%
  \vspace*{-5mm}
\end{figure}

\begin{restatable}[Grouping loss decomposition]{theorem}{gldecomp}
    \label{th:GL_decomp}

   Let $\Ccal: \Xc \to \N$ be a partition of the feature space.
   It holds that:
   \begin{equation}
    \GL(S)
    =\;\underbrace{\esp{\vark[h]{\espk{Q}{S,\Ccal}}{S}}}_{\glexp(S)}
    \;+\;\underbrace{\esp{\vark[h]{Q}{S, \Ccal}}}_{\glres(S)}
    \label{eq:gl-exp-res}
    \end{equation}
    \begin{flalign}
      \text{Moreover if the scoring rule is proper, then:} &\qquad\qquad\GL(S) \geq \glexp(S) \geq 0.&
    \end{flalign}
\end{restatable}

\autoref{th:GL_decomp:proof} gives the proof by showing that the law
of total variance is also valid for the $h$-variance, which allows to
decompose the grouping loss into explained and residual terms. $\glexp$
quantifies the $h$-variance captured through the partition $\Ccal$, \ie{}
coarse-grained $h$-variance reflecting between-region variations of $Q$, while $\glres$ captures the remaining intra-region $h$-variance.
Due to the positivity of $\glres$, $\glexp$ is a lower-bound of
the grouping loss that ranges between $0$ and $\GL$ depending on how
much $h$-variance the partition captures. 
%
Importantly, while $\vark[h]{Q}{S, \Ccal}$ cannot be estimated because the oracle $Q$ is unknown, it is possible to estimate $\espk{Q}{S,\Ccal}$ and thus $\glexp$.


\subsection{Controlling the grouping loss induced by binning classifier
scores $S$}%

The grouping loss as well as $\glexp$ involve a conditioning on the confidence scores $S$,
which cannot be estimated by mere counting when the scores are continuous.
To overcome this difficulty, standard practice in calibration approximates the conditional expectation using a binning strategy: the classifier scores are binned into a finite number of values (\autoref{def:binclf}).


\begin{minipage}{\linewidth}
\begin{wrapfigure}[3]{r}{.25\linewidth}%
    \vspace*{2mm}
    \centerline{\includegraphics[width=.9\linewidth]{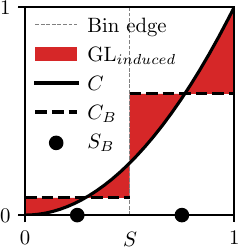}}%
    \vspace*{-2mm}
    \caption{\textbf{Binning inflates the grouping loss.} 
    \label{fig:binning}}%
\end{wrapfigure}
\begin{minipage}[t]{\linewidth}
\begin{definition}[Binned classifier]
   \label{def:binclf}
   Let $S : \Xc \to \Delta_K$ be a classifier. Let $\Bc := \{\Bc_j\}_{1\leq j \leq J}$ be a partition of $\Delta_K$. The binned version of $S$ outputs the average of $S$ on each bin:
   \begin{equation}
       S_B \colon
       \left|
       \begin{array}{l}
           \Xc \to \Sc\\
           x \mapsto \esp{S|S \in \Bc_j} \quad \text{where $\Bc_j$ is the bin $S(x)$ falls into.}
       \end{array}
       \right.
   \end{equation}
   The \textbf{binned calibrated scores} are defined by:\\
   \hspace*{1em} $C_B := \PPG{Y = 1}{S_B} = \espk{Q}{S_B} = \espk{C}{S_B}$.
\end{definition}
\end{minipage}%
\smallskip

This is the approach taken by the popular Expected Calibration Error (ECE) \citep{Naeini2015Obtaining}.
However, the loss estimated for a binned classifier deviates from that of
the original one. In particular, binning biases the calibration loss
downwards \citep{Kumar2019verified}. Here we show that on the contrary it
creates an upwards bias for the grouping loss. Binning a classifier $S$ into $S_B$ boils down to merging the level sets $S$ into a finite number of larger level sets of confidence score $S_B$. For example in \autoref{fig:binning}, all level sets with $S \in [0.5, 1]$ are merged into one level set of confidence $S_B = 0.75$, which artificially inflates the variance of $Q$ in each bin. This intuition is formalized in \autoref{prop:glbin}
\end{minipage}%
\begin{restatable}[Binning-induced grouping loss]{proposition}{bininduced}
      \label{prop:glbin}
      The grouping loss of the binned classifier $\GL(S_B)$ deviates from that of the original classifier $\GL(S)$ by an induced grouping loss $\GLind(S, S_B)$:
      \begin{align}
          \underbrace{\esp{\vark[h]{Q}{S_B}}}_{\GL(S_B)}
          & \; = \; \underbrace{\esp{\vark[h]{Q}{S}}}_{\GL(S)} \; + \; \underbrace{\esp{\vark[h]{C}{S_B}}}_{\GLind(S, S_B)}
      \end{align}
      Moreover, if the scoring rule is proper: $\qquad\qquad\GLind(S, S_B) \geq 0$.
\end{restatable}

\autoref{prop:glbin:proof} gives the proof. \autoref{prop:glbin} shows that the difference between the grouping loss of the binned and original classifier is given by the $h$-variance of the original calibrated scores in a bin. This result provides an expression for $\GLind$ which can then be estimated as shown in \autoref{ss:estimation}.


\begin{remark}
    Interestingly, the binning-induced grouping and calibration losses partly compensate each other (\autoref{cor:bounds} in \autoref{sec:clglinduced}).
\end{remark}



Applying the decomposition of \autoref{th:GL_decomp} to the binned classifier $S_B$ and accounting for binning using \autoref{prop:glbin}, we obtain a new decomposition of the grouping loss:
\begin{restatable}[Explained grouping loss accounting for binning]{proposition}{binaccounting}
  \label{prop:gllb-binning}
  \begin{equation}
    \GL(S)=\glexp(S_B)-\GLind(S, S_B)+\glres(S_B)
  \end{equation}
  \begin{flalign}
    \text{If the scoring rule is proper, then:} &\quad\qquad\GL(S) \geq \underbrace{\glexp(S_B) - \GLind(S, S_B)}_{\gllb(S, S_B)}.&
    \label{eq:lb-binning}
  \end{flalign}
\end{restatable}
The proof is given in \autoref{prop:gllb-binning:proof}. Importantly, contrary to the grouping loss, both terms in the lower-bound (\autoref{eq:lb-binning}) can be estimated. In the remainder of this paper, we will be interested in the estimation and optimization of the lower bound $\gllb(S, S_B)$.

\subsection{Grouping loss estimation}
\label{ss:estimation}
We now derive a grouping-loss estimation procedure by focusing on each of its components in turn: $\glexp(S_B)$ and $\GLind(S, S_B)$.

\paragraph{A debiased estimator for the explained grouping loss $\glexp(S_B)$}
The most natural estimator for the explained grouping loss is a plugin estimator, replacing $\espk{Q}{S,\Ccal}$ by the empirical means of $Y$ over each region. It is nonetheless generally biased. We show below that in the case of the Brier scoring rule, a direct empirical estimation of
$\glexp$ on the partition is biased upwards (cf \autoref{ss:proof_bias_upwards}), and propose a debiased estimator.


\begin{restatable}[Debiased estimator for the Brier score]{proposition}{debiasing}

\label{prop:debiasing}
    For all class $k \in \cbr{1, \dots, K}$ and bin $s \in \Scal$, let $n^{(s, k)}$ (resp. $n_j^{(s, k)}$) be the number of samples belonging to level set $\Rcal^{(s)}$ (resp. region $\Rcal_j^{(s)}$). We define the empirical average of $Y$ over these regions as:
    \begin{equation*}
        \hat \mu_j^{(s, k)} := \frac{1}{n_j^{(s, k)}}\sum_{i: X^{(i)} \in \Rcal_j^{(s)}} Y_k^{(i)}
        \quad \text{and} \quad
        \hat c^{(s, k)} = \frac{1}{n^{(s, k)}}\sum_{i: X^{(i)} \in \Rcal^{(s)}} Y_k^{(i)}
    \end{equation*}
    \vspace{-1em}
    \begin{flalign*}
      \text{The debiased estimator of $\glexp$ is:~~~}
      &\;
      \displaystyle \glexphat(S_B) = \sum_{k=1}^K \sum_{s \in \Scal} \frac{n^{(s, k)}}{n} \glexphat^{(s, k)}(S_B)
    \end{flalign*}
        with:
    \begin{flalign*}
        \glexphat^{(s, k)}(S_B) =
          \underbrace{\sum_{j=1}^J \frac{n^{(s, k)}_j}{n^{(s, k)}} \br{\hat \mu^{(s, k)}_j - \hat c^{(s, k)}}^2}_{\text{\small plugin estimator $\glpluginhat$}} - \biggr( \underbrace{\sum_{j=1}^J \frac{n_j^{(s, k)}}{n^{(s, k)}} {\frac{\hat \mu^{(s, k)}_j (1-\hat \mu^{(s, k)}_j)}{n^{(s, k)}_j - 1}} - \frac{\hat c^{(s, k)}(1-\hat c^{(s, k)})}{n^{(s, k)} - 1}}_{\text{\small bias estimation $\glbiashat$}} \biggl)
    \end{flalign*}
\end{restatable}
\autoref{ss:proof_debiasing} gives the proof, with a debiasing logic
similar to \cite{Brocker2012Estimating}. The leftmost
term corresponds to the plugin estimate: the estimator of the
explained grouping loss (\autoref{th:GL_decomp}) with sample estimators for
the quantities of interest. The two rightmost terms represent the finite-sample
variance in estimating expectations over regions. They correct the upwards bias of the plugin estimate. 


\paragraph{Estimation of the grouping loss induced by binning classifier
scores.}
$\GLind(S, S_B)$ involves the $h$-variance of the calibrated scores $C$ inside each bin, thus its estimation requires $C$.
A solution is to estimate a continous calibration curve $\hat{C}$,
which amounts to a one-dimensional problem for which various methods are available. In our experiments, we use 
a kernel-based method (\eg{} LOWESS). It is then easy to compute the $h$-variance of $\hat{C}$ inside each bin by evaluating $\hat{C}$ for all available samples. The resulting expression of the estimator $\GLindhat$ is given in \autoref{sec:glind-est}.

\paragraph{A partition to minimize $\glres$}
In order to achieve the best possible lower-bound, we choose partitions in \autoref{th:GL_decomp} to minimize $\glres$.
We use a decision tree with a loss corresponding to the scoring
rule --squared loss for Brier score-- on the labels $Y$ to
define regions that minimize the loss on a given level set of $S$. As
this approach relies on $Y$, a train-test split is used to
control for overfitting: a partitioning of the
feature space is defined using the leaves of the tree fitted on one part, then the empirical means used in $\glexphat$ are estimated on the other part given this partitioning.
In the experiments of \autoref{sec:exp_nn}, we work in the output space of the penultimate layer of the networks.

\section{Experimental study}
\label{sec:exp}

\subsection{Simulations: finer partitions give a tight grouping-loss lower bound}
\label{sec:exp_simu}

Here we investigate the behavior of our estimation procedure with respect
to the number of bins and number of regions on simulated data with known
grouping loss. The importance of both corrections - the binning-induced
grouping loss (\autoref{prop:glind-est}) and the debiasing
(\autoref{prop:debiasing}) - is also evaluated. For this, data $Y \in
\cbr{0, 1}$ is drawn according to a known true posterior probability $Q$
and we consider a calibrated logistic regression classifier for the
scores $S$ (details in \autoref{sec:app:examples:realistic} and
\autoref{fig:app:real:cal_noacc}). The estimation procedures are then
applied according to two different scenarios. First we vary the number of samples per region (\eg{} region ratio) while the number of bins is fixed (\autoref{fig:simu}a.). Then we vary the number of
bins while the region ratio is
fixed (\autoref{fig:simu}b.).

\begin{figure}[h!]
  \begin{minipage}{.64\linewidth}
    \includegraphics[width=\linewidth]{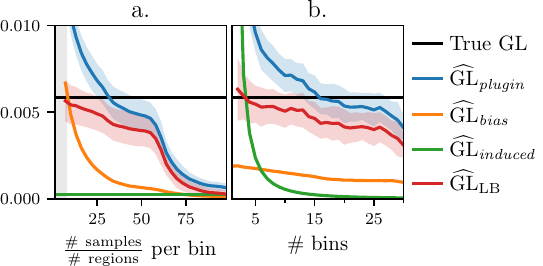}
  \end{minipage}%
  \hfill%
  \begin{minipage}{.35\linewidth}
    \vspace{3mm}
    \caption{
      \label{fig:simu}
      \textbf{Simulation:} estimating the grouping loss lower bound
$\widehat{\mathrm{GL}}_{\mathrm{LB}}$ on a simulated problem
(Appendix~\ref{sec:app:examples:realistic},
\autoref{fig:app:real:cal_noacc}). Right has a fixed ratio $\frac{\text{\# samples}}{\text{\# regions}} = 30$ per bin. Bins are equal-width. Averaged curves are plotted with a $\pm1$ standard deviation envelop.
    }
  \end{minipage}%
\end{figure}


For a fine-enough partition (a large number of regions, and hence a small
region ratio), $\gllbhat$ provides a tight lower bound to the true
grouping loss $\GL$. If the average number of samples per region becomes too
small, some regions have less than two samples, which breaks the estimate
(grayed out area in \autoref{fig:simu}a.).
Conversely, the naive plugin estimate $\glpluginhat$
substantially overestimates the true grouping loss as it does not
include the corrections $\GLindhat$ and $\glbiashat$.
\autoref{fig:simu}b. shows that to control the $\GLind$ due to binning, a reasonably large
number of bins is needed, \eg{} 15 as typical to compute ECE.
Given these bins, we suggest to use a tree to divide them in as many regions as possible while controlling the probability of regions ending up with less than two samples, typically targeting a region ratio of a dozen, to obtain the best possible lower bound $\gllbhat$.

\subsection{Modern neural networks display grouping loss}
\label{sec:exp_nn}

\paragraph{The grouping diagram: visualization of the grouping loss}
In a binary setting, calibration curves display the calibrated scores $C$ versus the confidence scores $S$ of the positive class. To visualize the heterogeneity among region scores in a level set, 
we add to this representation the estimated region scores $\hat \mu_j$, \ie{} the fraction of positives in each region obtained from the partitioning of level sets (\autoref{fig:showcase}). The further apart the region scores are, the greater the grouping loss.
\begin{figure}[h]
  \begin{minipage}{.35\linewidth}
    \vspace{3mm}
    \centering
    \includegraphics[width=.75\linewidth,trim={0 0 0.385cm 0},clip]{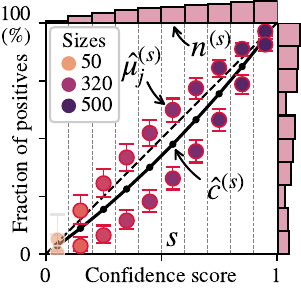}
  \end{minipage}%
  \hfill
  \begin{minipage}{.65\linewidth}
    \vspace{3mm}
    \caption{\textbf{Grouping diagram.}
    Calibration curve of a binned binary classifier augmented with the estimated region scores $\hat \mu_j$ for a partitioning of each level set into 2 regions. Region sample sizes are plotted as a gradient color. The classifier is binned into 10 equal-width bins whose sample sizes $n^{(s)}$ are given as an histogram. A Clopper-Pearson 95\% confidence interval is plotted on the region scores. Regions for which the calibrated score $\hat c^{(s)}$ lie within this interval are grayed out.
    \label{fig:showcase}
    }
  \end{minipage}%
\end{figure}

\begin{figure}[b!]
  \centering
  \makebox[\textwidth]{
  \makebox[0.9cm]{}
  \makebox[(\textwidth-1.5cm)/4]{a. ConvNeXt}
  \makebox[(\textwidth-1.5cm)/4]{b. ViT}
  \makebox[0.4cm]{}
  \makebox[(\textwidth-1.5cm)/4]{c. ConvNeXt}
  \makebox[(\textwidth-1.2cm)/4]{d. ViT}
  }%
  \vspace*{-0.2cm}
  \\
  \includegraphics[height=4.36cm,trim={0.26cm 0cm 0.635cm 0cm},clip]{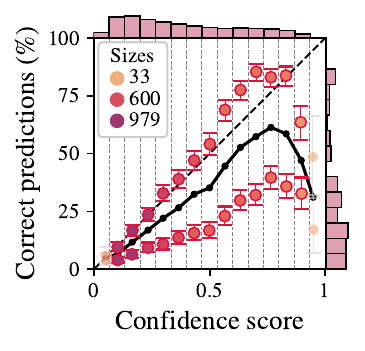}%
  \hspace*{-0.32cm}\hfill%
  \includegraphics[height=4.36cm,trim={0.2cm 0cm 0.635cm 0cm},clip]{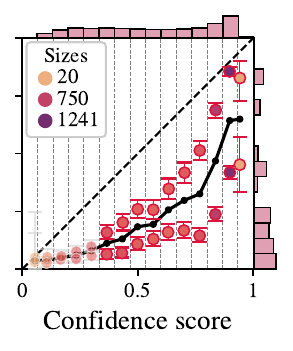}%
  \hspace*{-0.2cm}\hfill%
  \includegraphics[height=4.36cm,trim={0.2cm 0cm 0.635cm 0cm},clip]{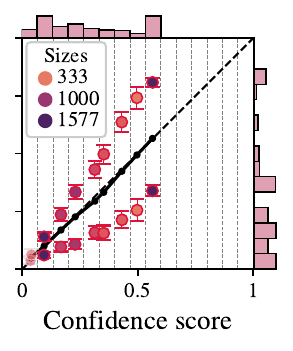}%
    \llap{\raisebox{1.55cm}{\parbox{1.7cm}{\flushright\small with\\isotonic recalibration}}}%
  \hspace*{-0.32cm}\hfill%
  \includegraphics[height=4.36cm,trim={0.2cm 0cm 0.635cm 0cm},clip]{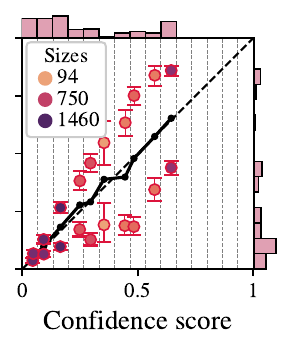}%
    \llap{\raisebox{1.55cm}{\parbox{1.7cm}{\flushright\small with\\isotonic recalibration}}}%

  \caption{\textbf{Vision}: Fraction of correct predictions versus
confidence score of predicted class ($\max_k S_k$) for ConvNeXt Tiny and
ViT L-16 on ImageNet-R, without post-hoc recalibration
(a. and b.) and with isotonic recalibration (c. and d.). In each bin on confidence scores, the level set is partitioned into 2 regions with a decision stump constrained to one balanced split, with a 50-50 train-test split strategy.\looseness=-1}
  \label{fig:vision:conv_vit}
\end{figure}

\begin{figure}[b!]
  \begin{minipage}{.44\linewidth}
    \caption{\textbf{Evaluating vision models}: a debiased estimate of the
grouping loss lower bound $\widehat{\mathrm{GL}}_{\mathrm{LB}}$
(\autoref{eq:lb-binning}) and an estimate of the calibration loss
$\widehat{\mathrm{CL}}$, both accounting for binning, evaluated on
ImageNet-R and sorted by model accuracy. Partitions $\Ccal$ are obtained
from a decision tree partitioning constrained to create at most
$\nicefrac{\text{\# samples in bin}}{30}$ regions in each bin. Isotonic
regression is used for post-hoc recalibration of the models (right). \autoref{tab:vision-nlp:comparison} gives the raw values.%
    \label{fig:vision:comparison}%
}%
  \end{minipage}%
    \hfill
  \begin{minipage}{.53\linewidth}
    \includegraphics[width=\linewidth]{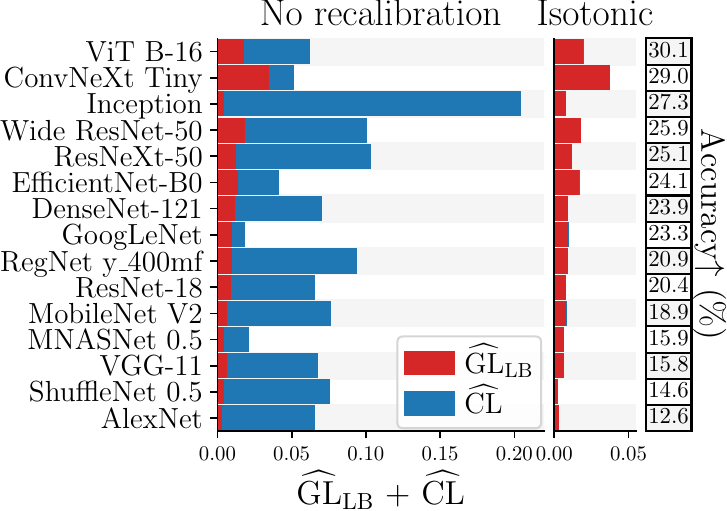}
  \end{minipage}%
\end{figure}

\paragraph{Vision}

We evaluate 15 vision models (listed on \autoref{fig:vision:comparison}) from PyTorch \emph{pre-trained} on ImageNet-1K \citep{deng2009imagenet}. Here we report evaluation on
ImageNet-R \citep{hendrycks2021many}, an ImageNet variant with 15
different renditions: paintings, toys, tattoos, origami... The dataset
contains 30\,000 images and 200 ImageNet classes. \autoref{sec:app:imagenet} reports
evaluation on the validation set of ImageNet-1K and ImageNet-C
\citep{hendrycks2019robustness}, an ImageNet variant with
corrupted versions of ImageNet images. As often with many
classes, the small number of samples per class (50) does not allow to
study the classwise calibration and grouping loss. Hence, following common
practice, we consider top-label versions (\autoref{def3}).
\autoref{sec:app:imagenet} gives experimental details.

We find substantial grouping loss inside
level sets for most networks on ImageNet-R (ConvNeXt Tiny and ViT L-16 in \autoref{fig:vision:conv_vit}, others in \autoref{sec:app:imagenet-r}), even after post-hoc recalibration (\autoref{fig:vision:conv_vit} right).
For instance, while ConvNext + Isotonic is calibrated (third graph), it is strongly over-confident in one part of the feature space and under-confident in the other, creating a high grouping loss.

\begin{figure}[t!]
  \centering
  \makebox[\textwidth]{%
  \makebox[0.9cm]{}%
  \makebox[(\textwidth-1.5cm)/4]{a. In-distribution}%
  \makebox[(\textwidth-1.5cm)/4]{b. Out-of-distribution}%
  \makebox[0.4cm]{}
  \makebox[(\textwidth-0.5cm)/4]{c. In-distribution}%
  \makebox[(\textwidth-1.5cm)/4]{d. Out-of-distribution}%
  }
  \\
  \vspace*{-0.12cm}%
  \hspace*{-0.05cm}%
  \includegraphics[height=4.36cm,trim={0.2cm 0cm 0.635cm 0cm},clip]{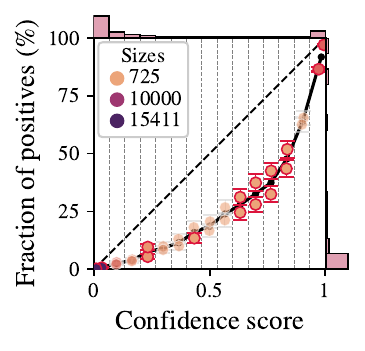}%
  \hspace*{2.5mm}%
  \includegraphics[height=4.36cm,trim={0.26cm 0cm 0.635cm 0cm},clip]{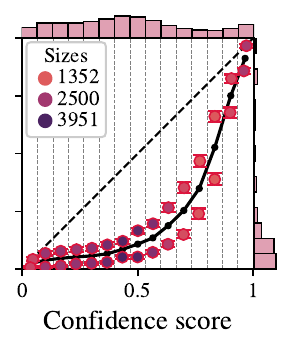}
  \hspace*{0.02cm}%
    \hfill%
  \includegraphics[height=4.36cm,trim={0.2cm 0cm 0.635cm 0cm},clip]{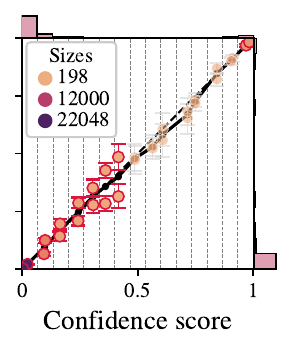}%
    \llap{\raisebox{1.3cm}{\parbox{1.7cm}{\small with isotonic recalibration}}}%
  \hspace*{2.5mm}%
  \includegraphics[height=4.36cm,trim={0.2cm 0cm 0.635cm 0cm},clip]{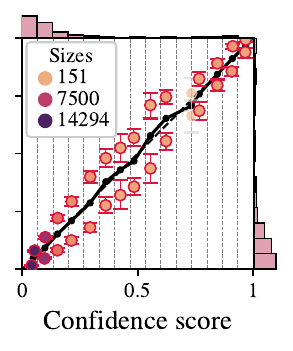}%
    \llap{\raisebox{1.3cm}{\parbox{1.7cm}{\small with isotonic recalibration}}}%
  \hspace*{-0.04cm}%

  \caption{\textbf{NLP}:
Fraction of positives versus confidence score of the positive class of
fine-tuned BART for zero-shot classification on the test set of Yahoo
Answers Topics without post-hoc recalibration (a. and
b.) and with isotonic recalibration (c. and d.). The test set is either restricted to the 5 topics on which the network was trained (in-distribution) or to 5 unseen topics (out-of-distribution). In each level set clusters are built with a balanced decision stump and a 50-50 train-test split strategy.}
\label{fig:nlp:bart}
\end{figure}

\autoref{fig:vision:comparison} shows
that the grouping loss varies across architectures, even with comparable
accuracy. For example, ViT has a slightly better accuracy than ConvNeXt,
but a lower estimated grouping loss. Post-hoc recalibration does not affect the grouping loss (\autoref{fig:vision:conv_vit} right and \autoref{fig:vision:comparison} right), leading to the same conclusions (see \autoref{sec:app:recal-impact} for the analytical impact of recalibration on the grouping loss). We observe the same effects on ImageNet-C (\autoref{sec:app:imagenet-c}), but little or none on ImageNet-1K (\autoref{sec:app:imagenet-1k}). This suggests that stronger grouping loss arises in out-of-distribution settings.
%
Visual inspection of the images suggests that the partitions
capture heterogeneity coming from how realistic an image is, or the
different rendition types (\autoref{sec:app:imagenet-r}).


\paragraph{NLP}
We evaluate the grouping loss on BART Large \citep{Lewis2019} from HuggingFace \emph{pre\nobreakdash-trained}
on the Multi-Genre Natural Language Inference dataset \citep{MNLI}. We
consider zero-shot topic classification on the Yahoo Answers Topics
dataset, composed of questions and topic labels. There are 60\,000 test
samples and 10 topics. The model is fine-tuned on 5 out of the 10 topics
of the training set, totaling 700\,000 samples. Given a question title
and a hypothesis (\eg{} ``This text is about Science \& Mathematics''),
the model outputs its confidence in the hypothesis to be true. The
classification being zero-shot, the hypothesis can be about an unseen
topic. We evaluate the model separately on the 5 unseen topics and the 5
seen topics of the test set. Both results in a binary classification task
on whether the hypothesis is correct or not.
\autoref{sec:app:nlp} gives experimental details.

The partitioning reveals grouping loss in the out-of-distribution setting both before and after recalibration (\autoref{fig:nlp:bart}b. and d.). However, we found no evidence of grouping loss in the in-distribution setting.
As in vision, this suggests that out-of-distribution settings lead to stronger grouping loss.

\section{Discussion and conclusion}
\label{sec:discussion}

\paragraph{A working estimator of grouping loss}
While calibrated scores can be estimated by solving a one-dimensional problem, the grouping loss is much harder to estimate: it measures the discrepancy to the true posterior probabilities, which are unknown.
We show that combining debiased partition-based estimators with an optimized partition
captures the grouping loss well.
This procedure allows us to characterize the grouping loss of popular neural networks for the first time.
We find that in vision and NLP, models can be calibrated --if needed via post-hoc recalibration-- but significant heterogeneity of errors remains, \eg{} ConvNeXt has larger grouping loss than calibration loss.

Several avenues could be explored to better capture the grouping loss. Complex level
sets may not be approximated well with the partitioning defined by a
tree, leaving a large residual in th.\,\ref{th:GL_decomp}. In this case,
the estimated grouping loss may only be a rather loose lower bound. Such
a lower bound is nevertheless useful to reject models with high grouping
loss.
In addition, we apply the tree on the penultimate layer of neural networks,
where class boundaries are simplified.
Finally, complementing the proposed lower bound with an upper bound would also allow to identify models without grouping loss.\looseness=-1

\paragraph{We need to talk about grouping loss}
Model should be evaluated not only on aggregate measures, but also on their individual predictions, using grouping loss.
The presence of grouping loss means that the model is systematically
under-confident for certain groups of individuals and over-confident for
others, questioning the use of such models for individual decision
making. The presence of grouping loss also means that downstream tasks
relying on confidence scores can be hindered, such as causal inference
with propensity scores or simulation-based inference. Finally, this
heterogeneity raises fairness concerns. In fact, the grouping loss and
our lower bound are fundamentally related to fairness --see sufficiency and group calibration \citep[chap 3]{barocas-hardt-narayanan}, and multicalibration \citep{Kleinberg2016}. We hope that
our measure of grouping loss will spur new research in this area.\looseness=-1

\clearpage
\section*{Reproducibility statement}
All datasets are publicly available (ImageNet-R, ImageNet-C, ImageNet-1K, Yahoo Answers Topics) and all models involved are pre\nobreakdash-trained and publicly available on PyTorch and HuggingFace. Simulated examples are described in \autoref{sec:app:examples}. Detailed experimental methods are given in \autoref{sec:app:imagenet}~and~\ref{sec:app:nlp}. Proofs of all theoretical results are listed in \autoref{sec:proofs}.
The source code for the implementation of the algorithm, experiments, simulations and figures is available on GitHub: \url{https://github.com/aperezlebel/beyond_calibration}.

\section*{Acknowledgments}
We acknowledge support in part by the French Agence Nationale de la Recherche under Grant ANR-20-CHIA-0026 (LearnI).

\bibliography{beyond_calibration}

\begin{thebibliography}{34}
\providecommand{\natexlab}[1]{#1}
\providecommand{\url}[1]{\texttt{#1}}
\expandafter\ifx\csname urlstyle\endcsname\relax
  \providecommand{\doi}[1]{doi: #1}\else
  \providecommand{\doi}{doi: \begingroup \urlstyle{rm}\Url}\fi

\bibitem[Barber(2020)]{Barber2020}
Rina~Foygel Barber.
\newblock {Is distribution-free inference possible for binary regression?}
\newblock \emph{Electronic Journal of Statistics}, 14\penalty0 (2):\penalty0
  3487--3524, apr 2020.

\bibitem[Barber et~al.(2019)Barber, Cand{\`{e}}s, Ramdas, and
  Tibshirani]{FoygelBarber2019}
Rina~Foygel Barber, Emmanuel~J. Cand{\`{e}}s, Aaditya Ramdas, and Ryan~J.
  Tibshirani.
\newblock {The limits of distribution-free conditional predictive inference}.
\newblock \emph{Information and Inference}, 10\penalty0 (2):\penalty0 455--482,
  mar 2019.
\newblock ISSN 20498772.
\newblock \doi{10.48550/arxiv.1903.04684}.

\bibitem[Barocas et~al.(2017)Barocas, Hardt, and
  Narayanan]{barocas-hardt-narayanan}
Solon Barocas, Moritz Hardt, and Arvind Narayanan.
\newblock \emph{Fairness and Machine Learning}.
\newblock Nips tutorial, 2017.

\bibitem[Br{\"o}cker(2009)]{Brocker2009Reliability}
Jochen Br{\"o}cker.
\newblock Reliability, sufficiency, and the decomposition of proper scores.
\newblock \emph{Quarterly Journal of the Royal Meteorological Society: A
  journal of the atmospheric sciences, applied meteorology and physical
  oceanography}, 135\penalty0 (643):\penalty0 1512--1519, 2009.

\bibitem[Br{\"o}cker(2012)]{Brocker2012Estimating}
Jochen Br{\"o}cker.
\newblock Estimating reliability and resolution of probability forecasts
  through decomposition of the empirical score.
\newblock \emph{Climate dynamics}, 39\penalty0 (3):\penalty0 655--667, 2012.

\bibitem[Dawid(1986)]{Dawid2014a}
A.~P. Dawid.
\newblock Probability forecasting, 1986.

\bibitem[Dawid \& Musio(2014)Dawid and Musio]{Dawid2014}
A.~Philip Dawid and Monica Musio.
\newblock Bayesian model selection based on proper scoring rules.
\newblock \emph{Bayesian Analysis}, 10:\penalty0 479--499, 9 2014.
\newblock \doi{10.1214/15-BA942}.

\bibitem[Deng et~al.(2009)Deng, Dong, Socher, Li, Li, and
  Fei-Fei]{deng2009imagenet}
Jia Deng, Wei Dong, Richard Socher, Li-Jia Li, Kai Li, and Li~Fei-Fei.
\newblock Imagenet: A large-scale hierarchical image database.
\newblock In \emph{2009 IEEE conference on computer vision and pattern
  recognition}, pp.\  248--255. Ieee, 2009.

\bibitem[Flach(2016)]{Flach2016a}
Peter~A. Flach.
\newblock {Classifier Calibration}.
\newblock \emph{Encyclopedia of Machine Learning and Data Mining}, pp.\  1--8,
  2016.

\bibitem[Foster \& Vohra(1998)Foster and Vohra]{foster1998asymptotic}
Dean~P Foster and Rakesh~V Vohra.
\newblock Asymptotic calibration.
\newblock \emph{Biometrika}, 85\penalty0 (2):\penalty0 379--390, 1998.

\bibitem[Garcin \& St{\'{e}}phan(2021)Garcin and St{\'{e}}phan]{Garcin2021}
Matthieu Garcin and Samuel St{\'{e}}phan.
\newblock {Credit scoring using neural networks and SURE posterior probability
  calibration}.
\newblock 2021.

\bibitem[Gneiting \& Raftery(2007)Gneiting and Raftery]{Gneiting2007}
Tilmann Gneiting and Adrian~E. Raftery.
\newblock Strictly proper scoring rules, prediction, and estimation.
\newblock \emph{Journal of the American Statistical Association}, 102, 2007.
\newblock ISSN 01621459.
\newblock \doi{10.1198/016214506000001437}.

\bibitem[Gneiting et~al.(2007)Gneiting, Balabdaoui, and
  Raftery]{gneiting2007probabilistic}
Tilmann Gneiting, Fadoua Balabdaoui, and Adrian~E Raftery.
\newblock Probabilistic forecasts, calibration and sharpness.
\newblock \emph{Journal of the Royal Statistical Society: Series B (Statistical
  Methodology)}, 69\penalty0 (2):\penalty0 243--268, 2007.

\bibitem[Guo et~al.(2017)Guo, Pleiss, Sun, and Weinberger]{Guo2017Calibration}
Chuan Guo, Geoff Pleiss, Yu~Sun, and Kilian~Q Weinberger.
\newblock On calibration of modern neural networks.
\newblock In \emph{International Conference on Machine Learning}, pp.\
  1321--1330. PMLR, 2017.

\bibitem[Gupta et~al.(2020)Gupta, Rahimi, Ajanthan, Mensink, Sminchisescu, and
  Hartley]{Gupta2020b}
Kartik Gupta, Amir Rahimi, Thalaiyasingam Ajanthan, Thomas Mensink, Cristian
  Sminchisescu, and Richard Hartley.
\newblock {Calibration of Neural Networks using Splines}.
\newblock 2020.
\newblock ISSN 2331-8422.

\bibitem[H{\'{e}}bert-Johnson et~al.(2018)H{\'{e}}bert-Johnson, Kim, Reingold,
  and Rothblum]{Hebert-Johnson2018}
Ursula H{\'{e}}bert-Johnson, Michael~P Kim, Omer Reingold, and Guy~N Rothblum.
\newblock {Multicalibration: Calibration for the (Computationally-Identifiable)
  Masses}, jul 2018.
\newblock ISSN 2640-3498.

\bibitem[Hendrycks \& Dietterich(2019)Hendrycks and
  Dietterich]{hendrycks2019robustness}
Dan Hendrycks and Thomas Dietterich.
\newblock Benchmarking neural network robustness to common corruptions and
  perturbations.
\newblock \emph{Proceedings of the International Conference on Learning
  Representations}, 2019.

\bibitem[Hendrycks et~al.(2021)Hendrycks, Basart, Mu, Kadavath, Wang, Dorundo,
  Desai, Zhu, Parajuli, Guo, Song, Steinhardt, and Gilmer]{hendrycks2021many}
Dan Hendrycks, Steven Basart, Norman Mu, Saurav Kadavath, Frank Wang, Evan
  Dorundo, Rahul Desai, Tyler Zhu, Samyak Parajuli, Mike Guo, Dawn Song, Jacob
  Steinhardt, and Justin Gilmer.
\newblock The many faces of robustness: A critical analysis of
  out-of-distribution generalization.
\newblock \emph{ICCV}, 2021.

\bibitem[Kleinberg et~al.(2016)Kleinberg, Mullainathan, and
  Raghavan]{Kleinberg2016}
Jon Kleinberg, Sendhil Mullainathan, and Manish Raghavan.
\newblock {Inherent Trade-Offs in the Fair Determination of Risk Scores}.
\newblock \emph{Leibniz International Proceedings in Informatics, LIPIcs}, 67,
  sep 2016.
\newblock ISSN 18688969.
\newblock \doi{10.48550/arxiv.1609.05807}.

\bibitem[Kuhn \& Johnson(2013)Kuhn and Johnson]{Kuhn2013}
Max Kuhn and Kjell Johnson.
\newblock \emph{{Applied predictive modeling}}.
\newblock 2013.

\bibitem[Kull \& Flach(2015)Kull and Flach]{Kull2015Novel}
Meelis Kull and Peter Flach.
\newblock Novel decompositions of proper scoring rules for classification:
  Score adjustment as precursor to calibration.
\newblock In \emph{Joint European Conference on Machine Learning and Knowledge
  Discovery in Databases}, pp.\  68--85. Springer, 2015.

\bibitem[Kull et~al.(2019)Kull, Perello~Nieto, K{\"a}ngsepp, Silva~Filho, Song,
  and Flach]{Kull2019Beyond}
Meelis Kull, Miquel Perello~Nieto, Markus K{\"a}ngsepp, Telmo Silva~Filho, Hao
  Song, and Peter Flach.
\newblock Beyond temperature scaling: Obtaining well-calibrated multi-class
  probabilities with dirichlet calibration.
\newblock \emph{Advances in neural information processing systems}, 32, 2019.

\bibitem[Kumar et~al.(2019)Kumar, Liang, and Ma]{Kumar2019verified}
Ananya Kumar, Percy~S Liang, and Tengyu Ma.
\newblock Verified uncertainty calibration.
\newblock In H.~Wallach, H.~Larochelle, A.~Beygelzimer, F.~d\textquotesingle
  Alch\'{e}-Buc, E.~Fox, and R.~Garnett (eds.), \emph{Advances in Neural
  Information Processing Systems}, volume~32. Curran Associates, Inc., 2019.

\bibitem[Lewis et~al.(2019)Lewis, Liu, Goyal, Ghazvininejad, Mohamed, Levy,
  Stoyanov, and Zettlemoyer]{Lewis2019}
Mike Lewis, Yinhan Liu, Naman Goyal, Marjan Ghazvininejad, Abdelrahman Mohamed,
  Omer Levy, Veselin Stoyanov, and Luke Zettlemoyer.
\newblock Bart: Denoising sequence-to-sequence pre-training for natural
  language generation, translation, and comprehension.
\newblock pp.\  7871--7880, 10 2019.

\bibitem[Mimori et~al.(2021)Mimori, Sasada, Matsui, and
  Sato]{Mimori2021Diagnostic}
Takahiro Mimori, Keiko Sasada, Hirotaka Matsui, and Issei Sato.
\newblock Diagnostic uncertainty calibration: Towards reliable machine
  predictions in medical domain.
\newblock In Arindam Banerjee and Kenji Fukumizu (eds.), \emph{Proceedings of
  The 24th International Conference on Artificial Intelligence and Statistics},
  volume 130 of \emph{Proceedings of Machine Learning Research}, pp.\
  3664--3672. PMLR, 13--15 Apr 2021.

\bibitem[Minderer et~al.(2021)Minderer, Djolonga, Romijnders, Hubis, Zhai,
  Houlsby, Tran, and Lucic]{minderer2021revisiting}
Matthias Minderer, Josip Djolonga, Rob Romijnders, Frances Hubis, Xiaohua Zhai,
  Neil Houlsby, Dustin Tran, and Mario Lucic.
\newblock Revisiting the calibration of modern neural networks.
\newblock \emph{Advances in Neural Information Processing Systems}, 34, 2021.

\bibitem[Naeini et~al.(2015)Naeini, Cooper, and
  Hauskrecht]{Naeini2015Obtaining}
Mahdi~Pakdaman Naeini, Gregory~F. Cooper, and Milos Hauskrecht.
\newblock Obtaining well calibrated probabilities using bayesian binning.
\newblock \emph{Proceedings of the ... AAAI Conference on Artificial
  Intelligence. AAAI Conference on Artificial Intelligence}, 2015:\penalty0
  2901--2907, 2015.

\bibitem[Osborne(1991)]{osborne1991statistical}
Christine Osborne.
\newblock Statistical calibration: a review.
\newblock \emph{International Statistical Review/Revue Internationale de
  Statistique}, pp.\  309--336, 1991.

\bibitem[Paszke et~al.(2019)Paszke, Gross, Massa, Lerer, Bradbury, Chanan,
  Killeen, Lin, Gimelshein, Antiga, Desmaison, Köpf, Yang, DeVito, Raison,
  Tejani, Chilamkurthy, Steiner, Fang, Bai, and Chintala]{Paszke2019}
Adam Paszke, Sam Gross, Francisco Massa, Adam Lerer, James Bradbury, Gregory
  Chanan, Trevor Killeen, Zeming Lin, Natalia Gimelshein, Luca Antiga, Alban
  Desmaison, Andreas Köpf, Edward Yang, Zach DeVito, Martin Raison, Alykhan
  Tejani, Sasank Chilamkurthy, Benoit Steiner, Lu~Fang, Junjie Bai, and Soumith
  Chintala.
\newblock Pytorch: An imperative style, high-performance deep learning library.
\newblock \emph{Advances in Neural Information Processing Systems}, 32, 12
  2019.
\newblock ISSN 10495258.

\bibitem[Roelofs et~al.(2022)Roelofs, Cain, Shlens, and
  Mozer]{roelofs2022mitigating}
Rebecca Roelofs, Nicholas Cain, Jonathon Shlens, and Michael~C Mozer.
\newblock Mitigating bias in calibration error estimation.
\newblock In \emph{International Conference on Artificial Intelligence and
  Statistics}, pp.\  4036--4054. PMLR, 2022.

\bibitem[Vaicenavicius et~al.(2019)Vaicenavicius, Widmann, Andersson, Lindsten,
  Roll, and Sch{\"o}n]{Vaicenavicius2019Evaluating}
Juozas Vaicenavicius, David Widmann, Carl Andersson, Fredrik Lindsten, Jacob
  Roll, and Thomas Sch{\"o}n.
\newblock Evaluating model calibration in classification.
\newblock In \emph{The 22nd International Conference on Artificial Intelligence
  and Statistics}, pp.\  3459--3467. PMLR, 2019.

\bibitem[Vovk et~al.(2005)Vovk, Gammerman, and Shafer]{Vovk2005a}
Vladimir Vovk, Alexander Gammerman, and Glenn Shafer.
\newblock \emph{{Algorithmic learning in a random world}}.
\newblock Springer US, 2005.
\newblock ISBN 0387001522.
\newblock \doi{10.1007/b106715}.

\bibitem[Williams et~al.(2018)Williams, Nangia, and Bowman]{MNLI}
Adina Williams, Nikita Nangia, and Samuel Bowman.
\newblock A broad-coverage challenge corpus for sentence understanding through
  inference.
\newblock In \emph{Proceedings of the 2018 Conference of the North American
  Chapter of the Association for Computational Linguistics: Human Language
  Technologies, Volume 1 (Long Papers)}, pp.\  1112--1122. Association for
  Computational Linguistics, 2018.

\bibitem[Zadrozny \& Elkan(2002)Zadrozny and Elkan]{Zadrozny2002Transforming}
Bianca Zadrozny and Charles Elkan.
\newblock Transforming classifier scores into accurate multiclass probability
  estimates.
\newblock In \emph{Proceedings of the eighth ACM SIGKDD international
  conference on Knowledge discovery and data mining}, pp.\  694--699, 2002.

\end{thebibliography}
\bibliographystyle{beyond_calibration}

\clearpage
\appendix
%
%
\renewcommand \thepart{}
\renewcommand \partname{}
\vspace*{-10mm}

\part{\texorpdfstring{}{Supplementary materials}} 

\vspace*{-15mm}

{\center\LARGE\sc {Supplementary materials}\par}
\vspace{20mm}

{\hypersetup{hidelinks}
\parttoc 
}

\clearpage

\section{Examples of confusing statements on calibration}

\label{sec:confusing_statements}

Here we detail specific examples of confusing statements on calibration
in the literature. We choose most of these examples in well-cited and well
regarded works.

\begin{itemize}
  \item \cite{Kuhn2013}: \inlinequote{We desire that the estimated class probabilities are reflective of the true underlying probability of the sample. That is, the predicted class probability (or probability-like value) needs to be well-calibrated. To be well-calibrated, the probabilities must effectively reflect the true likelihood of the event of interest.}

  The authors write that it is desirable to have confidence scores $S$
reflective of the true posterior probabilities $Q$, which is indeed
desirable as discussed in \autoref{sec:introduction}. However, they write
this is obtained through calibration. Although post-hoc recalibration makes the confidence scores closer to $Q$ in some sense, there is an implicit shortcut. As pointed out in \autoref{sec:cal_not_enough} and \autoref{sec:app:examples}, calibration, even with optimal accuracy, does not guarantee confidence scores $S$ to be close to the true posterior probabilities $Q$.


  \item \citep{Gupta2020b}: \inlinequote{A classifier is said to be calibrated if the probability values it associates with the class labels match the true probabilities of correct class assignments.}

  The authors write that calibration is matching the confidence scores $S$ of a classifier to the true posterior probabilities $Q$. In fact, calibration is matching the confidence scores $S$ to the calibrated scores $C$, which can be far from the true posterior probabilities $Q$ as pointed out in \autoref{sec:cal_not_enough} and \autoref{sec:app:examples}.

  \item \cite{Garcin2021}: \inlinequote{Ideally, we would like machine learning models to output accurate probabilities in the sense that they reflect the real unobserved probabilities. This is exactly the purpose of calibration techniques, which aim to map the predicted probabilities to the true ones in order to reduce the probability distribution error of the model.}

  The authors write that calibration is outputting confidence scores $S$ that are true posterior probabilities $Q$. As in the previous citations, calibration is outputting calibrated scores $C$, which can be far from  $Q$ (\autoref{sec:cal_not_enough} and \autoref{sec:app:examples}).



  \item \cite{Flach2016a}: \inlinequote{A probabilistic classifier is well calibrated if, among the instances receiving a predicted probability vector p, the class distribution is approximately distributed as p. Hence, the classifier approximates, in some sense, the class posterior.} \inlinequote{The main point is that knowing the true class posterior allows the classifier to make optimal decisions. It therefore makes sense for a classifier to (approximately) learn the true class posterior.}

  Here, calibration is rightly defined as outputting confidence scores $S$ that are equal to the calibrated scores $C$. However, by writing that confidence scores $S$ of a calibrated classifier approximate the true class posterior $Q$, the author makes an implicit assumption that the calibrated scores $C$ are close to the true posterior probabilities $Q$, which is not guaranteed in theory as pointed out in \autoref{sec:cal_not_enough} and \autoref{sec:app:examples}.




\end{itemize}

\section{Examples of accurate and calibrated classifiers with high grouping loss}
\label{sec:app:examples}

Here we build simple binary classification examples of calibrated classifiers with optimal accuracy having their confidence scores far from the true posterior probabilities.
In \autoref{sec:app:examples:1d} we build examples with an arbitrary link between true posterior probabilities $Q$ and confidence scores $S$ (up to a limit to keep the classifier's accuracy optimal).
In \autoref{sec:app:examples:realistic} we build a more realistic example based on the output of a neural network.

\subsection{Arbitrary link between true posterior probabilities Q and confidence scores S}
\label{sec:app:examples:1d}
To show that calibration, even combined with optimal accuracy, does not impose strong constraints on how close the true posterior probabilities $Q$ should be from the classifiers' confidence scores $S$, we build examples in which $Q$ and $S$ have an arbitrary link.
For simplicity we consider binary examples with a one-dimensional feature space $\mathcal{X}$. These can be extended to multiple dimensions by projecting onto a vector $\omega$ (via $x \mapsto \omega^Tx$).

The idea is to build a classifier that outputs confidence scores having at most two antecedents each. One antecedent should have its true posterior probability $Q$ at an arbitrary distance $+\Delta$ from the associated confidence score $S$, while the other has a distance $-\Delta$. Scores with only one antecedent should have $Q = S$. This combined with an equal density weight of $\mathcal{X}$ onto the two antecedents guarantees calibration: $\espk{Q}{S} = S$. To maintain the classifier's accuracy optimal, the offset $\Delta$ is constrained to keep $Q$ and $S$ on the same side of the decision threshold.

To achieve this, we cut the one-dimensional feature space $\mathcal{X}$ into three parts: $\R^{\star}_+$, $\R^{\star}_-$ and $\{0\}$. As a classifier, we take an even function $S(X)$ with $S^{-1}(\{0\})$ reduced to a singleton so that each confidence score has either two antecedents (one in $\R^{\star}_+$ and one in $\R^{\star}_-$) or one antecedent in $\{0\}$. To assign an equal weight to each antecedent, we choose a symmetric distribution for $\mathcal{X}$, \eg{} a standard normal distribution centered on $0$. We build the true posterior probabilities $Q$ from deviations $h : [0, 1] \to [0, 1]$ and $g : [0, 1] \to [0, 1]$ of the confidence scores $S$ in $\R^{\star}_+$ and $\R^{\star}_-$:
\begin{equation}
  Q : x \mapsto \1{x > 0}h(S(x)) + \1{x < 0}g(S(x)) + \1{x = 0}S(0)
\end{equation}
For $S$ to be calibrated, deviations must average to identity, \ie{} $\forall s \in S(\R), \frac{1}{2}(h(s) + g(s)) = s$. A proof of this statement is given below:

\begin{proof}
  \begin{align*}
    \espk{Q(X)}{S(X)}
    & = \espk{\1{X > 0}}{S(X)}h(S(X)) + \espk{\1{X < 0}}{S(X)}g(S(X)) + \espk{\1{X = 0}}{S(X)}S(0)\\
    & = \PPG{X\!>\!0}{S(X)}h(S(X)) \!+\! \PPG{X\!<\!0}{S(X)}g(S(X)) \!+\! \PPG{X\!=\!0}{S(X)}S(0)\\
    & = \nicefrac{1}{2}\1{S(X) \neq S(0)}(h(S(X)) + g(S(X))) + \1{S(X) = S(0)}S(0)\\
  \end{align*}
  since $\PPG{X > 0}{S(X)} = \PPG{X < 0}{S(X)} = \frac{1}{2}\1{S(X) \neq S(0)}$.

  Hence, $S(X)~calibrated \Leftrightarrow \espk{Q(X)}{S(X)} = S(X) \Leftrightarrow \frac{1}{2}(h(S(X)) + g(S(X))) = S(X)$.
\end{proof}

From here, we choose $h : [0, 1] \to [0, 1]$ and define $g : s \mapsto 2s - h(s)$. Note that to keep $g(s) \in [0, 1]$, $h$ is constrained by: $\forall s \in S(\R), 2s -1 \leq h(s) \leq 2s$. At this point of the procedure, classifiers $S$ may not have an optimal accuracy. \autoref{fig:app:cal_noacc} shows two examples of links $h$, one of which saturates the constraint $h(s) <= \min(2s, 1)$.

To make the classifiers accurate, the deviations $h(s) - s$ should be small enough to keep $S$ and $Q$ on the same side of the decision threshold. This adds two constraints on $h$: $\forall s \in S(\R)\cap[0, \nicefrac{1}{2}[, h(s) < \frac{1}{2}$ and $\forall s \in S(\R)\cap[\nicefrac{1}{2}, 1], h(s) \geq \frac{1}{2}$ (with the convention that a score of exactly $\frac{1}{2}$ predicts the positive class). \autoref{fig:app:cal_acc} (left) shows a classifier built following the above procedure. \autoref{fig:app:cal_acc} (right) shows that we can release the constraint $\nicefrac{1}{2}(h(s) + g(s)) = s$ if we tweak the distribution of $\mathcal{X}$ to adapt the weights between the two antecedents accordingly (and take \eg{} $g(s) = \1{h(s) < s})$.

\setlength{\fboxsep}{0pt}
\begin{figure}
  \centering
  \makebox[\textwidth-0.5mm]{
    \makebox[(\textwidth-5mm)/2]{$h : s \mapsto - s^2 + 2s$}
    \makebox[(\textwidth-5mm)/2]{$h : s \mapsto \min(2s, 1)$}
  }%
  \\
  \includegraphics[height=3.65cm]{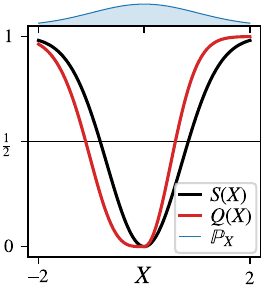}%
  \hspace*{0.5mm}%
  \includegraphics[height=3.65cm]{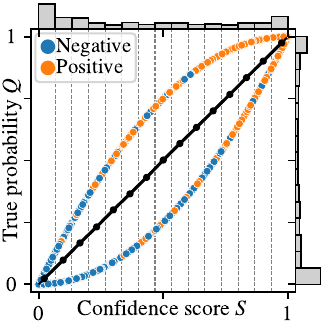}%
  \hspace*{0.65mm}%
  \includegraphics[height=3.65cm]{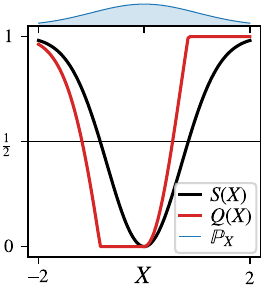}%
  \hspace*{0.5mm}%
  \includegraphics[height=3.65cm]{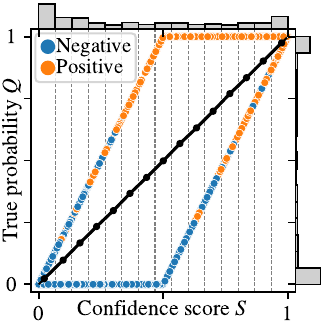}%

  \caption{\textbf{Calibrated but not accurate.} Example of calibrated classifiers $S$ constructed from links $h$ following the procedure described in \autoref{sec:app:examples:1d}. The accuracy of these two classifiers is not optimal as $Q$ and $S$ are not on the same side of the decision threshold ($\frac{1}{2}$) wherever $Q \neq \tfrac{1}{2}$. Refer to \autoref{fig:app:cal_acc} for an example with optimal accuracy. Calibration curves (in black on 2\textsuperscript{nd} and 4\textsuperscript{th} plot) are obtained from 1 million samples.}
  \label{fig:app:cal_noacc}
\end{figure}

\clearpage

\begin{figure}
  \centering
  \makebox[\textwidth]{
    \makebox[(\textwidth-5mm)/2]{$h : s \mapsto \max(\min(2s, \frac{1}{2})), 2s-1)$}
    \makebox[(\textwidth-5mm)/2]{$h : s \mapsto \frac{1}{2}\1{0<s<1} + \1{s=1}$}
  }%
  \\
  \includegraphics[height=3.65cm]{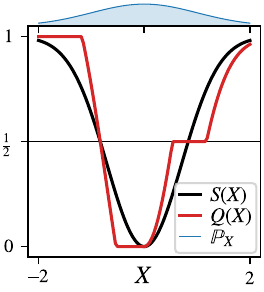}%
  \hspace*{0.5mm}%
  \includegraphics[height=3.65cm]{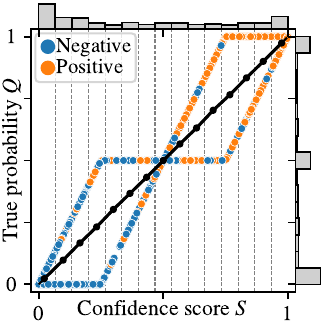}%
  \hspace*{0.65mm}%
  \includegraphics[height=3.65cm]{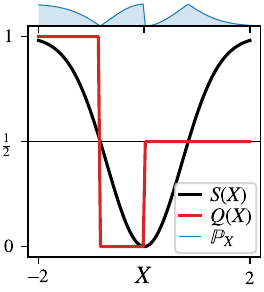}%
  \hspace*{0.5mm}%
  \includegraphics[height=3.65cm]{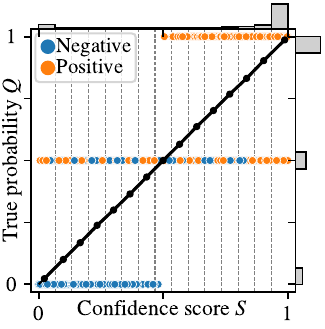}%

  \caption{\textbf{Calibrated and optimal accuracy.} Example of calibrated classifiers $S$ constructed from links $h$ following the procedure described in \autoref{sec:app:examples:1d}. The accuracy of these two classifiers is optimal as $Q$ and $S$ are on the same side of the decision threshold ($\frac{1}{2}$) wherever $Q \neq \tfrac{1}{2}$. However, confidence scores $S$ are almost everywhere different from the true posterior probabilities $Q$. Calibration curves (in black on 2\textsuperscript{nd} and 4\textsuperscript{th} plot) are obtained from 1 million samples.}
  \label{fig:app:cal_acc}
\end{figure}

\subsection{Realistic example based on neural network's output}
\label{sec:app:examples:realistic}
The examples of \autoref{sec:app:examples:1d}, while proving our point, are quite unusual in practice especially in the choice of classifier $S$. In this section we build a more realistic example based on the output of a neural network. We focus on a binary classification setting with a feature space $\mathcal{X}$ being at least two-dimensional. The classifier is taken as a sigmoid of $\omega^TX$ (akin to the last layer of a neural network predicting the confidence score of the positive class). Based on this choice of model, we build a class of calibrated and accurate classifiers with confidence scores $S$ far from the true posterior probabilities $Q$.

The idea is to create heterogeneity in the blind spot of calibration, \ie{} orthogonally to $\omega$. The perturbations creating heterogeneity must balance each other out to keep the classifier calibrated.

To achieve this, we define:
\begin{itemize}[topsep=0pt,itemsep=0ex,partopsep=0ex,parsep=0.5ex,leftmargin=5ex]
  \item $d \geq 2$ the dimension of the feature space $\mathcal{X}$.
  \item $\omega \in \R^d$, the last layer's weights.
  \item $\varphi : \R \to [0, 1]$ the link function mapping $\omega^Tx$ to confidence scores, \eg{} a sigmoid.
  \item $S : x \in \R^d \mapsto \varphi(\omega^Tx) \in [0, 1]$ the classifier's confidence scores of the positive class.
  \item $\omega_{\perp} \in \R^d$ such that $\omega^T\omega_{\perp} = 0$, the direction in which heterogeneity will be introduced.
  \item $\psi : \R \to [-1, 1]$ an odd perturbation introducing balanced heterogeneity along $\omega_{\perp}$.
  \item $\Delta_{max} : x \mapsto \min(1\!-\!S(x), S(x))$ modulating the range of the perturbation to keep $Q \in [0, 1]$.
  \item $Q : x \in \R^d \mapsto S(x) + \psi(\omega_{\perp}^Tx)\Delta_{max}(x) \in [0, 1]$ the constructed true posterior probabilities.
  \item $X \sim \mathcal{N}(0, \Sigma)$ the data distribution, with $\Sigma \in \R^{d \times d}$ having $\omega$ and $\omega_{\perp}$ among its eigenvectors.
\end{itemize}

With the above construction, the classifier $S$ is calibrated. Indeed,
\begin{align}
  \espk{Q(X)}{S(X)}
  & = S(X) + \espk{\psi(\omega_{\perp}^TX)\Delta_{max}(X)}{S(X)}\\
  & = S(X) + \espk{\psi(\omega_{\perp}^TX)}{S(X)}\Delta_{max}(X)
\end{align}
since $\Delta_{max}(X)$ is a function of $S(X)$. We have $\espk{\psi(\omega_{\perp}^TX)}{S(X)} = 0$ by construction: $\psi$ is odd and the distribution of $X$ has a symmetric weight along $\omega_{\perp}$ since $\Sigma$ is aligned on $\omega$ and $\omega_{\perp}$.
Hence $\espk{Q(X)}{S(X)} = S(X)$.
\autoref{fig:app:real:cal_noacc} shows two examples generated with this procedure. However, it is not yet accurate.
As in \autoref{sec:app:examples:1d}, the perturbation should be constrained to keep $Q$ and $S$ on the same side of the decision threshold to keep the accuracy optimal. This is simply achieved by defining $\Delta_{max} : x \mapsto \min(1\!-\!S(x), S(x), |\frac{1}{2} - S(x)|)$.

\begin{figure}[p]
  \makebox[\linewidth]{\textbf{a. Classifier $S(X) = (1 + \exp(-\omega^TX))^{-1}$}}\\
  \vspace*{3mm}%
  \centering{
  \includegraphics[height=4.3cm]{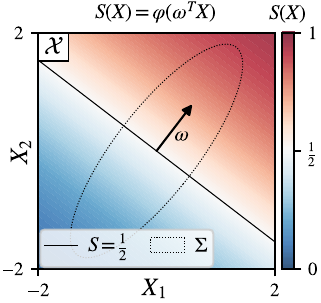}}\\
  \vspace*{4mm}%
  \makebox[\linewidth]{\textbf{b. Perturbation $\psi(z) = 2(1 + \exp(-z))^{-1} - 1$}}\\
  \vspace*{3mm}%
  \includegraphics[height=4.3cm]{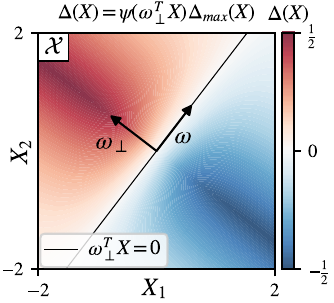}%
  \hspace*{3.8mm}%
  \includegraphics[height=4.3cm]{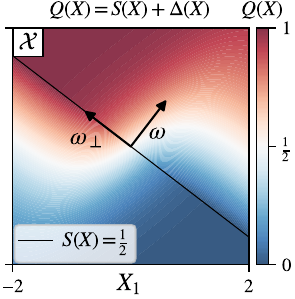}%
  \hspace*{3.6mm}%
  \includegraphics[height=4.3cm]{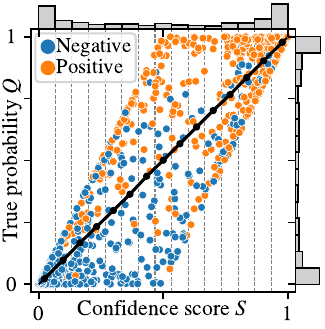}\\
  \vspace*{4mm}%
  \makebox[\linewidth]{\textbf{c. Perturbation $\psi(z) = \1{z > 0} - \1{z < 0}$}}\\
  \vspace*{3mm}%
  \includegraphics[height=4.3cm]{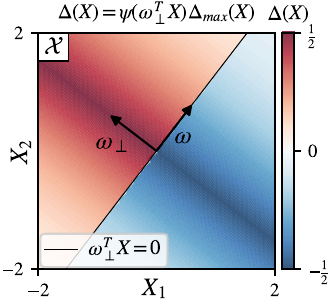}%
  \hspace*{3.8mm}%
  \includegraphics[height=4.3cm]{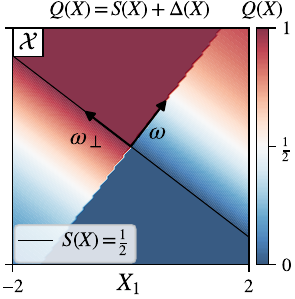}%
  \hspace*{3.6mm}%
  \includegraphics[height=4.3cm]{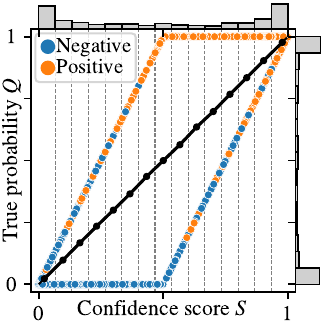}
  \caption{\textbf{Calibrated but not accurate.} Example of a calibrated classifier $S$ constructed following the procedure described in \autoref{sec:app:examples:realistic}. Its accuracy is not optimal as $Q$ and $S$ are not on the same side of the decision threshold ($\frac{1}{2}$) wherever $Q \neq \tfrac{1}{2}$. Refer to \autoref{fig:app:real:cal_acc} for an example with optimal accuracy. Calibration curves (in black on last column) are obtained from 1 million samples.}
  \label{fig:app:real:cal_noacc}
\end{figure}


\begin{figure}[p]
  \makebox[\linewidth]{\textbf{a. Classifier $S(X) = (1 + \exp(-\omega^TX))^{-1}$}}\\
  \vspace*{3mm}%
  \centering{
  \includegraphics[height=4.3cm]{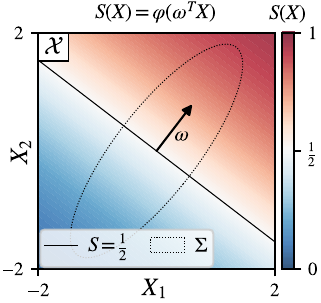}}\\
  \vspace*{4mm}%
  \makebox[\linewidth]{\textbf{b. Perturbation $\psi(z) = 2(1 + \exp(-z))^{-1} - 1$}}\\
  \vspace*{3mm}%
  \includegraphics[height=4.3cm]{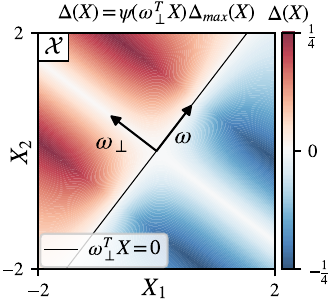}%
  \hspace*{3.8mm}%
  \includegraphics[height=4.3cm]{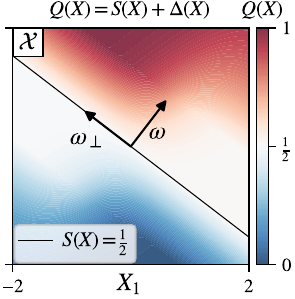}%
  \hspace*{3.6mm}%
  \includegraphics[height=4.3cm]{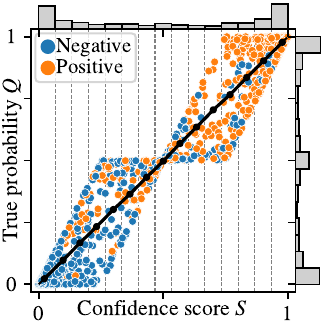}\\
  \vspace*{4mm}%
  \makebox[\linewidth]{\textbf{c. Perturbation $\psi(z) = \1{z > 0} - \1{z < 0}$}}\\
  \vspace*{3mm}%
  \includegraphics[height=4.3cm]{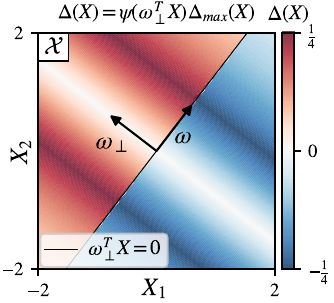}%
  \hspace*{3.8mm}%
  \includegraphics[height=4.3cm]{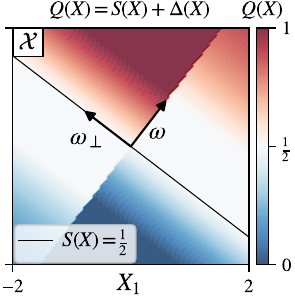}%
  \hspace*{3.6mm}%
  \includegraphics[height=4.3cm]{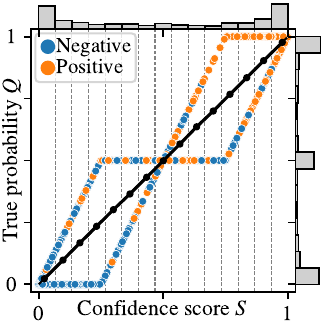}
  \caption{\textbf{Calibrated and optimal accuracy.} Example of a calibrated classifier $S$ constructed following the procedure described in \autoref{sec:app:examples:realistic}. Its accuracy is optimal as $Q$ and $S$ are on the same side of the decision threshold wherever $Q \neq \tfrac{1}{2}$. However, confidence scores $S$ are almost everywhere different from the true posterior probabilities $Q$. Calibration curves (in black on last column) is obtained from 1 million samples.}
  \label{fig:app:real:cal_acc}
\end{figure}


\section{Proofs}
\label{sec:proofs}

\subsection{Lemma~\ref{prop:e-variance}: The grouping loss as an h-variance}
\label{prop:e-variance:proof}

\evariance*
\begin{proof}[Proof of \autoref{prop:e-variance}]
  Let $\phi$ be a scoring rule, $h : p \mapsto -s_{\phi}(p,p)$ and $C = \espk{Q}{S}$.
  \begin{align}
      \esp{d_{\phi}(C, Q)}
      & = \esp{s_{\phi}(C, Q) - s_{\phi}(Q, Q)}
      & \text{Definition of divergence}\\
      & = \esp{s_{\phi}(C, Q) + h(Q)}
      & \text{Definition of $h$}\\
      & = \esp{\sum_{k=1}^K \phi(C, e_k)Q_k + h(Q)}
      & \text{Definition of expected score}\\
      & = \esp{\espk{\sum_{k=1}^K \phi(C, e_k)Q_k + h(Q)}{S}}
      & \text{Law of total expectation}\\
      & = \esp{\sum_{k=1}^K \espk{\phi(C, e_k)Q_k}{S} + \espk{h(Q)}{S}}
      & \text{Linearity of expectation}\\
      & = \esp{\sum_{k=1}^K \phi(C, e_k)\espk{Q_k}{S} + \espk{h(Q)}{S}}
      & \text{$\phi(C, e_k)$ is a function of $S$}\label{eq:smeasurable}\\
      & = \esp{\sum_{k=1}^K \phi(C, e_k)C_k + \espk{h(Q)}{S}}
      & C_k = \espk{Q_k}{S}\\
      & = \esp{-h(C) + \espk{h(Q)}{S}}
      & \text{Definition of $h$}\\
      & = \esp{\espk{h(Q)}{S} - h(\espk{Q}{S})}
      & C = \espk{Q}{S}\\
      & = \esp{\vark[h]{Q}{S}}
      & \text{Definition of $\vark[h]{Q}{S}$}
  \end{align}
\end{proof}

\subsection{Theorem~\ref{th:GL_decomp}: Grouping loss decomposition}
\label{th:GL_decomp:proof}

\begin{lemma}[Law of total $h$-variance]
  \label{lemma:law-tot-evar}
  Let $X, Y, Z : \Omega \to \R^d$ be random variables defined on the same probability space and a function $f : \R^d \to \R$. The law of total variance holds for the $f$-variance:\looseness=-1
  \begin{equation}
      \vark[f]{Y}{Z} = \espk{\vark[f]{Y}{X,Z}}{Z} + \vark[f]{\espk{Y}{X,Z}}{Z}
  \end{equation}
\end{lemma}
\begin{proof}
  \begin{align*}
      \esp{f(Y)} & = \esp{\espk{f(Y)}{X}}
      & \text{Law of total expectation}\\
      & = \esp{\vark[f]{Y}{X}} + \esp{f(\espk{Y}{X})}
      & \text{Definition of $\vark[f]{Y}{X}$}\\
      &&\nonumber\\
      \esp{f(Y)} - f(\esp{Y}) & = \esp{\vark[f]{Y}{X}} + \esp{f(\espk{Y}{X})} - f(\esp{\espk{Y}{X}})
      & \text{Law of total expectation}\\
      & = \esp{\vark[f]{Y}{X}} + \var[f]{\espk{Y}{X}}
      & \text{Definition of $\var[f]{\espk{Y}{X}}$}
  \end{align*}
  The same proof holds when the expectations and $h$-variances are conditioned on $Z$.
\end{proof}

\gldecomp*
\begin{proof}[Proof of \autoref{th:GL_decomp}]
  Applying \autoref{lemma:law-tot-evar} with $(\Ccal, Q, S)$ as $(X, Y, Z)$ gives the decomposition.
  Proper scoring rules have a convex negative entropy $h$ \citep[see][th. 1]{Gneiting2007}. Note that depending on the convention (maximization or minimization of scoring rules), one may find in the litterature that the entropy is either convex or concave. In the convention taken by this article (minimization of scoring rules), the entropy is concave and the negative entropy is convex. Using Jensen's inequality, we thus have $\vark[h]{Q}{S,\Ccal} \geq 0$. Hence both $\glexp$ and $\glres$ are positive, which gives $\GL \geq \glexp$.
\end{proof}

\subsection{Proposition~\ref{prop:glbin}: Binning-induced grouping loss}
\label{prop:glbin:proof}

\bininduced*
\begin{proof}[Proof of \autoref{prop:glbin}]
  \begin{align*}
      \vark[h]{Q}{S_B}
      & = \espk{\vark[h]{Q}{S,S_B}}{S_B} + \vark[h]{\espk{Q}{S,S_B}}{S_B}
      && \text{Law of total $h$-variance (\autoref{lemma:law-tot-evar})}
      \\
      & = \espk{\vark[h]{Q}{S}}{S_B} + \vark[h]{\espk{Q}{S}}{S_B}
      && \text{$S_B$ is a function of $S$}
      \\
      & = \espk{\vark[h]{Q}{S}}{S_B} + \vark[h]{C}{S_B}
      && C = \espk{Q}{S}\\
      &&&\nonumber\\
      \esp{\vark[h]{Q}{S_B}}
      & = \esp{\vark[h]{Q}{S}} + \esp{\vark[h]{C}{S_B}}
      && \text{Law of total expectation}\\
      \GL(S_B)
      & = \GL(S) + \GLind(S,S_B)
      && \text{\autoref{prop:e-variance} and definition of $\GLind$}
  \end{align*}
  Remark: this proposition does not require $S_B$ to be the average scores on the bins $\espk{S}{S\in\Bc_j}$.
\end{proof}

\subsection{Proposition~\ref{prop:gllb-binning}: Explained grouping loss accounting for binning}
\label{prop:gllb-binning:proof}

\binaccounting*
\begin{proof}[Proof of \autoref{prop:gllb-binning}]
  \begin{align*}
    \GL(S)
    & = \GL(S_B) - \GLind(S, S_B)
    & \text{Propostion~\ref{prop:glbin}}\\
    & = \glexp(S_B) + \glres(S_B) - \GLind(S, S_B)
    & \text{\autoref{th:GL_decomp} on $\GL(S_B)$}
  \end{align*}
  For proper scoring rules, \autoref{th:GL_decomp} gives $\glres(S_B) \geq 0$ which completes the proof.
\end{proof}

\subsection{Proposition~\ref{prop:debiasing}: Debiased estimator for the Brier score}
\label{ss:proof_debiasing}
\debiasing*
\begin{proof}
Let $s \in \Scal^{(k)}$ and $k \in \cbr{1, K}$, and define $\hat p_j^{(s, k)}:= \frac{n_j^{(s, k)}}{n^{(s, k)}}$.
We now compute the bias of the plugin estimator for $\glexphat^{(s, k)}$. To ease calculations, we start by rewriting the plugin estimate:
\begin{align}
    \glpluginhat^{(s, k)} &= \sum_{j=1}^J \hat p_j^{(s, k)} \br{\hat \mu_j^{(s, k)} - \hat c^{(s, k)}}^2\\
    &= \sum_{j=1}^J \hat p_j^{(s, k)} \br{\hat \mu_j^{(s, k)}}^2 - 2\hat c^{(s, k)} \br{\sum_{j=1}^J \hat p_j^{(s, k)} \hat \mu_j^{(s, k)}} + \br{\hat c^{(s, k)}}^2\\
    &= \sum_{j=1}^J \hat p_j^{(s, k)} \br{\hat \mu_j^{(s, k)}}^2 - \br{\hat c^{(s, k)}}^2
\end{align}

From now on, we omit the exponent $(s, k)$ to lighten notations. We now take the expectation of both terms in the lower-bound.

\begin{align}
    \E \sqb{\hat c^2} &= \E \sqb{\hat c}^2 + \text{Var}(\hat c)\\
    &= c^2 + \frac{c(1-c)}{n} \label{eq:part1}
\end{align}
where we made use of \autoref{lem:exp-var} for equation~\ref{eq:part1}. Similarly,
\begin{align}
    \E \sqb{\hat \mu_j^2 \cond \hat p_j} &= \E \sqb{\hat \mu_j}^2 + \text{Var}(\hat \mu_j)\\
    &= \mu_j^2 + \frac{\mu_j(1-\mu_j)}{n_j} \label{eq:undefined}
\end{align}
When $n_j = 0$ (or equivalently $\hat p_j = 0$), which happens with probability $\nu_j = (1-p_j)^n$, $\hat \mu_j$ as well as the right term in equation~\ref{eq:undefined} are undefined. The problem disappears when multiplying by $\hat p_j$, and agreeing that $\hat \mu_j = 0$ whenever $n_j = 0$.

\begin{align}
    \E \sqb{\sum_{j=1}^J \hat p_j \hat \mu_j^2} &= \sum_{j=1}^J \E \sqb{ \E \sqb{\hat p_j \hat \mu_j^2 \mathbb{1}_{\hat p_j \ge 0} \cond \hat p_j}}\\
    &= \sum_{j=1}^J \E \sqb{\hat p_j \mathbb{1}_{\hat p_j \ge 0} \br{\mu_j^2 + \frac{\mu_j(1-\mu_j)}{n_j}}}\\
    &= \sum_{j=1}^J \br{p_j \mu_j^2 + \E\sqb{\mathbb{1}_{\hat p_j \ge 0} \frac{\mu_j(1 - \mu_j)}{n}}}\\
    &= \sum_{j=1}^J \br{p_j \mu_j^2 + (1-\nu_j) \frac{\mu_j(1-\mu_j)}{n}} \label{eq:part2}
\end{align}

Putting together equations~\ref{eq:part1} and \ref{eq:part2}, we get:
\begin{align}
    \E \sqb{\glpluginhat} &= \sum_{j=1}^J p_j \mu_j^2 - c^2 + \sum_{j=1}^J (1-\nu_j) \frac{\mu_j(1-\mu_j)}{n} - \frac{c(1-c)}{n}\\
    & = \underbrace{\sum_{j=1}^J p_j (\mu_j - c)^2}_{\glexp} + \underbrace{\sum_{j=1}^J (1-\nu_j) \frac{\mu_j(1-\mu_j)}{n} - \frac{c(1-c)}{n}}_{\glbias} \label{eq:bias}
\end{align}

In practice $\nu_j$, which gives the probability that no sample falls in component $j$, is very close to 0 unless $p_j$ and $n$ are very small. Hence, we will approximate $\nu_j \approx 0$. More importantly, the expression of the bias given in \ref{eq:bias} depends on oracle quantities $\mu_j$ and $c$, which are unavailable. Therefore, we resort to debiasing the plugin estimate of the lower-bound using sample estimates of the bias, which gives:

\begin{equation}
    \glexphat^{(s, k)} = \underbrace{\sum_{j=1}^J \frac{n^{(s, k)}_j}{n^{(s, k)}} \br{\hat \mu^{(s, k)}_j - \hat c^{(s, k)}}^2}_{\text{\small plugin estimator $\glpluginhat$}} - \sum_{j=1}^J \frac{n_j^{(s, k)}}{n^{(s, k)}} {\frac{\hat \mu^{(s, k)}_j (1-\hat \mu^{(s, k)}_j)}{n^{(s, k)}_j - 1}} + \frac{\hat c^{(s, k)}(1-\hat c^{(s, k)})}{n^{(s, k)} - 1}
\end{equation}
where we used a Bessel correction for the estimation of population variances. Finally, a debiased estimator of $\glexp$ is obtained by summing over the debiased estimators for all $k \in \cbr{1, K}$ and all $s \in \Scal_k$.
\end{proof}

\begin{lemma}
\label{lem:exp-var}
    Define $\hat \mu_j^{(s, k)}$ and $\hat c^{(s, k)}$ as in \autoref{prop:debiasing}. Then:
    \begin{equation}
        \label{eq:mu_hat_moments}
        \E \sqb{\hat \mu_j^{(s, k)}} = \mu_j^{(s, k)} \quad \text{and} \quad Var\br{\hat \mu_j^{(s, k)}} = \frac{\mu_j^{(s, k)}\br{1-\mu_j^{(s, k)}}}{n_j^{(s, k)}}.
    \end{equation}
    Similarly,
    \begin{equation}
        \label{eq:c_hat_moments}
        \E \sqb{\hat c^{(s, k)}} = c^{(s, k)} \quad \text{and} \quad Var\br{\hat c^{(s, k)}} = \frac{c^{(s, k)}\br{1-c^{(s, k)}}}{n^{(s, k)}}.
    \end{equation}
\end{lemma}

The labels $Y_k^{(i)}$ are by definition drawn from a Bernoulli distribution with probability ${P(Y_k^{(i)} | X^{(i)}) = Q_k^{(i)}}$, \ie{}, for each sample $i$, the probability of the Bernoulli changes.

This lemma shows that despite these varying Bernoulli probabilities, the empirical average of labels $Y_k$ over a given subspace has the same expectation and variance as a binomial variable that would be drawn with a probability equal to the expectation of $Q_k$ over this subspace.

\begin{proof}[Proof of \autoref{lem:exp-var}]
Below we write the proof for the case of $\hat \mu_j^{(s, k)}$ (equation~\ref{eq:mu_hat_moments}) as the one for $\hat c^{(s, k)}$ (equation~\ref{eq:c_hat_moments}) follows exactly the same lines. Let $\Ical^{(s)}_j = \cbr{i: X^{(i)} \in \Rcal_j^{(s)}}$, be the subset of samples such that $X^{(i)}$ belongs to bin $\Rcal_j^{(s)}$.

\begin{align}
    \E \sqb{\hat \mu_j^{(s, k)}} &= \frac{1}{n_j^{(s, k)}}\sum_{i \in \Ical^{(s)}_j} \E \sqb{Y_k^{(i)} \mid S_k=s, \Ccal(X^{(i)}) = j}\\
    &= \frac{1}{n_j^{(s, k)}}\sum_{i \in \Ical^{(s)}_j} \E \sqb{ \E \sqb{Y_k^{(i)} \mid X^{(i)}} \mid S_k=s, \Ccal(X^{(i)}) = j} \label{eq:total_exp}\\
    &= \frac{1}{n_j^{(s, k)}} \sum_{i \in \Ical^{(s)}_j} \E \sqb{Q_k^{(i)} \mid S_k=s, \Ccal(X^{(i)}) = j} \label{eq:def_Q}\\
    &= \frac{1}{n_j^{(s, k)}} \sum_{i \in \Ical^{(s)}_j} \mu_j^{(s, k)} \label{eq:def_mu}\\
    &= \mu_j^{(s, k)}
\end{align}
where we used the law of total expectation in eq~\ref{eq:total_exp}, the definition of $Q_k$ in eq~\ref{eq:def_Q}, and the definition of $\mu_j^{(s, k)}$ in eq~\ref{eq:def_mu}.

\begin{align}
    \text{Var} \br{\hat \mu_j^{(s, k)}} &= \E \sqb{(\hat \mu_j^{(s, k)} - \mu_j^{(s, k)} )^2 \cond S_k=s, \Ccal(X^{(i)}) = j}\\
    &= \E \sqb{\br{\hat \mu_j^{(s, k)}}^2 \cond S_k=s, \Ccal(X^{(i)}) = j} - \br{\mu_j^{(s, k)}}^2\\
    &= \frac{1}{\br{n_j^{(s, k)}}^2} \E \sqb{\sum_{i \in \Ical^{(s)}_j} Y^{(i)} \sum_{l \in \Ical^{(s)}_j} Y_k^{(l)} \cond S_k=s, \Ccal(X^{(i)}) = j}  - \br{\mu_j^{(s, k)}}^2\\
    &= \frac{1}{\br{n_j^{(s, k)}}^2} \E \sqb{\sum_{i \in \Ical^{(s)}_j} Y_k^{(i)} + \sum_{\substack{i \ne l\\ i, l \in \Ical^{(s)}_j}} Y_k^{(i)} Y_k^{(l)} \cond S_k=s, \Ccal(X^{(i)}) = j}  - \br{\mu_j^{(s, k)}}^2 \label{eq:indep}\\
    &= \frac{1}{\br{n_j^{(s, k)}}^2} \br{\sum_{i \in \Ical^{(s)}_j} \mu^{(s, k)}_j + \sum_{\substack{i \ne l\\ i, l \in \Ical^{(s)}_j}} \br{\mu^{(s, k)}_j}^2 } - \br{\mu_j^{(s, k)}}^2\\
    &= \frac{1}{\br{n_j^{(s, k)}}^2} \br{n_j^{(s, k)} \mu^{(s, k)}_j + n_j^{(s, k)}(n_j^{(s, k)} - 1) \br{\mu_j^{(s, k)}}^2} - \br{\mu_j^{(s, k)}}^2\\
    &= \frac{\mu_j^{(s, k)} (1-\mu_j^{(s, k)})}{n_j^{(s, k)}}
\end{align}
where we used the fact that $Y_k^{(i)}$ and $Y_k^{(l)}$ are independent when $i \ne l$ in eq~\ref{eq:indep}.
\end{proof}

\subsection{The plugin estimator for the grouping loss lower bound is biased upwards.}
\label{ss:proof_bias_upwards}

\paragraph{Analytical evaluation of the sign of the bias} Let $k \in \cbr{1, \dots, K}$ and $s \in \Scal$. The bias of the plugin estimate $\glexphat^{(s, k)}(S_B)$ is given by (\ref{eq:bias}):
\begin{align}
    \text{bias}\br{\glexphat^{(s, k)}(S_B)} = \sum_{j=1}^J \br{1-\nu_j^{(s, k)}} \frac{\mu_j^{(s, k)}(1-\mu_j^{(s, k)})}{n^{(s, k)}} - \frac{c^{(s, k)}(1-c^{(s, k)})}{n^{(s, k)}}
\end{align}


By convexity of the function $x \mapsto \br{x - \E \sqb{x}}^2$, we have:
\begin{equation}
    \br{\sum_{j=1}^J \frac{n_j^{(s, k)}}{n^{(s, k)}} \hat \mu_j^{(s, k)} - \E \sqb{\sum_{j=1}^J \frac{n_j^{(s, k)}}{n^{(s, k)}} \hat \mu_j^{(s, k)}}}^2 \leq \sum_{j=1}^J \frac{n_j^{(s, k)}}{n^{(s, k)}} \br{\hat \mu_j^{(s, k)} - \E \sqb{\hat \mu_j^{(s, k)}}}^2
\end{equation}

Using the fact that $\hat c^{(s, k)} = \sum_{j=1}^J \frac{n_j^{(s, k)}}{n^{(s, k)}} \hat \mu_j^{(s, k)}$, and taking the expectation of both sides, we get:
\begin{equation}
    \text{Var}(\hat c^{(s, k)}) \leq \sum_{j=1}^J \frac{n_j^{(s, k)}}{n^{(s, k)}} \text{Var}(\hat \mu_j^{(s, k)})
\end{equation}

Finally, using \autoref{lem:exp-var}, we get:
\begin{equation}
    \frac{c^{(s, k)}\br{1-c^{(s, k)}}}{n^{(s, k)}} \leq \sum_{j=1}^J \frac{\mu_j^{(s, k)}\br{1-\mu_j^{(s, k)}}}{n^{(s, k)}}
\end{equation}

Hence, we have:
\begin{align}
    \label{eq:bias_sk}
    \text{bias}\br{\widehat \Lcal_{GL}^{(s, k)}} = \underbrace{\sum_{j=1}^J \frac{\mu_j^{(s, k)}(1-\mu_j^{(s, k)})}{n^{(s, k)}} - \frac{c^{(s, k)}(1-c^{(s, k)})}{n^{(s, k)}}}_{\geq 0} - \sum_{j=1}^J \nu_j^{(s, k)} \frac{\mu_j^{(s, k)}(1-\mu_j^{(s, k)})}{n^{(s, k)}}
\end{align}

Because of the term involving $\nu_j^{(s, k)}$, this inequality does not prove that the bias is always positive. However in practice $\nu_j^{(s, k)} = \br{1-p_j^{(s, k)}}^{n^{(s, k)}}$, which represents the probability that no point belongs to region $j$, is very close to 0 unless $p_j^{(s, k)}$ is very small or the total number of points $n^{(s, k)}$ is small. Hence, equality~\ref{eq:bias_sk} shows that the bias can only be 'slightly' negative. In the simulations below, the upwards bias of the plugin estimate appears clearly.

\subsection{Estimator for the induced grouping loss}
\label{sec:glind-est}
\begin{proposition}[Estimator for the induced grouping loss]
  \label{prop:glind-est}
Let $\hat{C}$ be an estimator of $C$.
An estimator of $C_B$ is $\hat{C}_B(s) = \frac{1}{n^{(s)}} \sum_{i:S_B(X^{(i)}) = s} \hat{C}(S(X^{(i)}))$ with $n^{(s)}$ the number of sample in the level set $s$.
An estimator of the grouping loss induced by the binning of $S$ into $S_B$ is:
\begin{equation}
  \GLindhat(S,S_B)
  =
  \sum_{s \in \Scal} \frac{n^{(s)}}{n} \bra{
    \frac{1}{n^{(s)}} \sum_{i:S_B(X^{(i)}) = s} e(\hat{C}(S(X^{(i)})) - e(\hat{C}_B(s))
  }
\end{equation}
\end{proposition}

\subsection{Analysis of binning-induced errors for the Brier score}
\label{sec:clglinduced}

It is well known that binning can induce error in estimating calibration loss, leading to underestimating it \citep{Brocker2012Estimating,Kumar2019verified,roelofs2022mitigating}. \autoref{prop:glbin} shows that it also leads to errors on the grouping loss, overestimating it. Here we characterize the errors on the calibration and grouping loss for the Brier score and show that they partly compensate each other and the error on the sum of both can be bounded.

\autoref{prop:clbin} gives the deviation term induced by the binning for the calibration loss with the Brier scoring rule.

\begin{proposition}[Calibration loss decomposition]
    \label{prop:clbin}
    Let $h$ be the negative entropy of the Brier scoring rule and $C = \espk{Q}{S}$. The binned calibration loss $\CL(S_B)$ deviates from the calibration loss $\CL(S)$ by a negative induced calibration loss $\CLind(S, S_B)$:
    \begin{align}
        \underbrace{\esp{\|S_B - C_B\|^2}}_{\textstyle\CL(S_B)}
        & \; = \; \underbrace{\esp{\|S-C\|^2}}_{\textstyle\CL(S)} \;\, \underbrace{- \;\; \esp{\vark[h]{S-C\vphantom{\|S-C\|^2}}{S_B}\!}}_{\textstyle\CLind(S,S_B)}
    \end{align}
\end{proposition}


The calibration loss induced by the binning, $\CLind(S, S_B)$, is always negative. $\CL(S_B)$ is thus biased downward, which is already known from \cite{Kumar2019verified,roelofs2022mitigating}.
Conversely, the grouping loss induced by the binning, $\GLind(S, S_B)$, is always positive. $\GL(S_B)$ is thus biased upward.
The mere effect of binning artificially creates grouping loss and artificially reduces calibration error.
For calibrated continuous classifiers, $\CLind=0$ and induced grouping loss is small: with $N$ equal-width bins, $\GLind \leq \tfrac{1}{4N^2}$. If in addition the scores are uniform on the bins: $\GLind = \tfrac{1}{12N^2}$ (\autoref{lem:GLbound-n}). Both induced calibration and grouping losses can be large since $\vark{C}{S_B}$ can be large. High $\GLind$ expresses strong miscalibrations within the bin. However interestingly, both induced losses compensate. In a binary setting, the sum of induced calibration and grouping losses is contained as showed by \autoref{thm:clgl-bounds}, and can be bounded by estimable quantities (\autoref{cor:bounds}). While measuring $\CL(S_B)$ and $\GL(S_B)$ separately can lead to high binning-induced bias, measuring $\CL(S_B) + \GL(S_B)$ through $\clhat(S_B) + \glexphat(S_B)$ enables reducing binning-induced errors and minorizing $\mathrm{MSE(S, Q)}$ (\autoref{cor:mse}).

\begin{theorem}[Bounds on induced calibration and grouping losses]
    \label{thm:clgl-bounds}
    In a binary setting, the calibration and grouping losses induced by the binning of classifier $S$ into $S_B$ sums to:
    \begin{align*}
        \CLind + \GLind & = \esp{2\covk{S}{C}{S_B} - \vark{S}{S_B}}
    \end{align*}
    which is bounded by:
    \begin{align*}
        -\esp{\sqrt{\vark{S}{S_B}}\pr{2\sqrt{\vark{C}{S_B}} + \sqrt{\vark{S}{S_B}}}}
        & \leq \CLind + \GLind
        \\[-.8ex]
        &
        \qquad\leq \esp{\sqrt{\vark{S}{S_B}}\pr{2\sqrt{\vark{C}{S_B}} - \sqrt{\vark{S}{S_B}}}}
    \end{align*}
    Suppose that $[0, 1]$ is divided in $N$ equal-width bins. Then:
    \begin{align*}
        -\tfrac{1}{N}\esp{\sqrt{C_B(1-C_B)}} -\tfrac{1}{4N^2}
        & \leq \CLind + \GLind
        \leq \tfrac{1}{N}\esp{\sqrt{C_B(1-C_B)}}
    \end{align*}
\end{theorem}
\begin{corollary}
  \label{cor:bounds}
    \begin{align*}
      -\esp{\sqrt{\vark{S}{S_B}}\pr{2\sqrt{C_B(1-C_B)} + \sqrt{\vark{S}{S_B}}}}
      & \leq \CLind + \GLind
      \\[-.8ex]
      &
      \qquad \leq \esp{\sqrt{\vark{S}{S_B}}\pr{2\sqrt{C_B(1-C_B)} - \sqrt{\vark{S}{S_B}}}}
    \end{align*}
    \text{With $N$ equal-width bins: \quad }
    \begin{align*}
        -\tfrac{1}{N}\esp{\sqrt{C_B(1-C_B)}} -\tfrac{1}{4N^2}
        & \leq \CLind + \GLind
        \leq \tfrac{1}{N}\esp{\sqrt{C_B(1-C_B)}}
    \end{align*}
\end{corollary}
\begin{corollary}
  \label{cor:mse}
    The mean square error (MSE) between continuous $S$ and $Q$ is lower bounded by:
    \begin{align*}
        \mathrm{MSE}(S, Q)
        & = \CL + \GL\\
        & \geq \ell^2\text{-}\mathrm{ECE}_B + \Lcal_{\GL_B} - \esp{\sqrt{\vark{S}{S_B}}\pr{2\sqrt{\vark{C}{S_B}} - \sqrt{\vark{S}{S_B}}}}\\
        & \geq \ell^2\text{-}\mathrm{ECE}_B + \Lcal_{\GL_B} - \esp{\sqrt{\vark{S}{S_B}}\pr{2\sqrt{C_B(1-C_B)} - \sqrt{\vark{S}{S_B}}}}\\
        \text{With $N$ equal bins:}
        & \geq \ell^2\text{-}\mathrm{ECE}_B + \Lcal_{\GL_B} - \tfrac{1}{N}\esp{\sqrt{C_B(1-C_B)}}
    \end{align*}
    where $\ell^2\text{-}\mathrm{ECE}_B$ is the $\ell^2$ Expected Calibration Error of the binned classifier $S_B$ and $\Lcal_{\GL_B}$ is the grouping loss lower bound of $S_B$.
\end{corollary}

\subsection*{Proofs}
\begin{proof}[Proof of \autoref{prop:clbin}]
  Let $h$ the negative entropy of the Brier scoring rule.
    \begin{flalign*}
        \|S_B - C_B\|^2
        & = \|\espk{S}{S_B} - \espk{C}{S_B}\|^2
        && S_B = \espk{S}{S_B}, C_B = \espk{C}{S_B}\\
        & = \|\espk{S-C}{S_B}\|^2
        && \text{Linearity of expectation}
        \\
        & = \espk{\|S-C\|^2}{S_B} - \vark[h]{S-C}{S_B}
        && \text{Definition of $\vark[h]{S-C}{S_B}$}\\
        &&&\nonumber\\
        \esp{(S_B - C_B)^2}
        & = \esp{\|S-C\|^2} - \esp{\vark[h]{S-C}{S_B}}
        && \text{Law of total expectation}
    \end{flalign*}
\end{proof}

\begin{lemma}
  \label{lem:GLbound-n}
  In a binary setting, suppose that $[0, 1]$ is divided in $N$ equal-width bins. Then:
  \begin{align}
    \vark{S}{S_B} & \leq \tfrac{1}{4N^2}\\
    \intertext{If in addition, scores $S$ are uniform:}
    \vark{S}{S_B} & = \tfrac{1}{12N^2}
  \end{align}

\end{lemma}
\begin{proof}[Proof of Lemma~\ref{lem:GLbound-n}]
    Without loss of generality, consider the first bin $[0, \tfrac{1}{N}]$ with binned score $s_1$.
    \begin{flalign*}
      \vark{S}{S_B = s_1}
      & = \espk{S^2}{S_B = s_1} - \espk{S}{S_B = s_1}^2
      && \text{Definition of the variance}
      \\
      & \leq \tfrac{1}{N}\espk{S}{S_B = s_1} - \espk{S}{S_B = s_1}^2
      && 0 \leq S \leq \tfrac{1}{N} \Rightarrow S^2 \leq \tfrac{1}{N}S\\
      & = \tfrac{1}{N^2}(1 - N\espk{S}{S_B = s_1})N\espk{S}{S_B = s_1}\\
      & \leq \tfrac{1}{4N^2}
      && \text{Max when $N\espk{S}{S_B = s_1} = \tfrac{1}{2}$}\nonumber
    \end{flalign*}
    For uniform scores: $S |S_B = s_1 \sim \mathcal{U}([0, \tfrac{1}{N}])$. Hence $\vark{S}{S_B = s_1} = \tfrac{1}{12}(\tfrac{1}{N} - 0)^2 = \tfrac{1}{12N^2}$.

    Other bins have same variance as $\vark{S}{S_B = s_1}$ (variance is translation-invariant).\\
    Remark: this proves that $\GLind \leq \tfrac{1}{4N^2}$ for $S$ calibrated ($S = C \Rightarrow \vark{C}{S_B} = \vark{S}{S_B}$).
\end{proof}

\begin{proof}[Proof of \autoref{thm:clgl-bounds}]
  In a binary setting for the Brier scoring rule, we have $\Var[h] = \Var$. Hence:
    \begin{flalign*}
        \CLind + \GLind
        & = - \esp{\vark[h]{S-C}{S_B}} + \esp{\vark[h]{C}{S_B}}
        && \text{Propositons~\ref{prop:glbin} and \ref{prop:clbin}}\\
        & = - \esp{\vark{S-C}{S_B}} + \esp{\vark{C}{S_B}}
        && \text{$\Var[h] = \Var$}\\
        & = \esp{2\covk{S}{C}{S_B} - \vark{S}{S_B}}
        && \text{Expansion of $\vark{S-C}{S_B}$}\\
        \\
        2\covk{S}{C}{S_B} - \vark{S}{S_B}
        & \leq 2\sqrt{\vark{S}{S_B}}\sqrt{\vark{C}{S_B}} - \vark{S}{S_B}
        && \text{Cauchy-Schwarz}\\
        \\
        2\covk{S}{C}{S_B} - \vark{S}{S_B}
        & \geq -2\abs{\covk{S}{C}{S_B}} - \vark{S}{S_B}\\
        & \geq -2\sqrt{\vark{S}{S_B}}\sqrt{\vark{C}{S_B}} - \vark{S}{S_B}
        && \text{Cauchy-Schwarz}
    \end{flalign*}
    \begin{flalign*}
        \intertext{With $N$ equal-width bins:}
        2\covk{S}{C}{S_B} - \vark{S}{S_B}
        & \leq 2\sqrt{\vark{S}{S_B}}\sqrt{\vark{C}{S_B}}
        && \text{Positivity of the variance}\\
        & \leq \tfrac{1}{N}\sqrt{\vark{C}{S_B}}
        && \vark{S}{S_B} \leq \tfrac{1}{4N^2}
        \\
        & \leq \tfrac{1}{N}\sqrt{C_B(1-C_B)}
        && \vark{C}{S_B} \leq C_B(1-C_B)\\
        \\
        2\covk{S}{C}{S_B} - \vark{S}{S_B}
        & \geq -\tfrac{1}{N}\sqrt{\vark{C}{S_B}} - \tfrac{1}{4N^2}
        && \vark{S}{S_B} \leq \tfrac{1}{4N^2}
        \\
        & \geq -\tfrac{1}{N}\sqrt{C_B(1-C_B)} - \tfrac{1}{4N^2}
        && \vark{C}{S_B} \leq C_B(1-C_B)
    \end{flalign*}
\end{proof}

\subsection{Extension to classwise calibration}
\label{sec:classwise-extension}
\subsubsection{Proper scoring rules decomposition}
\label{sec:decomposition}

We show below that the proper scoring rules decomposition of \cite{Kull2015Novel} holds for classwise-calibration (\autoref{def2}) for the Brier score and the log-loss.

\begin{restatable}[Brier and log-loss classwise decomposition]{proposition}{decomposition}
\label{prop:decomposition}
    For the Brier score as well as the log-loss, the decomposition into calibration, grouping, and irreducible losses (\autoref{eq:decomp}) holds when replacing the calibrated scores by the classwise-calibrated scores (\autoref{def2}).
\end{restatable}


\begin{proof}[Proof of \autoref{prop:decomposition}]
For all $k \in \cbr{1, \dots, K}$, let $C_k=\E \sqb{Y_k|S_k}$ be the classwise-calibrated scores (\autoref{def2}).
\paragraph{Brier Score} Given any two probability vectors $P$ and $Q$, the divergence associated to the Brier score reads:
\begin{equation}
    d(P, Q) = \sum_{k=1}^K (P_k - Q_k)^2
\end{equation}
For all $k \in \cbr{1, \dots, K}$, let $d_k: P_k, Q_k \mapsto (P_k - Q_k)^2$.
\begin{align}
    d_k(S_k, Y_k) & =(S_k - Y_k)^2\\
    &= (S_k - C_k + C_k - Q_k + Q_k - Y_k)^2\\
    \begin{split}
    &= (S_k - C_k)^2 + (C_k - Q_k)^2 + (Q_k - Y_k)^2 + 2(S_k - C_k)(C_k - Q_k)\\
    & \quad \quad  + 2(S_k - C_k)(Q_k - Y_k) + 2(C_k - Q_k)(Q_k - Y_k)
    \end{split}
\end{align}

Taking the expectation on both sides conditional on $X$:
\begin{equation}
    \begin{split}
        \E \sqb{d_k(S_k, Y_k) \cond X} = (S_k - C_k)^2 + (C_k - Q_k)^2 + \E \sqb{(Q_k - Y_k)^2 \cond X}\\
        + 2(S_k - C_k)(C_k - Q_k)
    \end{split}
\end{equation}

since $S_k$ and $Q_k$ are function of $X$, $C_k$ is a function of $S_k$ and thus of $X$, and $\E \sqb{Y_k\cond X} = Q_k$. Then taking the expectation conditional on $S_k$:
\begin{equation}
    \E \sqb{d_k(S_k, Y_k) \cond S_k} = (S_k - C_k)^2 + \E \sqb{(C_k - Q_k)^2 \cond S_k} + \E \sqb{(Q_k - Y_k)^2 \cond S_k}
\end{equation}
where we use the fact that $C_k$ is a function of $S_k$, that $\E \sqb{Q_k \cond S_k} = C_k$, and the property according to which for two random variables $U$ and $V$ and a function $h$, $\E \sqb{\E \sqb{V \cond U} \cond h(U)} = \E \sqb{V \cond h(U)}$.
Finally, taking the expectation over $S_k$ we get:
\begin{equation}
    \E \sqb{d_k(S_k, Y_k)} = \E \sqb{(S_k - C_k)^2} + \E \sqb{(C_k - Q_k)^2} + \E \sqb{(Q_k - Y_k)^2}
\end{equation}
The desired decomposition is then obtained by summing over the K classes on both sides.

\paragraph{log-loss} Given any two probability vectors $P$ and $Q$, the divergence associated to the log loss reads:
\begin{equation}
    d(P, Q) = \sum_{k=1}^K Q_k \log \br{\frac{Q_k}{P_k}}
\end{equation}
For all $k \in \cbr{1, \dots, K}$, let $d_k: P_k, Q_k \mapsto Q_k \log \br{\frac{Q_k}{P_k}}$.
\begin{align}
    d_k(S_k, Y_k) &= Y_k \log \br{\frac{Y_k}{S_k}}\\
    &= Y_k \log \br{\frac{Y_k}{Q_k}} + Y_k \log \br{\frac{Q_k}{C_k}} + Y_k\log \br{\frac{C_k}{S_k}}\\
    \E \sqb{d_k(S_k, Y_k) \cond X} &= \E \sqb{Y_k \log \br{\frac{Y_k}{Q_k}} \cond X} + Q_k \log \br{\frac{Q_k}{C_k}} + Q_k\log \br{\frac{C_k}{S_k}}\\
    \E \sqb{d_k(S_k, Y_k) \cond S_k} &= \E \sqb{Y_k \log \br{\frac{Y_k}{Q_k}} \cond S_k} + \E \sqb{Q_k \log \br{\frac{Q_k}{C_k}} \cond S_k} + C_k \log \br{\frac{C_k}{S_k}}\\
    &= \E\sqb{d_k(Q_k, Y_k) \cond S_k} + \E\sqb{d_k(C_k, Q_k) \cond S_k} + d_k(S_k, C_k)
\end{align}
where we have used the same properties as those described for the proof of the Brier score classwise decomposition above. The desired decomposition is then obtained by taking the expectation over $S_k$ and summing over the K classes.

\end{proof}

\paragraph{The proper scoring rule decomposition holds for top-label calibration.}

\looseness=-1
Unlike classwise calibration, top-label calibration does not define a vector $C \in \RR^K$ of calibrated probabilities. Instead, it defines a notion of calibration for a simpler binary problem in which labels indicate whether the classifier predicts the correct class for a given $X$. More precisely, the labels for this binary problem are given by $Y^\prime:=1_{Y = e_{\arg \max(S)}}$. Since $S$ is a function of $X$, the random variable $Y^\prime$ is a function of $Y$ and $X$. Define now the scores associated to this binary problem as $S^\prime := \max(S) \in \R$. Reformulated in terms of these notations, top-label calibration states that $S^\prime$ is well calibrated if for all $s, P(Y^\prime=1|S^\prime=s) = s$. Thus, as for a classical binary problem, we can define $C^\prime := \E[Y^\prime|S^\prime]$ and $Q^\prime = \E \sqb{Y^\prime|X}$. $C^\prime$ (resp $Q^\prime$) gives the probability that the classifier predicts the correct class for a given score $S^\prime$ (resp. a given input $X$). As the quantities $S^\prime$, $C^\prime$, $Q^\prime$ and $Y^\prime$ define a classical binary problem, the decomposition~(\ref{eq:decomp}) into calibration, grouping, and irreducible loss holds for this problem. Compared to the classwise definition of calibration and grouping, here the calibration loss measures whether on average over all points scored $S$ \textit{across all classes}, the proportion of correctly predicted points in actually $S$. In this setting, the grouping loss also measures to what extent there exist over-confident scores for certain classes that compensate under-confident scores for other classes.

\subsubsection{Results hold for Brier and log-loss in classwise setting}
\autoref{sec:decomposition} proves the scoring rule decomposition~(\ref{eq:decomp}) in a classwise setting for Brier and log-loss scoring rules, which is necessary for the other results to hold. However, the proof of \autoref{prop:e-variance} does not readily apply to classwise calibration. \autoref{eq:smeasurable} uses a conditioning on the full vector of joint confiences $S$ to move $\phi(C, e_k)$ outside of the conditional expectation on $S$ and turn $Q_k$ into $C_k$ in expectation. In classwise calibration the conditioning is on each marginal $S_k$ instead of the joint $S$. As a result, in the general case, $\phi(C, e_k)$ cannot be moved outside of the conditional expectation given $S_k$ since $C$ depends on all marginals of $C$, not just $C_k$. However for some scoring rules, $\phi(p, e_k)$ depends only on $p_k$ and the proof can be adapted. This is the case of the log-loss for which $\phi^{LL}(p, e_k) = -\log(p_k)$.

\begin{restatable}[Adaptation of \autoref{prop:e-variance} for classwise calibration]{lemma}{evarianceclasswise}
      \label{prop:e-variance:classwise}
      Suppose there exists $g : \R^K \to \R$ such that for all $k$ in $\{1, \dots, K\}$ and $x$ in $\R^K$, $\phi(x, e_k) = g(x_k)$.
     Define $h_k : p \mapsto -\phi(p, e_k)p_k$, the $k$\textsuperscript{th} component of the negative entropy of the scoring rule $\phi$.
     The grouping loss $\GL$ of the classifier $S$ with calibrated scores $C_k = \espk{Q_k}{S_k}$ and scoring rule $\phi$ writes:
     \begin{equation}
      \underbrace{\esp{d_{\phi}(C,Q)}}_{\GL(S)}\;=\;\textstyle\sum_{k=1}^K\esp{\vark[h_k]{Q_k}{S_k}}
      \label{eq:gl-var:classwise}
     \end{equation}
  \end{restatable}
\begin{proof}[Proof of \autoref{prop:e-variance:classwise}]
Define the vector $C$ with $C_k = \espk{Q_k}{S_k}$ for all $k$ in $\{1, \dots, K\}$. Let $\phi$ be a scoring rule, $h : p \mapsto -s_{\phi}(p,p)$. Suppose for all $k$ in $\{1, \dots, K\}$, $\phi(x, e_k) = g(x_k)$ with $g : \R^K \to \R$. Then:
  \begin{align}
      \esp{d_{\phi}(C, Q)}
      & = \esp{s_{\phi}(C, Q) - s_{\phi}(Q, Q)}
      & \text{Definition of $d_{\phi}$}\\
      & = \textstyle\esp{\sum_{k=1}^K \phi(C, e_k)Q_k + \sum_{k=1}^K \phi(Q, e_k)Q_k}
      & \text{Definition of $s_{\phi}$}\\
      & = \textstyle \sum_{k=1}^K\esp{\phi(C, e_k)Q_k - \phi(Q, e_k)Q_k}
      & \text{Linearity of expectation}\\
      & = \textstyle \sum_{k=1}^K\esp{g(C_k)Q_k - g(Q_k)Q_k}
      & \text{Hypothesis on $\phi$}\\
      & = \textstyle \sum_{k=1}^K\esp{\espk{g(C_k)Q_k}{S_k} - \espk{g(Q_k)Q_k}{S_k}}
      & \text{Law of total expectation}\\
      & = \textstyle \sum_{k=1}^K\esp{g(C_k)\espk{Q_k}{S_k} - \espk{g(Q_k)Q_k}{S_k}}
      & \text{$C_k$ is a function of $S_k$}\\
      & = \textstyle \sum_{k=1}^K\esp{g(C_k)C_k - \espk{g(Q_k)Q_k}{S_k}}
      & \text{Definition of $C_k$}\\
      & = \textstyle \sum_{k=1}^K\esp{\vark[h_k]{Q_k}{S_k}}
      & \text{Definition of $\Var[h_k]$}
  \end{align}
\end{proof}

\begin{restatable}[Results in classwise setting]{theorem}{resultsclasswise}
      \label{th:results:classwise}
    Suppose \autoref{eq:gl-var:classwise} is satisfied for the scoring rule $\phi$.
     For all $k \in \{1, \dots, K\}$, let $\Ccal_k: \Xc \to \N$ be a partition of the feature space.
     It holds that:
     \begin{align}
      \GL(S)
      & =\;\underbrace{\textstyle\sum_{k=1}^K\esp{\vark[h_k]{\espk{Q_k}{S_k,\Ccal_k}}{S_k}}}_{\glexp(S)}
      \;+\;\underbrace{\textstyle\sum_{k=1}^K\esp{\vark[h_k]{Q_k}{S_k, \Ccal_k}}}_{\glres(S)}
      \label{eq:classwise1}
      \\
      \underbrace{\textstyle\sum_{k=1}^K\esp{\vark[h_k]{Q_k}{S_{B_k}}}}_{\GL(S_B)}
      \; & = \; \underbrace{\textstyle\sum_{k=1}^K\esp{\vark[h_k]{Q_k}{S_k}}}_{\GL(S)} \; + \; \underbrace{\textstyle\sum_{k=1}^K\esp{\vark[h_k]{C_k}{S_{B_k}}}}_{\GLind(S, S_B)}
      \label{eq:classwise2}
      \\
      \GL(S)
      & =\glexp(S_B)-\GLind(S, S_B)+\glres(S_B)
      \label{eq:classwise3}
    \end{align}
    Moreover, if $h_k$ is convex, then:
      \begin{align}
        \GL(S) \geq \glexp(S) \geq 0
        \label{eq:classwise4}
        \\
        \GLind(S, S_B) \geq 0
        \label{eq:classwise5}
        \\
        \GL(S) \geq \underbrace{\glexp(S_B) - \GLind(S, S_B)}_{\gllb(S, S_B)}
        \label{eq:classwise6}
      \end{align}
\end{restatable}

\begin{proof}[Proof of \autoref{th:results:classwise}]
  Applying the law of total variance (\autoref{lemma:law-tot-evar}) on each of the $\vark[h_k]{Q_k}{S_k}$ with $\Ccal_k$ as conditioning variable proves \autoref{eq:classwise1}.
  Similarly, applying the law of total variance on each of the $\vark[h_k]{Q_k}{S_{B_k}}$ with $S_k$ as conditioning variable proves \autoref{eq:classwise2}.
  The proof for \autoref{eq:classwise3} is the same as \autoref{prop:gllb-binning}.

  Using Jensen's inequality, if $h_k$ is convex, then $\Var[h_k] \geq 0$, which proves \autoref{eq:classwise4},~\ref{eq:classwise5} and \ref{eq:classwise6}.
\end{proof}

For the log-loss scoring rule, we have $\phi^{LL}(p, e_k) = -\log(p_k)$ and $h_k(p) = \log(p_k)p_k$ wich is convex. Thus, \autoref{th:results:classwise} holds for the log-loss. Unfortunately the Brier score does not satisfy the assumptions of \autoref{prop:e-variance:classwise} since $\phi^{BS}(p, e_k)$ is not a function of $p_k$. But a forumlation similar to \autoref{eq:gl-var:classwise} holds for the Brier score:

\begin{align}
  \esp{d_{\phi^{BS}}(C, Q)}
  & = \esp{s_{\phi^{BS}}(C, Q) - s_{\phi^{BS}}(Q, Q)}
  & \text{Definition of $d_{\phi^{BS}}$}\\
  & = \esp{(C-Q)\cdot(C-Q)}
  & \text{Definition of $s_{\phi^{BS}}$}\\
  & = \esp{(C\cdot C - 2C\cdot Q + Q\cdot Q)}
  & \\
  & = \esp{(Q\cdot Q - C\cdot C)}
  & \esp{C\cdot Q} = C\cdot C\\
  & = \textstyle\sum_{k=1}^K\esp{(Q_k^2 - C_k^2)}
  & \text{Linearity of expectation}\\
  & = \textstyle\sum_{k=1}^K\esp{(\espk{Q_k^2}{S_k} - C_k^2)}
  & \text{Law of total expectation}\\
  & = \textstyle\sum_{k=1}^K\esp{\vark{Q_k}{S_k}}
  & \text{Definition of the variance}\\
  \textit{with:}&&\\
  \esp{C\cdot Q}
  & = \textstyle\sum_{k=1}^K\esp{C_kQ_k}
  & \text{}\\
  & = \textstyle\sum_{k=1}^K\esp{\espk{C_kQ_k}{S_k}}
  & \text{Law of total expectation}\\
  & = \textstyle\sum_{k=1}^K\esp{C_k\espk{Q_k}{S_k}}
  & \text{$C_k$ is a function of $S_k$}\\
  & = \textstyle\sum_{k=1}^K\esp{C_k^2}
  & \text{Definition of $C_k$}\\
  & = \esp{C\cdot C}
\end{align}

Since $\Var = \Var[f]$ with $f: x \mapsto x^2$ \autoref{eq:gl-var:classwise} is satisfied for the Brier score. Since $f$ is convex, \autoref{th:results:classwise} holds for the Brier score.

To conclude, \autoref{th:results:classwise} holds for the Brier score and the log-loss in a classwise setting. It is likely that some other proper scoring rules satisfy \autoref{eq:gl-var:classwise} and \autoref{th:results:classwise}.

\subsection{Impact of recalibration on the grouping loss}
\label{sec:app:recal-impact}
\begin{lemma}
  \label{th:recal-impact}
Let $\hat{c}$ be a recalibration mapping and $S' = \hat{c}(S)$ the classifier recalibrated with that mapping. The grouping loss of the recalibrated classifier $\GL(S')$ deviates from that of the original classifier $\GL(S)$ as follows:
\begin{equation*}
  \GL(S') = \GL(S) + \esp{\vark[h]{C}{S'}}
\end{equation*}
If the mapping is perfect (\ie{} $S' = C$) or invertible, then $\GL(S') = \GL(S)$.
\end{lemma}
\begin{proof}[Proof of \autoref{th:recal-impact}]
\begin{align*}
  \GL(S')
  &
  = \esp{\vark[h]{Q}{S'}}
  & \text{Definition of $\GL$}\\
  & = \esp{\vark[h]{Q}{S',S}} + \esp{\vark[h]{\espk{Q}{S',S}}{S'}}
  & \text{Law of total $h$-variance on $S$ (\autoref{lemma:law-tot-evar}})\\
  & = \esp{\vark[h]{Q}{S}} + \esp{\vark[h]{\espk{Q}{S}}{S'}}
  & \text{$S'$ is a function of $S$}\\
  & = \GL(S) + \esp{\vark[h]{C}{S'}}
  & \text{Definition of $\GL(S)$ and $C$}
\end{align*}
If $S' = C$, then $\vark[h]{C}{S'} = 0$, hence $\GL(S') = \GL(S)$.
If the mapping $\hat{c}$ is invertible, then knowing $S'$ is knowing $S$. Hence $\vark[h]{C}{S'} = \vark[h]{C}{S} = 0$ since $C$ a function of $S$. Hence $\GL(S') = \GL(S)$.
\end{proof}



\section{ImageNet}
\label{sec:app:imagenet}

ImageNet-1K (ILSVRC2012) \citep{deng2009imagenet} is a classification dataset for computer vision with 1\,000 classes. Networks studied in this article are pre-trained on the training set of ImageNet-1K, comprising 1.2 million samples. Models' architectures and weights are available on PyTorch v0.12 \citep{Paszke2019}. We evaluated the networks on ImageNet variants' ImageNet-R and ImageNet-C (\autoref{sec:app:imagenet-r},~\ref{sec:app:imagenet-c}) as well as the validation set of ImageNet-1K (\autoref{sec:app:imagenet-1k}).
We work in the high-level feature space of the networks, \ie{} the output space of the penultimate layer (embedding space).\looseness=-1

For each of ImageNet-R, ImageNet-C and the validation set of ImageNet, we
plot the grouping diagrams of each network with and without post-hoc recalibration, obtained with a balanced decision stump. For each network, if several versions are available, we study both the smallest and the best performing one on the validation set of ImageNet-1K (usually the largest version). For ImageNet-R we also provide the grouping diagrams obtained with a 2-cluster $k$-means.
Each experiment is detailed in \autoref{sec:app:imagenet-r},~\ref{sec:app:imagenet-c} and \ref{sec:app:imagenet-1k}.

\paragraph{Detailed experimental method}\label{par:app:exp} First, we
forward each sample of the evaluation dataset (ImageNet-R, ImageNet-C or
the validation set of ImageNet-1K) through the studied network. We build
confidence scores by applying a softmax to the output logits. We extract
a representation of the input images in the high-level feature space of
the network (\ie{} the input space of the last linear layer). Since there
is not enough samples per class (50), we restrict our study to the
top-label problem (\autoref{def3}). For each sample, the class with the
highest confidence is predicted. The label is 1 if the network predicted
a correct class (0 otherwise) and the associated confidence score is the
one of the predicted class. We divide the samples of the evaluation set
in half making sure that the confidence score distribution is the same in
both resulting subsets. On one set, we train the isotonic regression for
calibration and calibrate the confidence scores of both sets. If no
post-hoc recalibration is used, we skip this step. Then, we create groups of same-level confidences by binning the confidence scores with 15 equal-width bins in $[0, 1]$. We partition each of the 15 level sets independently. For each of them, we create the partition by training the partitioning method on the training samples of the isotonic regression. We then evaluate region scores on the remaining samples to avoid overfitting. For the grouping diagrams, we mainly use a balanced decision stump with 2 clusters (\eg{} using scikit-learn's \texttt{DecisionTreeRegressor} with \texttt{min\_samples\_leaf} taken as half the samples in the bin), resulting in one split along one of the axis of the high-level feature space. For comparison, we also used $k$-means with 2 clusters. Constraining the partitioning methods to 2 regions is a choice to provide visually informative grouping diagrams rather than to maximize the lower bound $\glexphat$. When optimizing the lower bound, (\autoref{fig:vision:comparison} and \autoref{fig:app:imagenet:comparison}), we increase the number of allowed regions in the partition by setting a region ratio: the number of training samples in the bin over the number of allowed regions in the bin. Fixing a region ratio prevents from having regions with too few samples. In our experiments, we fix the region ratio to 30.

\subsection{ImageNet-R}
\label{sec:app:imagenet-r}
ImageNet-R \citep{hendrycks2021many} is a variant of ImageNet containing
renditions of the ImageNet classes. Example of renditions are: paintings,
toys, tattoos and origami. There are 15 rendition types in total listed
in \autoref{fig:app:imagenet-r:renditions}. The dataset contains 30\,000
images and is limited to 200 of the 1\,000 ImageNet classes.
\autoref{fig:app:imagenet:comparison}a. compares estimated grouping loss
lower bound and calibration errors of all networks (small and best
versions) on ImageNet-R. Overall, we observe a strong grouping loss in
most of the networks, especially those with highest accuracy. The
estimated debiased lower bound remains high after post-hoc recalibration.
Grouping diagrams of all networks are available at:
\begin{itemize}
  \itemsep0em
  \item Section \ref{sec:app:imagenet-r:nocal}: \nameref{sec:app:imagenet-r:nocal}.
  \item Section \ref{sec:app:imagenet-r:nocal:kmeans}: \nameref{sec:app:imagenet-r:nocal:kmeans}.
  \item Section \ref{sec:app:imagenet-r:nocal:best}: \nameref{sec:app:imagenet-r:nocal:best}.
  \item Section \ref{sec:app:imagenet-r:cal}: \nameref{sec:app:imagenet-r:cal}.
  \item Section \ref{sec:app:imagenet-r:cal:best}: \nameref{sec:app:imagenet-r:cal:best}.
\end{itemize}

We also investigate whether there is heterogeneity among renditions. In \autoref{fig:app:imagenet-r:renditions} we observe that some renditions are better predicted than average (\eg{} deviant art, sketch or art) while some others are predicted worse than average (\eg{} embroidery, cartoon, tattoo). These considerations would be useful in a fairness setting. Also, \autoref{fig:app:imagenet-r:renditions} highlights that if we could build regions out of renditions (\ie{} renditions are well separated in the feature space), this would result in a high grouping loss lower bound.\looseness=-1

\begin{figure}[t]
  \begin{minipage}{.58\linewidth}
    \includegraphics[width=\linewidth]{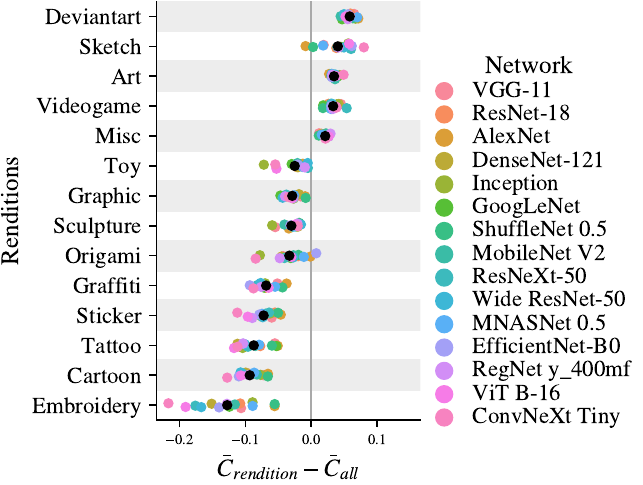}
  \end{minipage}%
  \hfill
  \begin{minipage}{.38\linewidth}
    \caption{\textbf{Comparison of renditions on ImageNet-R.} Differences between the calibrated scores of samples of one rendition $\bar{C}_{rendition}$ and the calibrated scores of all samples $\bar{C}_{all}$, weighted by the number of samples of this rendition in the level set and summed over the 15 bins on confidence scores. Interpretation: should renditions define regions of the feature space, they would exhibit a high grouping loss lower bound.}
    \label{fig:app:imagenet-r:renditions}
  \end{minipage}%
\end{figure}

\subsection{ImageNet-C}
\label{sec:app:imagenet-c}
ImageNet-C is a variant of ImageNet containing corrupted versions of ImageNet images. Examples of corruptions are: blur, noise, saturate, contrast, brightness and compression. There are 19 corruption types in total. Each corruption has a severity ranging from 1 to 5. The dataset contains the 50\,000 images of the validation set of ImageNet, each of them being applied 19 corruptions with 5 severity levels each. We built a merged version of ImageNet-C by randomly sampling one corruption for each image. We also study one corruption only (snow). For both the merged version and the snow version, we study the maximum severity of the corruption (5). \autoref{fig:app:imagenet:comparison}b. and c. compare estimated grouping loss lower bound and calibration errors of all networks (small and best versions) on ImageNet-C merged and snow. Overall, we observe similar effect than on ImageNet-R. However, when all samples have the same corruption (snow), we exhibit more grouping loss among the networks than when the 19 corruptions are randomly applied on the dataset (merged) (\autoref{fig:app:imagenet:comparison}c.). An intuition is that heterogeneity created by one corruption is canceled out by another one having heterogeneity in the opposite direction, leading to region scores closer to the average.
Grouping diagrams of all networks are available at:

\begin{itemize}
  \itemsep0em
  \item Section \ref{sec:app:imagenet-c:nocal}: \nameref{sec:app:imagenet-c:nocal}.
  \item Section \ref{sec:app:imagenet-c:nocal:best}: \nameref{sec:app:imagenet-c:nocal:best}.
  \item Section \ref{sec:app:imagenet-c:cal}: \nameref{sec:app:imagenet-c:cal}.
  \item Section \ref{sec:app:imagenet-c:cal:best}: \nameref{sec:app:imagenet-c:cal:best}.
\end{itemize}

\subsection{ImageNet-1K validation set}
\label{sec:app:imagenet-1k}

The validation set of ImageNet-1K comprises 50\,000 samples for 1\,000 classes. \autoref{fig:app:imagenet:comparison}d. compares estimated grouping loss lower bound and calibration errors of all networks (small and best versions) on the validation set of ImageNet-1K. Conversely to ImageNet-R and ImageNet-C, we cannot exhibit substantial grouping loss on any of the networks. The grouping diagrams (\autoref{fig:app:imagenet-1k:nocal:stump}) show however that ConvNeXt Tiny displays more heterogeneity than the other networks on this dataset.
Grouping diagrams are available in:
\begin{itemize}
  \itemsep0em
  \item Section \ref{sec:app:imagenet-1k:nocal}: \nameref{sec:app:imagenet-1k:nocal}.
  \item Section \ref{sec:app:imagenet-1k:nocal:best}: \nameref{sec:app:imagenet-1k:nocal:best}.
  \item Section \ref{sec:app:imagenet-1k:cal}: \nameref{sec:app:imagenet-1k:cal}.
  \item Section \ref{sec:app:imagenet-1k:cal:best}: \nameref{sec:app:imagenet-1k:cal:best}.
\end{itemize}


\begin{table}[b]
  \centering
  \vspace*{-4mm}
  \makebox[\textwidth]{
    \makebox[\textwidth]{\textbf{a. Vision}}
  }\vspace*{1mm}
  \\
  \begin{subtable}{\textwidth}
    \centering
    \begin{tabular}{lccccc}
\toprule
        Network & $\widehat{\mathrm{CL}}$ & $\widehat{\mathrm{CL}}'$ & $\widehat{\mathrm{GL}}_{\mathrm{LB}}$ & $\widehat{\mathrm{GL}}_{\mathrm{LB}}'$ &  Accuracy$\uparrow$ (\%) \\
\midrule
       ViT B-16 &                   0.044 &                    0.000 &                                 0.019 &                                  0.021 &                     30.1 \\
  ConvNeXt Tiny &                   0.017 &                    0.000 &                                 0.036 &                                  0.039 &                     29.0 \\
      Inception &                   0.200 &                    0.000 &                                 0.006 &                                  0.009 &                     27.3 \\
 Wide ResNet-50 &                   0.082 &                    0.000 &                                 0.020 &                                  0.019 &                     25.9 \\
     ResNeXt-50 &                   0.091 &                    0.000 &                                 0.013 &                                  0.013 &                     25.1 \\
EfficientNet-B0 &                   0.028 &                    0.000 &                                 0.015 &                                  0.018 &                     24.1 \\
   DenseNet-121 &                   0.059 &                    0.000 &                                 0.013 &                                  0.010 &                     23.9 \\
      GoogLeNet &                   0.009 &                    0.001 &                                 0.011 &                                  0.011 &                     23.3 \\
RegNet y\_400mf &                   0.084 &                    0.000 &                                 0.011 &                                  0.010 &                     21.0 \\
      ResNet-18 &                   0.057 &                    0.000 &                                 0.010 &                                  0.009 &                     20.4 \\
   MobileNet V2 &                   0.070 &                    0.000 &                                 0.007 &                                  0.009 &                     18.9 \\
    MNASNet 0.5 &                   0.017 &                    0.000 &                                 0.005 &                                  0.007 &                     15.9 \\
         VGG-11 &                   0.061 &                    0.000 &                                 0.007 &                                  0.007 &                     15.8 \\
 ShuffleNet 0.5 &                   0.072 &                    0.000 &                                 0.004 &                                  0.003 &                     14.6 \\
        AlexNet &                   0.063 &                    0.000 &                                 0.003 &                                  0.003 &                     12.6 \\
\bottomrule
\end{tabular}

  \end{subtable}
  \vspace*{4mm}\\
  \makebox[\textwidth]{
    \makebox[\textwidth]{\textbf{b. NLP}}
  }\vspace*{1mm}
  \\
  \begin{subtable}{\textwidth}
    \centering
    \begin{tabular}{lccccc}
\toprule
            Setting & $\widehat{\mathrm{CL}}$ & $\widehat{\mathrm{CL}}'$ & $\widehat{\mathrm{GL}}_{\mathrm{LB}}$ & $\widehat{\mathrm{GL}}_{\mathrm{LB}}'$ &  Accuracy $\uparrow$ (\%) \\
\midrule
    in-distribution &                   0.026 &                    0.000 &                                 0.000 &                                  0.000 &                      88.2 \\
out-of-distribution &                   0.091 &                    0.000 &                                 0.015 &                                  0.015 &                      71.1 \\
\bottomrule
\end{tabular}

  \end{subtable}
  \caption{Raw values of the estimators in the vision (\autoref{fig:vision:comparison}) and NLP experiments of \autoref{sec:exp_nn}, before ($\clhat$ and $\gllbhat$) and after ($\clhat'$ and $\gllbhat'$) isotonic recalibration.}
  \label{tab:vision-nlp:comparison}
\end{table}

\begin{figure}

  \centering
  \makebox[\textwidth]{
    \makebox[(\textwidth-9mm)/2]{\textbf{a. ImageNet-R}}
    \makebox[(\textwidth-9mm)/2]{\textbf{b. ImageNet-C (merged)}}
  }\vspace*{5mm}
  \\
  \includegraphics[width=0.5\linewidth]{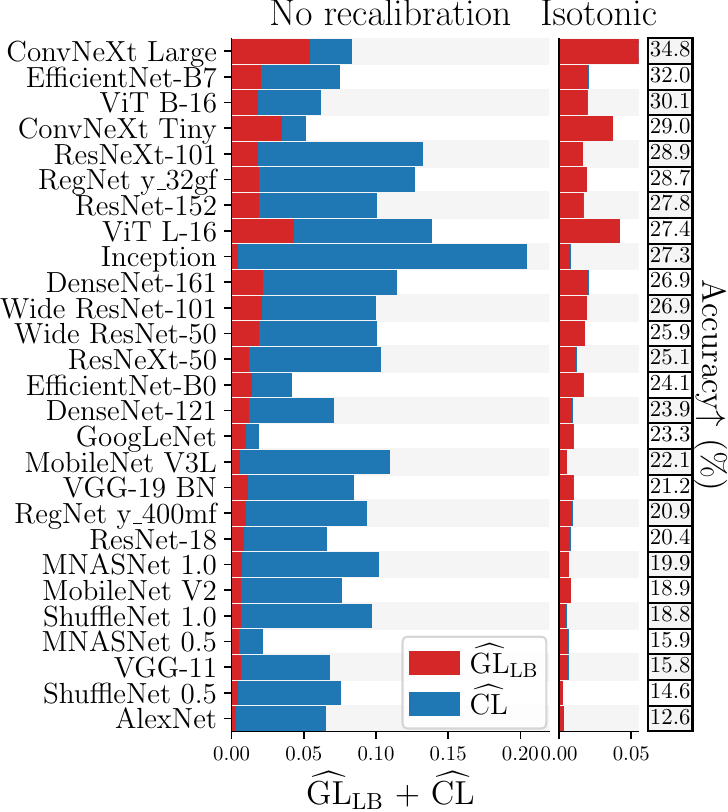}%
  \includegraphics[width=0.5\linewidth]{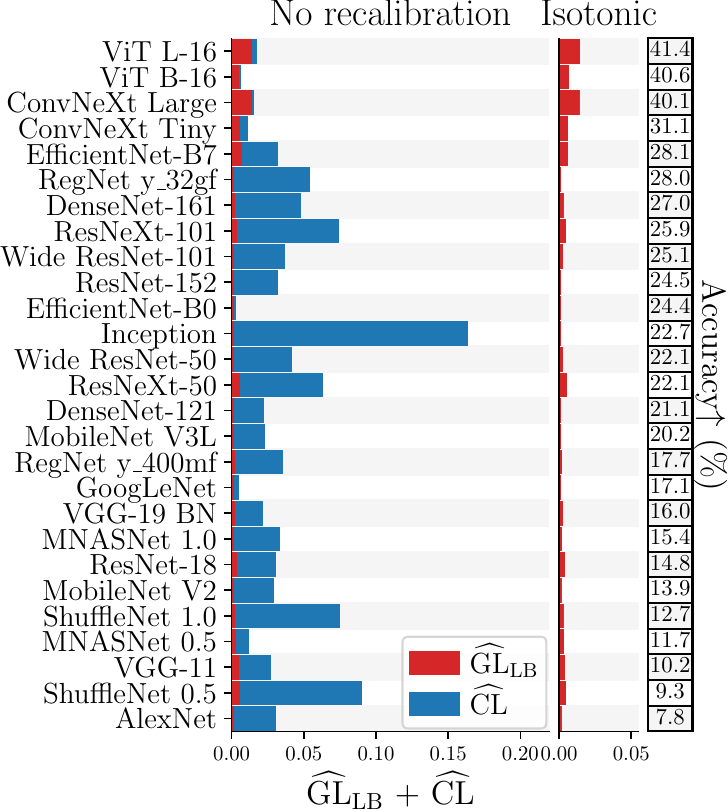}\\
  \vspace*{15mm}
  \makebox[\textwidth]{
    \makebox[(\textwidth-9mm)/2]{\textbf{c. ImageNet-C (snow)}}
    \makebox[(\textwidth-9mm)/2]{\textbf{d. ImageNet-1K (validation set)}}
  }\vspace*{5mm}
  \\
  \includegraphics[width=0.5\linewidth]{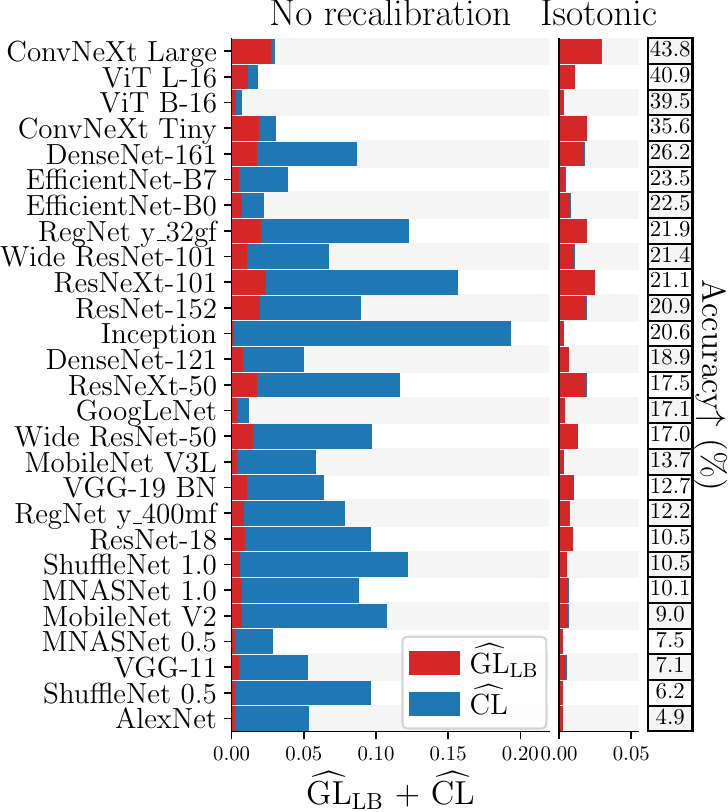}%
  \includegraphics[width=0.5\linewidth]{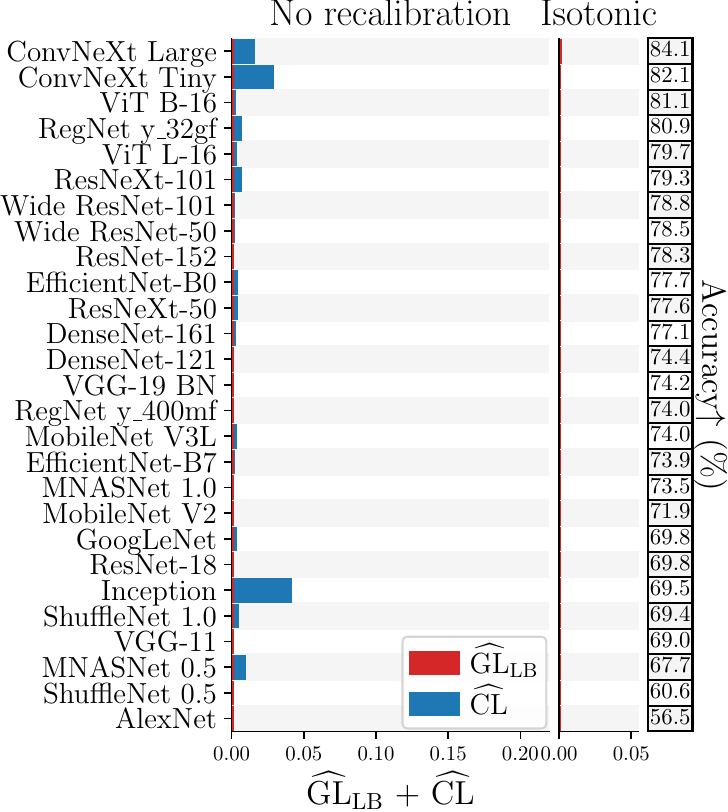}\\
  \caption{\textbf{Comparing vision models}: a debiased estimate of the
  grouping loss lower bound $\gllbhat$
  (\autoref{eq:lb-binning}) and an estimate of the calibration loss
  $\clhat$, both accounting for binning, evaluated on
  ImageNet-R, ImageNet-C and
  ImageNet-1K, sorted by model accuracy. Partitions $\Ccal$ are obtained
  from a decision tree partitioning constrained to create at most
  $\nicefrac{\text{\# samples in bin}}{30}$ regions in each bin. Isotonic
  regression is used for post-hoc recalibration of the models (right).}
  \label{fig:app:imagenet:comparison}

\end{figure}

\newcommand{\fourdiagrams}[8]{
  \makebox[\textwidth]{
    \makebox[2mm]{}
    \makebox[(\textwidth-9mm)/4]{#1}
    \makebox[(\textwidth-9mm)/4]{#3}
    \makebox[(\textwidth-9mm)/4]{#5}
    \makebox[(\textwidth-9mm)/4]{#7}
  }%
  \\
  \includegraphics[height=3.7cm]{#2}%
  \includegraphics[height=3.7cm]{#4}%
  \includegraphics[height=3.7cm]{#6}%
  \includegraphics[height=3.7cm]{#8}%
}

\begin{table}
  \small
  \centerfloat
  \makebox[\textwidth]{
    \makebox[0.72\textwidth]{\textbf{a. ImageNet-R}}
    \makebox[0.72\textwidth]{\textbf{b. ImageNet-C (merged)}}
  }\vspace*{1mm}
  \\
  \begin{subtable}{0.72\textwidth}
    \centering
    \begin{tabular}{lccccc}
\toprule
        Network & $\widehat{\mathrm{CL}}$ & $\widehat{\mathrm{CL}}'$ & $\widehat{\mathrm{GL}}_{\mathrm{LB}}$ & $\widehat{\mathrm{GL}}_{\mathrm{LB}}'$ &  Accuracy$\uparrow$ (\%) \\
\midrule
 ConvNeXt Large &                   0.029 &                    0.000 &                                 0.055 &                                  0.060 &                     34.8 \\
EfficientNet-B7 &                   0.055 &                    0.000 &                                 0.022 &                                  0.021 &                     32.0 \\
       ViT B-16 &                   0.044 &                    0.000 &                                 0.019 &                                  0.021 &                     30.1 \\
  ConvNeXt Tiny &                   0.017 &                    0.000 &                                 0.036 &                                  0.039 &                     29.0 \\
    ResNeXt-101 &                   0.114 &                    0.000 &                                 0.019 &                                  0.018 &                     28.9 \\
 RegNet y\_32gf &                   0.108 &                    0.000 &                                 0.020 &                                  0.020 &                     28.7 \\
     ResNet-152 &                   0.082 &                    0.000 &                                 0.020 &                                  0.018 &                     27.8 \\
       ViT L-16 &                   0.096 &                    0.000 &                                 0.044 &                                  0.043 &                     27.4 \\
      Inception &                   0.200 &                    0.000 &                                 0.006 &                                  0.009 &                     27.3 \\
   DenseNet-161 &                   0.093 &                    0.001 &                                 0.023 &                                  0.021 &                     26.9 \\
Wide ResNet-101 &                   0.079 &                    0.000 &                                 0.022 &                                  0.020 &                     26.9 \\
 Wide ResNet-50 &                   0.082 &                    0.000 &                                 0.020 &                                  0.019 &                     25.9 \\
     ResNeXt-50 &                   0.091 &                    0.000 &                                 0.013 &                                  0.013 &                     25.1 \\
EfficientNet-B0 &                   0.028 &                    0.000 &                                 0.015 &                                  0.018 &                     24.1 \\
   DenseNet-121 &                   0.059 &                    0.000 &                                 0.013 &                                  0.010 &                     23.9 \\
      GoogLeNet &                   0.009 &                    0.001 &                                 0.011 &                                  0.011 &                     23.3 \\
  MobileNet V3L &                   0.104 &                    0.000 &                                 0.007 &                                  0.006 &                     22.2 \\
      VGG-19 BN &                   0.073 &                    0.000 &                                 0.012 &                                  0.011 &                     21.2 \\
RegNet y\_400mf &                   0.084 &                    0.000 &                                 0.011 &                                  0.010 &                     21.0 \\
      ResNet-18 &                   0.057 &                    0.000 &                                 0.010 &                                  0.009 &                     20.4 \\
    MNASNet 1.0 &                   0.095 &                    0.000 &                                 0.008 &                                  0.008 &                     19.9 \\
   MobileNet V2 &                   0.070 &                    0.000 &                                 0.007 &                                  0.009 &                     18.9 \\
 ShuffleNet 1.0 &                   0.091 &                    0.000 &                                 0.007 &                                  0.006 &                     18.8 \\
    MNASNet 0.5 &                   0.017 &                    0.000 &                                 0.005 &                                  0.007 &                     15.9 \\
         VGG-11 &                   0.061 &                    0.000 &                                 0.007 &                                  0.007 &                     15.8 \\
 ShuffleNet 0.5 &                   0.072 &                    0.000 &                                 0.004 &                                  0.003 &                     14.6 \\
        AlexNet &                   0.063 &                    0.000 &                                 0.003 &                                  0.003 &                     12.6 \\
\bottomrule
\end{tabular}

  \end{subtable}%
  \begin{subtable}{0.72\textwidth}
    \centering
    \begin{tabular}{lccccc}
\toprule
        Network & $\widehat{\mathrm{CL}}$ & $\widehat{\mathrm{CL}}'$ & $\widehat{\mathrm{GL}}_{\mathrm{LB}}$ & $\widehat{\mathrm{GL}}_{\mathrm{LB}}'$ &  Accuracy$\uparrow$ (\%) \\
\midrule
       ViT L-16 &                   0.004 &                    0.000 &                                 0.016 &                                  0.016 &                     41.4 \\
       ViT B-16 &                   0.001 &                    0.000 &                                 0.007 &                                  0.009 &                     40.6 \\
 ConvNeXt Large &                   0.002 &                    0.000 &                                 0.016 &                                  0.016 &                     40.1 \\
  ConvNeXt Tiny &                   0.005 &                    0.000 &                                 0.008 &                                  0.008 &                     31.1 \\
EfficientNet-B7 &                   0.025 &                    0.000 &                                 0.009 &                                  0.007 &                     28.1 \\
 RegNet y\_32gf &                   0.053 &                    0.000 &                                 0.002 &                                  0.002 &                     28.0 \\
   DenseNet-161 &                   0.045 &                    0.000 &                                 0.004 &                                  0.004 &                     27.0 \\
    ResNeXt-101 &                   0.071 &                    0.000 &                                 0.005 &                                  0.006 &                     25.9 \\
Wide ResNet-101 &                   0.036 &                    0.000 &                                 0.003 &                                  0.004 &                     25.1 \\
     ResNet-152 &                   0.031 &                    0.000 &                                 0.002 &                                  0.002 &                     24.5 \\
EfficientNet-B0 &                   0.003 &                    0.000 &                                 0.002 &                                  0.002 &                     24.4 \\
      Inception &                   0.171 &                    0.000 &                                -0.006 &                                  0.001 &                     22.7 \\
 Wide ResNet-50 &                   0.040 &                    0.000 &                                 0.003 &                                  0.003 &                     22.1 \\
     ResNeXt-50 &                   0.058 &                    0.000 &                                 0.006 &                                  0.006 &                     22.1 \\
   DenseNet-121 &                   0.021 &                    0.000 &                                 0.002 &                                  0.002 &                     21.1 \\
  MobileNet V3L &                   0.024 &                    0.000 &                                 0.001 &                                  0.001 &                     20.2 \\
RegNet y\_400mf &                   0.034 &                    0.000 &                                 0.003 &                                  0.002 &                     17.7 \\
      GoogLeNet &                   0.004 &                    0.000 &                                 0.001 &                                  0.001 &                     17.1 \\
      VGG-19 BN &                   0.019 &                    0.000 &                                 0.003 &                                  0.003 &                     16.0 \\
    MNASNet 1.0 &                   0.032 &                    0.000 &                                 0.002 &                                  0.002 &                     15.4 \\
      ResNet-18 &                   0.027 &                    0.000 &                                 0.004 &                                  0.004 &                     14.8 \\
   MobileNet V2 &                   0.028 &                    0.000 &                                 0.002 &                                  0.002 &                     13.9 \\
 ShuffleNet 1.0 &                   0.073 &                    0.000 &                                 0.002 &                                  0.003 &                     12.7 \\
    MNASNet 0.5 &                   0.009 &                    0.000 &                                 0.003 &                                  0.003 &                     11.7 \\
         VGG-11 &                   0.022 &                    0.000 &                                 0.005 &                                  0.004 &                     10.2 \\
 ShuffleNet 0.5 &                   0.084 &                    0.000 &                                 0.006 &                                  0.005 &                      9.3 \\
        AlexNet &                   0.029 &                    0.000 &                                 0.001 &                                  0.001 &                      7.8 \\
\bottomrule
\end{tabular}

  \end{subtable}
  \vspace*{4mm}\\
  \makebox[\textwidth]{
    \makebox[0.72\textwidth]{\textbf{c. ImageNet-C (snow)}}
    \makebox[0.72\textwidth]{\textbf{d. ImageNet-1K (validation set)}}
  }\vspace*{1mm}
  \\
  \begin{subtable}{0.72\textwidth}
    \centering
    \begin{tabular}{lccccc}
\toprule
        Network & $\widehat{\mathrm{CL}}$ & $\widehat{\mathrm{CL}}'$ & $\widehat{\mathrm{GL}}_{\mathrm{LB}}$ & $\widehat{\mathrm{GL}}_{\mathrm{LB}}'$ &  Accuracy$\uparrow$ (\%) \\
\midrule
 ConvNeXt Large &                   0.002 &                    0.000 &                                 0.030 &                                  0.031 &                     43.8 \\
       ViT L-16 &                   0.007 &                    0.000 &                                 0.013 &                                  0.013 &                     40.9 \\
       ViT B-16 &                   0.004 &                    0.000 &                                 0.005 &                                  0.005 &                     39.5 \\
  ConvNeXt Tiny &                   0.011 &                    0.000 &                                 0.021 &                                  0.021 &                     35.6 \\
   DenseNet-161 &                   0.069 &                    0.000 &                                 0.019 &                                  0.019 &                     26.2 \\
EfficientNet-B7 &                   0.034 &                    0.000 &                                 0.006 &                                  0.005 &                     23.5 \\
EfficientNet-B0 &                   0.015 &                    0.000 &                                 0.008 &                                  0.009 &                     22.5 \\
 RegNet y\_32gf &                   0.102 &                    0.000 &                                 0.021 &                                  0.020 &                     21.9 \\
Wide ResNet-101 &                   0.057 &                    0.000 &                                 0.012 &                                  0.011 &                     21.4 \\
    ResNeXt-101 &                   0.133 &                    0.000 &                                 0.025 &                                  0.025 &                     21.1 \\
     ResNet-152 &                   0.070 &                    0.000 &                                 0.020 &                                  0.020 &                     20.9 \\
      Inception &                   0.194 &                    0.000 &                                 0.000 &                                  0.004 &                     20.6 \\
   DenseNet-121 &                   0.042 &                    0.000 &                                 0.009 &                                  0.007 &                     18.9 \\
     ResNeXt-50 &                   0.100 &                    0.000 &                                 0.018 &                                  0.019 &                     17.5 \\
      GoogLeNet &                   0.008 &                    0.000 &                                 0.005 &                                  0.005 &                     17.1 \\
 Wide ResNet-50 &                   0.083 &                    0.000 &                                 0.015 &                                  0.014 &                     17.0 \\
  MobileNet V3L &                   0.055 &                    0.000 &                                 0.004 &                                  0.003 &                     13.7 \\
      VGG-19 BN &                   0.053 &                    0.000 &                                 0.011 &                                  0.010 &                     12.7 \\
RegNet y\_400mf &                   0.070 &                    0.000 &                                 0.009 &                                  0.007 &                     12.2 \\
      ResNet-18 &                   0.088 &                    0.000 &                                 0.009 &                                  0.010 &                     10.5 \\
 ShuffleNet 1.0 &                   0.116 &                    0.000 &                                 0.006 &                                  0.005 &                     10.5 \\
    MNASNet 1.0 &                   0.081 &                    0.000 &                                 0.007 &                                  0.006 &                     10.1 \\
   MobileNet V2 &                   0.100 &                    0.000 &                                 0.007 &                                  0.006 &                      9.0 \\
    MNASNet 0.5 &                   0.026 &                    0.000 &                                 0.002 &                                  0.003 &                      7.5 \\
         VGG-11 &                   0.047 &                    0.000 &                                 0.005 &                                  0.004 &                      7.1 \\
 ShuffleNet 0.5 &                   0.094 &                    0.000 &                                 0.002 &                                  0.002 &                      6.2 \\
        AlexNet &                   0.051 &                    0.000 &                                 0.002 &                                  0.002 &                      4.9 \\
\bottomrule
\end{tabular}

  \end{subtable}%
  \begin{subtable}{0.72\textwidth}
    \centering
    \begin{tabular}{lccccc}
\toprule
        Network & $\widehat{\mathrm{CL}}$ & $\widehat{\mathrm{CL}}'$ & $\widehat{\mathrm{GL}}_{\mathrm{LB}}$ & $\widehat{\mathrm{GL}}_{\mathrm{LB}}'$ &  Accuracy$\uparrow$ (\%) \\
\midrule
 ConvNeXt Large &                   0.015 &                    0.000 &                                 0.002 &                                  0.002 &                     84.1 \\
  ConvNeXt Tiny &                   0.028 &                    0.000 &                                 0.002 &                                  0.002 &                     82.1 \\
       ViT B-16 &                   0.004 &                    0.000 &                                -0.000 &                                  0.000 &                     81.1 \\
 RegNet y\_32gf &                   0.009 &                    0.000 &                                -0.001 &                                 -0.000 &                     80.9 \\
       ViT L-16 &                   0.003 &                    0.000 &                                 0.001 &                                  0.001 &                     79.7 \\
    ResNeXt-101 &                   0.010 &                    0.000 &                                -0.002 &                                  0.000 &                     79.3 \\
Wide ResNet-101 &                   0.004 &                    0.000 &                                -0.001 &                                 -0.000 &                     78.8 \\
 Wide ResNet-50 &                   0.004 &                    0.000 &                                -0.001 &                                 -0.000 &                     78.5 \\
     ResNet-152 &                   0.004 &                    0.000 &                                -0.001 &                                  0.000 &                     78.3 \\
EfficientNet-B0 &                   0.006 &                    0.000 &                                 0.000 &                                  0.000 &                     77.7 \\
     ResNeXt-50 &                   0.006 &                    0.000 &                                -0.001 &                                  0.000 &                     77.6 \\
   DenseNet-161 &                   0.004 &                    0.000 &                                -0.001 &                                 -0.000 &                     77.1 \\
   DenseNet-121 &                   0.001 &                    0.000 &                                -0.000 &                                 -0.000 &                     74.4 \\
      VGG-19 BN &                   0.002 &                    0.000 &                                -0.001 &                                 -0.000 &                     74.2 \\
RegNet y\_400mf &                   0.001 &                    0.000 &                                -0.001 &                                 -0.000 &                     74.0 \\
  MobileNet V3L &                   0.006 &                    0.000 &                                -0.001 &                                 -0.000 &                     74.0 \\
EfficientNet-B7 &                   0.001 &                    0.000 &                                 0.003 &                                  0.002 &                     73.9 \\
    MNASNet 1.0 &                   0.003 &                    0.000 &                                -0.001 &                                 -0.000 &                     73.5 \\
   MobileNet V2 &                   0.001 &                    0.000 &                                -0.001 &                                  0.000 &                     71.9 \\
      GoogLeNet &                   0.006 &                    0.000 &                                -0.000 &                                  0.000 &                     69.8 \\
      ResNet-18 &                   0.001 &                    0.000 &                                 0.000 &                                 -0.000 &                     69.8 \\
      Inception &                   0.050 &                    0.000 &                                -0.007 &                                  0.001 &                     69.5 \\
 ShuffleNet 1.0 &                   0.008 &                    0.000 &                                -0.001 &                                  0.000 &                     69.4 \\
         VGG-11 &                   0.000 &                    0.000 &                                -0.000 &                                  0.000 &                     69.0 \\
    MNASNet 0.5 &                   0.012 &                    0.000 &                                -0.001 &                                 -0.000 &                     67.7 \\
 ShuffleNet 0.5 &                   0.004 &                    0.000 &                                -0.001 &                                  0.000 &                     60.6 \\
        AlexNet &                   0.001 &                    0.000 &                                -0.001 &                                 -0.001 &                     56.5 \\
\bottomrule
\end{tabular}

  \end{subtable}
  \caption{Raw values of the estimators of \autoref{fig:app:imagenet:comparison}, before ($\clhat$ and $\gllbhat$) and after ($\clhat'$ and $\gllbhat'$) isotonic recalibration.}
  \label{tab:app:imagenet:comparison}
  \end{table}

\section{NLP}
\label{sec:app:nlp}
We use BART Large \citep{Lewis2019} pre-trained on the Multi-Genre Natural Language Inference dataset \citep{MNLI} and fine-tuned on the Yahoo Answers Topics dataset for zero-shot topic classification. The fine-tuned model is available on HuggingFace at \url{https://huggingface.co/joeddav/bart-large-mnli-yahoo-answers}. Yahoo Answers Topics is composed of question titles and bodies and topic labels. There are 1\,400\,000 training samples, 60\,000 test samples and 10 topics. The dataset is available at \url{https://huggingface.co/datasets/yahoo_answers_topics}. The model is fine-tuned on 5 out of the 10 topics of the training set, totalizing 700\,000 samples. Given a question title and a hypothesis (\eg{} ``This text is about Science \& Mathematics''), the model outputs its confidence in the hypothesis to be true for the given question. The classification being zero-shot, the hypothesis can be about an unseen topic. We evaluate the model separately on the 5 unseen topics and the 5 seen topics of the test set (\ie{} seen topics but unseen samples). This results in a binary classification task in which each sample is composed of a question title and a hypothesis and each label is 1 or 0 whether the hypothesis is correct or not. As for the clustering and calibration procedure, we used a balanced decision stump in the same way as described in \autoref{par:app:exp}: ``\nameref{par:app:exp}''. We work in the high-level feature space of the network, \ie{} the output space of the penultimate layer (embedding space).
\vspace{-5mm}

\setlength{\fboxsep}{0pt}
\newcommand{\diagrams}[1]{
  \setsepchar{;}
  \readlist\arg{#1}

    \vspace*{5mm}
    \makebox[\textwidth]{
      \makebox[2mm]{}
      \makebox[(\textwidth-9mm)/4]{\arg[1]}
      \makebox[(\textwidth-9mm)/4]{\arg[3]}
      \makebox[(\textwidth-9mm)/4]{\arg[5]}
      \makebox[(\textwidth-9mm)/4]{\arg[7]}
    }%
    \\
    \includegraphics[height=3.7cm,trim={0.cm 0cm 0.3795cm 0cm},clip]{\arg[2]}\hspace*{3mm}%
    \includegraphics[height=3.7cm,trim={0.cm 0cm 0.3795cm 0cm},clip]{\arg[4]}\hspace*{3mm}%
    \includegraphics[height=3.7cm,trim={0.cm 0cm 0.3795cm 0cm},clip]{\arg[6]}\hspace*{3mm}%
    \includegraphics[height=3.7cm,trim={0.cm 0cm 0.3795cm 0cm},clip]{\arg[8]}%
    \\
    \vspace*{5mm}\\
    \makebox[\textwidth]{
      \makebox[2mm]{}
      \makebox[(\textwidth-9mm)/4]{\arg[9]}
      \makebox[(\textwidth-9mm)/4]{\arg[11]}
      \makebox[(\textwidth-9mm)/4]{\arg[13]}
      \makebox[(\textwidth-9mm)/4]{\arg[15]}
    }\\
    \includegraphics[height=3.7cm,trim={0.cm 0cm 0.3795cm 0cm},clip]{\arg[10]}\hspace*{3mm}%
    \includegraphics[height=3.7cm,trim={0.cm 0cm 0.3795cm 0cm},clip]{\arg[12]}\hspace*{3mm}%
    \includegraphics[height=3.7cm,trim={0.cm 0cm 0.3795cm 0cm},clip]{\arg[14]}\hspace*{3mm}%
    \includegraphics[height=3.7cm,trim={0.cm 0cm 0.3795cm 0cm},clip]{\arg[16]}%
    \\
    \vspace*{5mm}\\
    \makebox[\textwidth]{
      \makebox[2mm]{}
      \makebox[(\textwidth-9mm)/4]{\arg[17]}
      \makebox[(\textwidth-9mm)/4]{\arg[19]}
      \makebox[(\textwidth-9mm)/4]{\arg[21]}
      \makebox[(\textwidth-9mm)/4]{\arg[23]}
    }\\
    \includegraphics[height=3.7cm,trim={0.cm 0cm 0.3795cm 0cm},clip]{\arg[18]}\hspace*{3mm}%
    \includegraphics[height=3.7cm,trim={0.cm 0cm 0.3795cm 0cm},clip]{\arg[20]}\hspace*{3mm}%
    \includegraphics[height=3.7cm,trim={0.cm 0cm 0.3795cm 0cm},clip]{\arg[22]}\hspace*{3mm}%
    \includegraphics[height=3.7cm,trim={0.cm 0cm 0.3795cm 0cm},clip]{\arg[24]}%
    \\
    \vspace*{5mm}\\
    \makebox[\textwidth]{
      \makebox[2mm]{}
      \makebox[(\textwidth-9mm)/4]{\arg[25]}
      \makebox[(\textwidth-9mm)/4]{\arg[27]}
      \makebox[(\textwidth-9mm)/4]{\arg[29]}
      \makebox[(\textwidth-9mm)/4]{}
    }\\
    \includegraphics[height=3.7cm,trim={0.cm 0cm 0.3795cm 0cm},clip]{\arg[26]}\hspace*{3mm}%
    \includegraphics[height=3.7cm,trim={0.cm 0cm 0.3795cm 0cm},clip]{\arg[28]}\hspace*{3mm}%
    \includegraphics[height=3.7cm,trim={0.cm 0cm 0.3795cm 0cm},clip]{\arg[30]}%
    \hfill%
    \\


}

\newcommand{\diagramsbest}[1]{
  \setsepchar{;}
  \readlist\arg{#1}

    \vspace*{5mm}
    \makebox[\textwidth]{
      \makebox[2mm]{}
      \makebox[(\textwidth-9mm)/4]{\arg[1]}
      \makebox[(\textwidth-9mm)/4]{\arg[3]}
      \makebox[(\textwidth-9mm)/4]{\arg[5]}
      \makebox[(\textwidth-9mm)/4]{\arg[7]}
    }%
    \\
    \includegraphics[height=3.7cm,trim={0.cm 0cm 0.3795cm 0cm},clip]{\arg[2]}\hspace*{3mm}%
    \includegraphics[height=3.7cm,trim={0.cm 0cm 0.3795cm 0cm},clip]{\arg[4]}\hspace*{3mm}%
    \includegraphics[height=3.7cm,trim={0.cm 0cm 0.3795cm 0cm},clip]{\arg[6]}\hspace*{3mm}%
    \includegraphics[height=3.7cm,trim={0.cm 0cm 0.3795cm 0cm},clip]{\arg[8]}%
    \\
    \vspace*{5mm}\\
    \makebox[\textwidth]{
      \makebox[2mm]{}
      \makebox[(\textwidth-9mm)/4]{\arg[9]}
      \makebox[(\textwidth-9mm)/4]{\arg[11]}
      \makebox[(\textwidth-9mm)/4]{\arg[13]}
      \makebox[(\textwidth-9mm)/4]{\arg[15]}
    }\\
    \includegraphics[height=3.7cm,trim={0.cm 0cm 0.3795cm 0cm},clip]{\arg[10]}\hspace*{3mm}%
    \includegraphics[height=3.7cm,trim={0.cm 0cm 0.3795cm 0cm},clip]{\arg[12]}\hspace*{3mm}%
    \includegraphics[height=3.7cm,trim={0.cm 0cm 0.3795cm 0cm},clip]{\arg[14]}\hspace*{3mm}%
    \includegraphics[height=3.7cm,trim={0.cm 0cm 0.3795cm 0cm},clip]{\arg[16]}%
    \\
    \vspace*{5mm}\\
    \makebox[\textwidth]{
      \makebox[2mm]{}
      \makebox[(\textwidth-9mm)/4]{\arg[17]}
      \makebox[(\textwidth-9mm)/4]{\arg[19]}
      \makebox[(\textwidth-9mm)/4]{\arg[21]}
      \makebox[(\textwidth-9mm)/4]{\arg[23]}
    }\\
    \includegraphics[height=3.7cm,trim={0.cm 0cm 0.3795cm 0cm},clip]{\arg[18]}\hspace*{3mm}%
    \includegraphics[height=3.7cm,trim={0.cm 0cm 0.3795cm 0cm},clip]{\arg[20]}\hspace*{3mm}%
    \includegraphics[height=3.7cm,trim={0.cm 0cm 0.3795cm 0cm},clip]{\arg[22]}\hspace*{3mm}%
    \includegraphics[height=3.7cm,trim={0.cm 0cm 0.3795cm 0cm},clip]{\arg[24]}%
    \\


}

\setcounter{section}{4}
\setcounter{subsection}{1}
\begin{figure}[p]
  \addcontentsline{toc}{section}{Grouping diagrams of all vision networks}
  \addcontentsline{toc}{subsection}{ImageNet-R}
  \subsubsection{ImageNet-R – No post-hoc recalibration, small versions}
  \label{sec:app:imagenet-r:nocal}

  \diagrams{%
  AlexNet;%
  xp_nn_calibration/fig_imagenet_calibration/net_ILSVRC2012_img_test_r_alexnet_cal_None_clu_dt.pdf;
  VGG 11;%
  xp_nn_calibration/fig_imagenet_calibration/net_ILSVRC2012_img_test_r_vgg_11_cal_None_clu_dt.pdf;%
  ResNet 18;%
  xp_nn_calibration/fig_imagenet_calibration/net_ILSVRC2012_img_test_r_resnet_18_cal_None_clu_dt.pdf;%
  DenseNet 121;%
  xp_nn_calibration/fig_imagenet_calibration/net_ILSVRC2012_img_test_r_densenet_121_cal_None_clu_dt.pdf;%
  Inception;%
  xp_nn_calibration/fig_imagenet_calibration/net_ILSVRC2012_img_test_r_inception_cal_None_clu_dt.pdf;%
  GoogleNet;%
  xp_nn_calibration/fig_imagenet_calibration/net_ILSVRC2012_img_test_r_googlenet_cal_None_clu_dt.pdf;%
  ShuffleNet V2 x0.5;%
  xp_nn_calibration/fig_imagenet_calibration/net_ILSVRC2012_img_test_r_shufflenet_0_5_cal_None_clu_dt.pdf;%
  MobileNet V2;%
  xp_nn_calibration/fig_imagenet_calibration/net_ILSVRC2012_img_test_r_mobilenet_v2_cal_None_clu_dt.pdf;%
  ResNext 50;%
  xp_nn_calibration/fig_imagenet_calibration/net_ILSVRC2012_img_test_r_resnext_50_cal_None_clu_dt.pdf;%
  WideResNet 50;%
  xp_nn_calibration/fig_imagenet_calibration/net_ILSVRC2012_img_test_r_wideresnet_50_cal_None_clu_dt.pdf;%
  MNASNet 0.5;%
  xp_nn_calibration/fig_imagenet_calibration/net_ILSVRC2012_img_test_r_mnasnet_0_5_cal_None_clu_dt.pdf;%
  EfficientNet B0;%
  xp_nn_calibration/fig_imagenet_calibration/net_ILSVRC2012_img_test_r_efficientnet_b0_cal_None_clu_dt.pdf;%
  RegNet y\_400mf;%
  xp_nn_calibration/fig_imagenet_calibration/net_ILSVRC2012_img_test_r_regnet_y_400mf_cal_None_clu_dt.pdf;%
  ViT B-16;%
  xp_nn_calibration/fig_imagenet_calibration/net_ILSVRC2012_img_test_r_visiontransformer_b_16_cal_None_clu_dt.pdf;%
  ConvNeXt Tiny;%
  xp_nn_calibration/fig_imagenet_calibration/net_ILSVRC2012_img_test_r_convnext_tiny_cal_None_clu_dt.pdf%
  }

  \caption{\textbf{Vision}: Fraction of correct predictions versus
confidence score of predicted class ($\max_k S_k$) on ImageNet-R for
small versions of pre-trained networks, without post-hoc recalibration. In each bin on confidence scores, the level set is partitioned into 2 regions with a decision stump constrained to one balanced split, with a 50-50 train-test split strategy.}
  \label{fig:app:imagenet-r:nocal:stump}
\end{figure}

\begin{figure}[p]
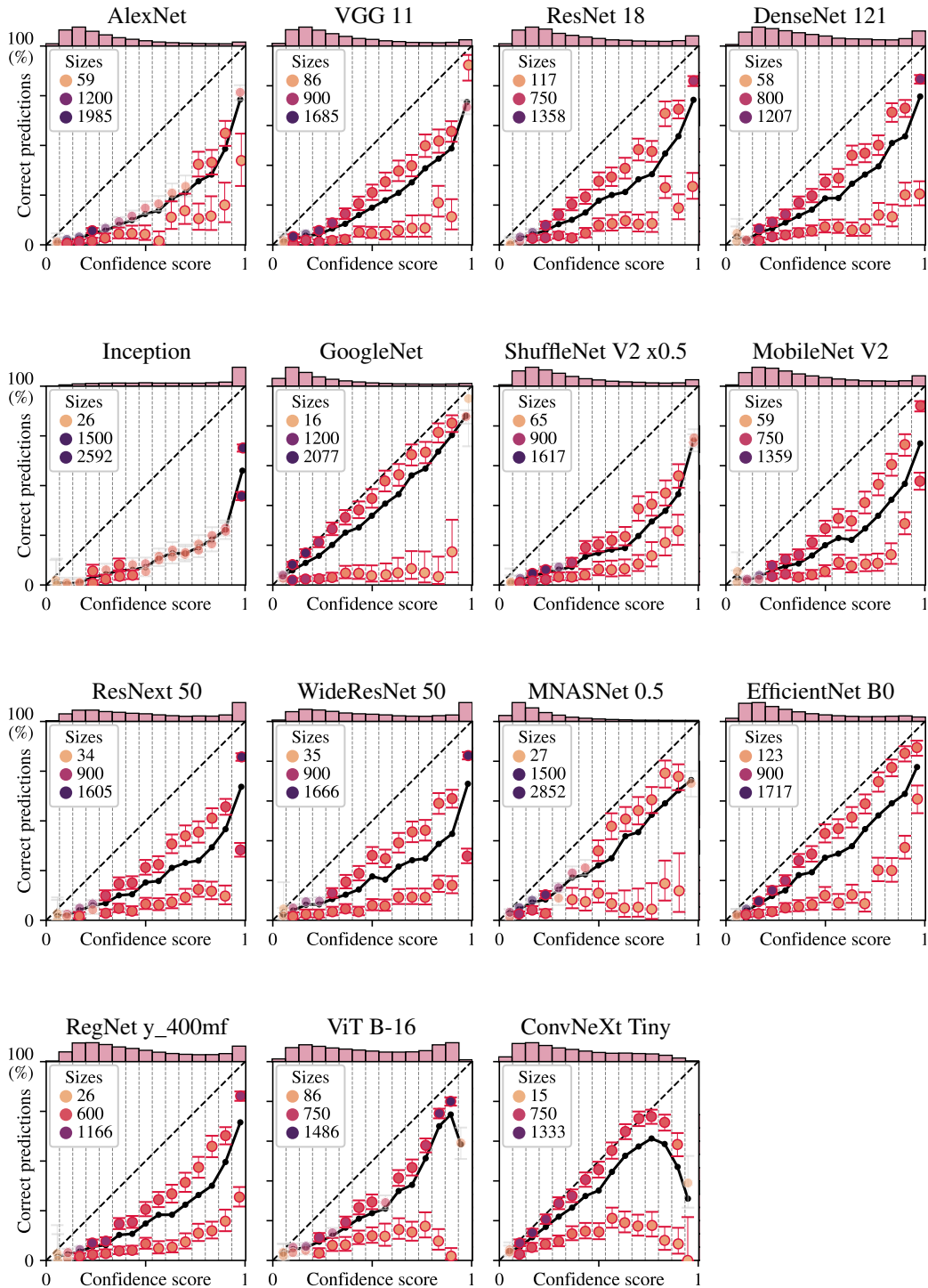

  \subsubsection{ImageNet-R – No post-hoc recalibration, small versions, k-means}
  \label{sec:app:imagenet-r:nocal:kmeans}

  \diagrams{%
  AlexNet;%
  xp_nn_calibration/fig_imagenet_calibration/net_ILSVRC2012_img_test_r_alexnet_cal_None_clu_km.pdf;
  VGG 11;%
  xp_nn_calibration/fig_imagenet_calibration/net_ILSVRC2012_img_test_r_vgg_11_cal_None_clu_km.pdf;%
  ResNet 18;%
  xp_nn_calibration/fig_imagenet_calibration/net_ILSVRC2012_img_test_r_resnet_18_cal_None_clu_km.pdf;%
  DenseNet 121;%
  xp_nn_calibration/fig_imagenet_calibration/net_ILSVRC2012_img_test_r_densenet_121_cal_None_clu_km.pdf;%
  Inception;%
  xp_nn_calibration/fig_imagenet_calibration/net_ILSVRC2012_img_test_r_inception_cal_None_clu_km.pdf;%
  GoogleNet;%
  xp_nn_calibration/fig_imagenet_calibration/net_ILSVRC2012_img_test_r_googlenet_cal_None_clu_km.pdf;%
  ShuffleNet V2 x0.5;%
  xp_nn_calibration/fig_imagenet_calibration/net_ILSVRC2012_img_test_r_shufflenet_0_5_cal_None_clu_km.pdf;%
  MobileNet V2;%
  xp_nn_calibration/fig_imagenet_calibration/net_ILSVRC2012_img_test_r_mobilenet_v2_cal_None_clu_km.pdf;%
  ResNext 50;%
  xp_nn_calibration/fig_imagenet_calibration/net_ILSVRC2012_img_test_r_resnext_50_cal_None_clu_km.pdf;%
  WideResNet 50;%
  xp_nn_calibration/fig_imagenet_calibration/net_ILSVRC2012_img_test_r_wideresnet_50_cal_None_clu_km.pdf;%
  MNASNet 0.5;%
  xp_nn_calibration/fig_imagenet_calibration/net_ILSVRC2012_img_test_r_mnasnet_0_5_cal_None_clu_km.pdf;%
  EfficientNet B0;%
  xp_nn_calibration/fig_imagenet_calibration/net_ILSVRC2012_img_test_r_efficientnet_b0_cal_None_clu_km.pdf;%
  RegNet y\_400mf;%
  xp_nn_calibration/fig_imagenet_calibration/net_ILSVRC2012_img_test_r_regnet_y_400mf_cal_None_clu_km.pdf;%
  ViT B-16;%
  xp_nn_calibration/fig_imagenet_calibration/net_ILSVRC2012_img_test_r_visiontransformer_b_16_cal_None_clu_km.pdf;%
  ConvNeXt Tiny;%
  xp_nn_calibration/fig_imagenet_calibration/net_ILSVRC2012_img_test_r_convnext_tiny_cal_None_clu_km.pdf%
  }

  \caption{\textbf{Vision}: Fraction of correct predictions versus
confidence score of predicted class ($\max_k S_k$) on ImageNet-R for
small versions of pre-trained networks, without post-hoc recalibration. In each bin on confidence scores, the level set is partitioned into 2 regions with a $k$-means clustering, with a 50-50 train-test split strategy (for a fair comparison with decision stump clustering).}
  \label{fig:app:imagenet-r:nocal:kmeans}
\end{figure}

\begin{figure}[p]
  \subsubsection{ImageNet-R – No post-hoc recalibration, best versions}
  \label{sec:app:imagenet-r:nocal:best}

  \diagramsbest{%
  VGG 19 BN;%
  xp_nn_calibration/fig_imagenet_calibration/net_ILSVRC2012_img_test_r_vgg_19_bn_cal_None_clu_dt.pdf;%
  ResNet 152;%
  xp_nn_calibration/fig_imagenet_calibration/net_ILSVRC2012_img_test_r_resnet_152_cal_None_clu_dt.pdf;%
  DenseNet 161;%
  xp_nn_calibration/fig_imagenet_calibration/net_ILSVRC2012_img_test_r_densenet_161_cal_None_clu_dt.pdf;%
  ShuffleNet V2 x1.0;%
  xp_nn_calibration/fig_imagenet_calibration/net_ILSVRC2012_img_test_r_shufflenet_1_0_cal_None_clu_dt.pdf;%
  MobileNet V3 Large;%
  xp_nn_calibration/fig_imagenet_calibration/net_ILSVRC2012_img_test_r_mobilenet_v3L_cal_None_clu_dt.pdf;%
  ResNext 101;%
  xp_nn_calibration/fig_imagenet_calibration/net_ILSVRC2012_img_test_r_resnext_101_cal_None_clu_dt.pdf;%
  WideResNet 101;%
  xp_nn_calibration/fig_imagenet_calibration/net_ILSVRC2012_img_test_r_wideresnet_101_cal_None_clu_dt.pdf;%
  MNASNet 1.0;%
  xp_nn_calibration/fig_imagenet_calibration/net_ILSVRC2012_img_test_r_mnasnet_1_0_cal_None_clu_dt.pdf;%
  EfficientNet B7;%
  xp_nn_calibration/fig_imagenet_calibration/net_ILSVRC2012_img_test_r_efficientnet_b7_cal_None_clu_dt.pdf;%
  RegNet y\_32gf;%
  xp_nn_calibration/fig_imagenet_calibration/net_ILSVRC2012_img_test_r_regnet_y_32gf_cal_None_clu_dt.pdf;%
  ViT L-16;%
  xp_nn_calibration/fig_imagenet_calibration/net_ILSVRC2012_img_test_r_visiontransformer_l_16_cal_None_clu_dt.pdf;%
  ConvNeXt Large;%
  xp_nn_calibration/fig_imagenet_calibration/net_ILSVRC2012_img_test_r_convnext_large_cal_None_clu_dt.pdf%
  }

  \caption{\textbf{Vision}: Fraction of correct predictions versus
confidence score of predicted class ($\max_k S_k$) on ImageNet-R for best
versions of pre-trained networks, without post-hoc recalibration. In each bin on confidence scores, the level set is partitioned into 2 regions with a decision stump constrained to one balanced split, with a 50-50 train-test split strategy.}
  \label{fig:app:imagenet-r:nocal:stump:best}
  \null
  \vfill
\end{figure}

\begin{figure}[p]
  \subsubsection{ImageNet-R – Isotonic recalibration, small versions}
  \label{sec:app:imagenet-r:cal}

  \diagrams{%
  AlexNet;%
  xp_nn_calibration/fig_imagenet_calibration/net_ILSVRC2012_img_test_r_alexnet_cal_isotonic_clu_dt.pdf;
  VGG 11;%
  xp_nn_calibration/fig_imagenet_calibration/net_ILSVRC2012_img_test_r_vgg_11_cal_isotonic_clu_dt.pdf;%
  ResNet 18;%
  xp_nn_calibration/fig_imagenet_calibration/net_ILSVRC2012_img_test_r_resnet_18_cal_isotonic_clu_dt.pdf;%
  DenseNet 121;%
  xp_nn_calibration/fig_imagenet_calibration/net_ILSVRC2012_img_test_r_densenet_121_cal_isotonic_clu_dt.pdf;%
  Inception;%
  xp_nn_calibration/fig_imagenet_calibration/net_ILSVRC2012_img_test_r_inception_cal_isotonic_clu_dt.pdf;%
  GoogleNet;%
  xp_nn_calibration/fig_imagenet_calibration/net_ILSVRC2012_img_test_r_googlenet_cal_isotonic_clu_dt.pdf;%
  ShuffleNet V2 x0.5;%
  xp_nn_calibration/fig_imagenet_calibration/net_ILSVRC2012_img_test_r_shufflenet_0_5_cal_isotonic_clu_dt.pdf;%
  MobileNet V2;%
  xp_nn_calibration/fig_imagenet_calibration/net_ILSVRC2012_img_test_r_mobilenet_v2_cal_isotonic_clu_dt.pdf;%
  ResNext 50;%
  xp_nn_calibration/fig_imagenet_calibration/net_ILSVRC2012_img_test_r_resnext_50_cal_isotonic_clu_dt.pdf;%
  WideResNet 50;%
  xp_nn_calibration/fig_imagenet_calibration/net_ILSVRC2012_img_test_r_wideresnet_50_cal_isotonic_clu_dt.pdf;%
  MNASNet 0.5;%
  xp_nn_calibration/fig_imagenet_calibration/net_ILSVRC2012_img_test_r_mnasnet_0_5_cal_isotonic_clu_dt.pdf;%
  EfficientNet B0;%
  xp_nn_calibration/fig_imagenet_calibration/net_ILSVRC2012_img_test_r_efficientnet_b0_cal_isotonic_clu_dt.pdf;%
  RegNet y\_400mf;%
  xp_nn_calibration/fig_imagenet_calibration/net_ILSVRC2012_img_test_r_regnet_y_400mf_cal_isotonic_clu_dt.pdf;%
  ViT B-16;%
  xp_nn_calibration/fig_imagenet_calibration/net_ILSVRC2012_img_test_r_visiontransformer_b_16_cal_isotonic_clu_dt.pdf;%
  ConvNeXt Tiny;%
  xp_nn_calibration/fig_imagenet_calibration/net_ILSVRC2012_img_test_r_convnext_tiny_cal_isotonic_clu_dt.pdf%
  }
  \caption{\textbf{Vision}: Fraction of correct predictions versus confidence score of predicted class ($\max_k S_k$) on ImageNet-R for small versions of pre-trained networks, with isotonic recalibration. In each bin on confidence scores, the level set is partitioned into 2 regions with a decision stump constrained to one balanced split, with a 50-50 train-test split strategy.}
  \label{fig:app:imagenet-r:cal:stump}
\end{figure}

\begin{figure}[p]
  \subsubsection{ImageNet-R – Isotonic recalibration, best versions}
  \label{sec:app:imagenet-r:cal:best}

  \diagramsbest{%
  VGG 19 BN;%
  xp_nn_calibration/fig_imagenet_calibration/net_ILSVRC2012_img_test_r_vgg_19_bn_cal_isotonic_clu_dt.pdf;%
  ResNet 152;%
  xp_nn_calibration/fig_imagenet_calibration/net_ILSVRC2012_img_test_r_resnet_152_cal_isotonic_clu_dt.pdf;%
  DenseNet 161;%
  xp_nn_calibration/fig_imagenet_calibration/net_ILSVRC2012_img_test_r_densenet_161_cal_isotonic_clu_dt.pdf;%
  ShuffleNet V2 x1.0;%
  xp_nn_calibration/fig_imagenet_calibration/net_ILSVRC2012_img_test_r_shufflenet_1_0_cal_isotonic_clu_dt.pdf;%
  MobileNet V3 Large;%
  xp_nn_calibration/fig_imagenet_calibration/net_ILSVRC2012_img_test_r_mobilenet_v3L_cal_isotonic_clu_dt.pdf;%
  ResNext 101;%
  xp_nn_calibration/fig_imagenet_calibration/net_ILSVRC2012_img_test_r_resnext_101_cal_isotonic_clu_dt.pdf;%
  WideResNet 101;%
  xp_nn_calibration/fig_imagenet_calibration/net_ILSVRC2012_img_test_r_wideresnet_101_cal_isotonic_clu_dt.pdf;%
  MNASNet 1.0;%
  xp_nn_calibration/fig_imagenet_calibration/net_ILSVRC2012_img_test_r_mnasnet_1_0_cal_isotonic_clu_dt.pdf;%
  EfficientNet B7;%
  xp_nn_calibration/fig_imagenet_calibration/net_ILSVRC2012_img_test_r_efficientnet_b7_cal_isotonic_clu_dt.pdf;%
  RegNet y\_32gf;%
  xp_nn_calibration/fig_imagenet_calibration/net_ILSVRC2012_img_test_r_regnet_y_32gf_cal_isotonic_clu_dt.pdf;%
  ViT L-16;%
  xp_nn_calibration/fig_imagenet_calibration/net_ILSVRC2012_img_test_r_visiontransformer_l_16_cal_isotonic_clu_dt.pdf;%
  ConvNeXt Large;%
  xp_nn_calibration/fig_imagenet_calibration/net_ILSVRC2012_img_test_r_convnext_large_cal_isotonic_clu_dt.pdf%
  }

  \caption{\textbf{Vision}: Fraction of correct predictions versus confidence score of predicted class ($\max_k S_k$) on ImageNet-R for best versions of pre-trained networks, with isotonic recalibration. In each bin on confidence scores, the level set is partitioned into 2 regions with a decision stump constrained to one balanced split, with a 50-50 train-test split strategy.}
  \label{fig:app:imagenet-r:cal:stump:best}
  \null
  \vfill
\end{figure}

\stepcounter{subsection}
\begin{figure}[p]
  \addcontentsline{toc}{subsection}{ImageNet-C}
  \subsubsection{ImageNet-C – No post-hoc recalibration, small versions}
  \label{sec:app:imagenet-c:nocal}

  \diagrams{%
  AlexNet;%
  xp_nn_calibration/fig_imagenet_calibration/net_ILSVRC2012_img_test_c__merged_no_rep5_alexnet_cal_None_clu_dt.pdf;
  VGG 11;%
  xp_nn_calibration/fig_imagenet_calibration/net_ILSVRC2012_img_test_c__merged_no_rep5_vgg_11_cal_None_clu_dt.pdf;%
  ResNet 18;%
  xp_nn_calibration/fig_imagenet_calibration/net_ILSVRC2012_img_test_c__merged_no_rep5_resnet_18_cal_None_clu_dt.pdf;%
  DenseNet 121;%
  xp_nn_calibration/fig_imagenet_calibration/net_ILSVRC2012_img_test_c__merged_no_rep5_densenet_121_cal_None_clu_dt.pdf;%
  Inception;%
  xp_nn_calibration/fig_imagenet_calibration/net_ILSVRC2012_img_test_c__merged_no_rep5_inception_cal_None_clu_dt.pdf;%
  GoogleNet;%
  xp_nn_calibration/fig_imagenet_calibration/net_ILSVRC2012_img_test_c__merged_no_rep5_googlenet_cal_None_clu_dt.pdf;%
  ShuffleNet V2 x0.5;%
  xp_nn_calibration/fig_imagenet_calibration/net_ILSVRC2012_img_test_c__merged_no_rep5_shufflenet_0_5_cal_None_clu_dt.pdf;%
  MobileNet V2;%
  xp_nn_calibration/fig_imagenet_calibration/net_ILSVRC2012_img_test_c__merged_no_rep5_mobilenet_v2_cal_None_clu_dt.pdf;%
  ResNext 50;%
  xp_nn_calibration/fig_imagenet_calibration/net_ILSVRC2012_img_test_c__merged_no_rep5_resnext_50_cal_None_clu_dt.pdf;%
  WideResNet 50;%
  xp_nn_calibration/fig_imagenet_calibration/net_ILSVRC2012_img_test_c__merged_no_rep5_wideresnet_50_cal_None_clu_dt.pdf;%
  MNASNet 0.5;%
  xp_nn_calibration/fig_imagenet_calibration/net_ILSVRC2012_img_test_c__merged_no_rep5_mnasnet_0_5_cal_None_clu_dt.pdf;%
  EfficientNet B0;%
  xp_nn_calibration/fig_imagenet_calibration/net_ILSVRC2012_img_test_c__merged_no_rep5_efficientnet_b0_cal_None_clu_dt.pdf;%
  RegNet y\_400mf;%
  xp_nn_calibration/fig_imagenet_calibration/net_ILSVRC2012_img_test_c__merged_no_rep5_regnet_y_400mf_cal_None_clu_dt.pdf;%
  ViT B-16;%
  xp_nn_calibration/fig_imagenet_calibration/net_ILSVRC2012_img_test_c__merged_no_rep5_visiontransformer_b_16_cal_None_clu_dt.pdf;%
  ConvNeXt Tiny;%
  xp_nn_calibration/fig_imagenet_calibration/net_ILSVRC2012_img_test_c__merged_no_rep5_convnext_tiny_cal_None_clu_dt.pdf%
  }

  \caption{\textbf{Vision}: Fraction of correct predictions versus
confidence score of predicted class ($\max_k S_k$) on ImageNet-C for
small versions of pre-trained networks, without post-hoc recalibration. In each bin on confidence scores, the level set is partitioned into 2 regions with a decision stump constrained to one balanced split, with a 50-50 train-test split strategy.}
  \label{fig:app:imagenet-c:nocal:stump}
\end{figure}

\begin{figure}[p]
  \subsubsection{ImageNet-C – No post-hoc recalibration, best versions}
  \label{sec:app:imagenet-c:nocal:best}

  \diagramsbest{%
  VGG 19 BN;%
  xp_nn_calibration/fig_imagenet_calibration/net_ILSVRC2012_img_test_c__merged_no_rep5_vgg_19_bn_cal_None_clu_dt.pdf;%
  ResNet 152;%
  xp_nn_calibration/fig_imagenet_calibration/net_ILSVRC2012_img_test_c__merged_no_rep5_resnet_152_cal_None_clu_dt.pdf;%
  DenseNet 161;%
  xp_nn_calibration/fig_imagenet_calibration/net_ILSVRC2012_img_test_c__merged_no_rep5_densenet_161_cal_None_clu_dt.pdf;%
  ShuffleNet V2 x1.0;%
  xp_nn_calibration/fig_imagenet_calibration/net_ILSVRC2012_img_test_c__merged_no_rep5_shufflenet_1_0_cal_None_clu_dt.pdf;%
  MobileNet V3 Large;%
  xp_nn_calibration/fig_imagenet_calibration/net_ILSVRC2012_img_test_c__merged_no_rep5_mobilenet_v3L_cal_None_clu_dt.pdf;%
  ResNext 101;%
  xp_nn_calibration/fig_imagenet_calibration/net_ILSVRC2012_img_test_c__merged_no_rep5_resnext_101_cal_None_clu_dt.pdf;%
  WideResNet 101;%
  xp_nn_calibration/fig_imagenet_calibration/net_ILSVRC2012_img_test_c__merged_no_rep5_wideresnet_101_cal_None_clu_dt.pdf;%
  MNASNet 1.0;%
  xp_nn_calibration/fig_imagenet_calibration/net_ILSVRC2012_img_test_c__merged_no_rep5_mnasnet_1_0_cal_None_clu_dt.pdf;%
  EfficientNet B7;%
  xp_nn_calibration/fig_imagenet_calibration/net_ILSVRC2012_img_test_c__merged_no_rep5_efficientnet_b7_cal_None_clu_dt.pdf;%
  RegNet y\_32gf;%
  xp_nn_calibration/fig_imagenet_calibration/net_ILSVRC2012_img_test_c__merged_no_rep5_regnet_y_32gf_cal_None_clu_dt.pdf;%
  ViT L-16;%
  xp_nn_calibration/fig_imagenet_calibration/net_ILSVRC2012_img_test_c__merged_no_rep5_visiontransformer_l_16_cal_None_clu_dt.pdf;%
  ConvNeXt Large;%
  xp_nn_calibration/fig_imagenet_calibration/net_ILSVRC2012_img_test_c__merged_no_rep5_convnext_large_cal_None_clu_dt.pdf%
  }

  \caption{\textbf{Vision}: Fraction of correct predictions versus
confidence score of predicted class ($\max_k S_k$) on ImageNet-C for best
versions of pre-trained networks, without post-hoc recalibration. In each bin on confidence scores, the level set is partitioned into 2 regions with a decision stump constrained to one balanced split, with a 50-50 train-test split strategy.}
  \label{fig:app:imagenet-c:nocal:stump:best}
  \null
  \vfill
\end{figure}

\begin{figure}[p]
  \subsubsection{ImageNet-C – Isotonic recalibration, small versions}
  \label{sec:app:imagenet-c:cal}

  \diagrams{%
  AlexNet;%
  xp_nn_calibration/fig_imagenet_calibration/net_ILSVRC2012_img_test_c__merged_no_rep5_alexnet_cal_isotonic_clu_dt.pdf;
  VGG 11;%
  xp_nn_calibration/fig_imagenet_calibration/net_ILSVRC2012_img_test_c__merged_no_rep5_vgg_11_cal_isotonic_clu_dt.pdf;%
  ResNet 18;%
  xp_nn_calibration/fig_imagenet_calibration/net_ILSVRC2012_img_test_c__merged_no_rep5_resnet_18_cal_isotonic_clu_dt.pdf;%
  DenseNet 121;%
  xp_nn_calibration/fig_imagenet_calibration/net_ILSVRC2012_img_test_c__merged_no_rep5_densenet_121_cal_isotonic_clu_dt.pdf;%
  Inception;%
  xp_nn_calibration/fig_imagenet_calibration/net_ILSVRC2012_img_test_c__merged_no_rep5_inception_cal_isotonic_clu_dt.pdf;%
  GoogleNet;%
  xp_nn_calibration/fig_imagenet_calibration/net_ILSVRC2012_img_test_c__merged_no_rep5_googlenet_cal_isotonic_clu_dt.pdf;%
  ShuffleNet V2 x0.5;%
  xp_nn_calibration/fig_imagenet_calibration/net_ILSVRC2012_img_test_c__merged_no_rep5_shufflenet_0_5_cal_isotonic_clu_dt.pdf;%
  MobileNet V2;%
  xp_nn_calibration/fig_imagenet_calibration/net_ILSVRC2012_img_test_c__merged_no_rep5_mobilenet_v2_cal_isotonic_clu_dt.pdf;%
  ResNext 50;%
  xp_nn_calibration/fig_imagenet_calibration/net_ILSVRC2012_img_test_c__merged_no_rep5_resnext_50_cal_isotonic_clu_dt.pdf;%
  WideResNet 50;%
  xp_nn_calibration/fig_imagenet_calibration/net_ILSVRC2012_img_test_c__merged_no_rep5_wideresnet_50_cal_isotonic_clu_dt.pdf;%
  MNASNet 0.5;%
  xp_nn_calibration/fig_imagenet_calibration/net_ILSVRC2012_img_test_c__merged_no_rep5_mnasnet_0_5_cal_isotonic_clu_dt.pdf;%
  EfficientNet B0;%
  xp_nn_calibration/fig_imagenet_calibration/net_ILSVRC2012_img_test_c__merged_no_rep5_efficientnet_b0_cal_isotonic_clu_dt.pdf;%
  RegNet y\_400mf;%
  xp_nn_calibration/fig_imagenet_calibration/net_ILSVRC2012_img_test_c__merged_no_rep5_regnet_y_400mf_cal_isotonic_clu_dt.pdf;%
  ViT B-16;%
  xp_nn_calibration/fig_imagenet_calibration/net_ILSVRC2012_img_test_c__merged_no_rep5_visiontransformer_b_16_cal_isotonic_clu_dt.pdf;%
  ConvNeXt Tiny;%
  xp_nn_calibration/fig_imagenet_calibration/net_ILSVRC2012_img_test_c__merged_no_rep5_convnext_tiny_cal_isotonic_clu_dt.pdf%
  }
  \caption{\textbf{Vision}: Fraction of correct predictions versus confidence score of predicted class ($\max_k S_k$) on ImageNet-C for small versions of pre-trained networks, with isotonic recalibration. In each bin on confidence scores, the level set is partitioned into 2 regions with a decision stump constrained to one balanced split, with a 50-50 train-test split strategy.}
  \label{fig:app:imagenet-c:cal:stump}
\end{figure}

\begin{figure}[p]
  \subsubsection{ImageNet-C – Isotonic recalibration, best versions}
  \label{sec:app:imagenet-c:cal:best}

  \diagramsbest{%
  VGG 19 BN;%
  xp_nn_calibration/fig_imagenet_calibration/net_ILSVRC2012_img_test_c__merged_no_rep5_vgg_19_bn_cal_isotonic_clu_dt.pdf;%
  ResNet 152;%
  xp_nn_calibration/fig_imagenet_calibration/net_ILSVRC2012_img_test_c__merged_no_rep5_resnet_152_cal_isotonic_clu_dt.pdf;%
  DenseNet 161;%
  xp_nn_calibration/fig_imagenet_calibration/net_ILSVRC2012_img_test_c__merged_no_rep5_densenet_161_cal_isotonic_clu_dt.pdf;%
  ShuffleNet V2 x1.0;%
  xp_nn_calibration/fig_imagenet_calibration/net_ILSVRC2012_img_test_c__merged_no_rep5_shufflenet_1_0_cal_isotonic_clu_dt.pdf;%
  MobileNet V3 Large;%
  xp_nn_calibration/fig_imagenet_calibration/net_ILSVRC2012_img_test_c__merged_no_rep5_mobilenet_v3L_cal_isotonic_clu_dt.pdf;%
  ResNext 101;%
  xp_nn_calibration/fig_imagenet_calibration/net_ILSVRC2012_img_test_c__merged_no_rep5_resnext_101_cal_isotonic_clu_dt.pdf;%
  WideResNet 101;%
  xp_nn_calibration/fig_imagenet_calibration/net_ILSVRC2012_img_test_c__merged_no_rep5_wideresnet_101_cal_isotonic_clu_dt.pdf;%
  MNASNet 1.0;%
  xp_nn_calibration/fig_imagenet_calibration/net_ILSVRC2012_img_test_c__merged_no_rep5_mnasnet_1_0_cal_isotonic_clu_dt.pdf;%
  EfficientNet B7;%
  xp_nn_calibration/fig_imagenet_calibration/net_ILSVRC2012_img_test_c__merged_no_rep5_efficientnet_b7_cal_isotonic_clu_dt.pdf;%
  RegNet y\_32gf;%
  xp_nn_calibration/fig_imagenet_calibration/net_ILSVRC2012_img_test_c__merged_no_rep5_regnet_y_32gf_cal_isotonic_clu_dt.pdf;%
  ViT L-16;%
  xp_nn_calibration/fig_imagenet_calibration/net_ILSVRC2012_img_test_c__merged_no_rep5_visiontransformer_l_16_cal_isotonic_clu_dt.pdf;%
  ConvNeXt Large;%
  xp_nn_calibration/fig_imagenet_calibration/net_ILSVRC2012_img_test_c__merged_no_rep5_convnext_large_cal_isotonic_clu_dt.pdf%
  }

  \caption{\textbf{Vision}: Fraction of correct predictions versus confidence score of predicted class ($\max_k S_k$) on ImageNet-C for best versions of pre-trained networks, with isotonic recalibration. In each bin on confidence scores, the level set is partitioned into 2 regions with a decision stump constrained to one balanced split, with a 50-50 train-test split strategy.}
  \label{fig:app:imagenet-c:cal:stump:best}
  \null
  \vfill
\end{figure}

\stepcounter{subsection}
\begin{figure}[p]
  \subsubsection{ImageNet-1K – No post-hoc recalibration, small versions}
  \label{sec:app:imagenet-1k:nocal}

  \diagrams{%
  AlexNet;%
  xp_nn_calibration/fig_imagenet_calibration/net_ILSVRC2012_img_val_alexnet_cal_None_clu_dt.pdf;
  VGG 11;%
  xp_nn_calibration/fig_imagenet_calibration/net_ILSVRC2012_img_val_vgg_11_cal_None_clu_dt.pdf;%
  ResNet 18;%
  xp_nn_calibration/fig_imagenet_calibration/net_ILSVRC2012_img_val_resnet_18_cal_None_clu_dt.pdf;%
  DenseNet 121;%
  xp_nn_calibration/fig_imagenet_calibration/net_ILSVRC2012_img_val_densenet_121_cal_None_clu_dt.pdf;%
  Inception;%
  xp_nn_calibration/fig_imagenet_calibration/net_ILSVRC2012_img_val_inception_cal_None_clu_dt.pdf;%
  GoogleNet;%
  xp_nn_calibration/fig_imagenet_calibration/net_ILSVRC2012_img_val_googlenet_cal_None_clu_dt.pdf;%
  ShuffleNet V2 x0.5;%
  xp_nn_calibration/fig_imagenet_calibration/net_ILSVRC2012_img_val_shufflenet_0_5_cal_None_clu_dt.pdf;%
  MobileNet V2;%
  xp_nn_calibration/fig_imagenet_calibration/net_ILSVRC2012_img_val_mobilenet_v2_cal_None_clu_dt.pdf;%
  ResNext 50;%
  xp_nn_calibration/fig_imagenet_calibration/net_ILSVRC2012_img_val_resnext_50_cal_None_clu_dt.pdf;%
  WideResNet 50;%
  xp_nn_calibration/fig_imagenet_calibration/net_ILSVRC2012_img_val_wideresnet_50_cal_None_clu_dt.pdf;%
  MNASNet 0.5;%
  xp_nn_calibration/fig_imagenet_calibration/net_ILSVRC2012_img_val_mnasnet_0_5_cal_None_clu_dt.pdf;%
  EfficientNet B0;%
  xp_nn_calibration/fig_imagenet_calibration/net_ILSVRC2012_img_val_efficientnet_b0_cal_None_clu_dt.pdf;%
  RegNet y\_400mf;%
  xp_nn_calibration/fig_imagenet_calibration/net_ILSVRC2012_img_val_regnet_y_400mf_cal_None_clu_dt.pdf;%
  ViT B-16;%
  xp_nn_calibration/fig_imagenet_calibration/net_ILSVRC2012_img_val_visiontransformer_b_16_cal_None_clu_dt.pdf;%
  ConvNeXt Tiny;%
  xp_nn_calibration/fig_imagenet_calibration/net_ILSVRC2012_img_val_convnext_tiny_cal_None_clu_dt.pdf%
  }

  \caption{\textbf{Vision}: Fraction of correct predictions versus
confidence score of predicted class ($\max_k S_k$) on ImageNet-1K
(validation set) for small versions of pre-trained networks, without
post-hoc recalibration. In each bin on confidence scores, the level set is partitioned into 2 regions with a decision stump constrained to one balanced split, with a 50-50 train-test split strategy.}
  \label{fig:app:imagenet-1k:nocal:stump}
\end{figure}

\begin{figure}[p]
  \subsubsection{ImageNet-1K – No post-hoc recalibration, best versions}
  \label{sec:app:imagenet-1k:nocal:best}

  \diagramsbest{%
  VGG 19 BN;%
  xp_nn_calibration/fig_imagenet_calibration/net_ILSVRC2012_img_val_vgg_19_bn_cal_None_clu_dt.pdf;%
  ResNet 152;%
  xp_nn_calibration/fig_imagenet_calibration/net_ILSVRC2012_img_val_resnet_152_cal_None_clu_dt.pdf;%
  DenseNet 161;%
  xp_nn_calibration/fig_imagenet_calibration/net_ILSVRC2012_img_val_densenet_161_cal_None_clu_dt.pdf;%
  ShuffleNet V2 x1.0;%
  xp_nn_calibration/fig_imagenet_calibration/net_ILSVRC2012_img_val_shufflenet_1_0_cal_None_clu_dt.pdf;%
  MobileNet V3 Large;%
  xp_nn_calibration/fig_imagenet_calibration/net_ILSVRC2012_img_val_mobilenet_v3L_cal_None_clu_dt.pdf;%
  ResNext 101;%
  xp_nn_calibration/fig_imagenet_calibration/net_ILSVRC2012_img_val_resnext_101_cal_None_clu_dt.pdf;%
  WideResNet 101;%
  xp_nn_calibration/fig_imagenet_calibration/net_ILSVRC2012_img_val_wideresnet_101_cal_None_clu_dt.pdf;%
  MNASNet 1.0;%
  xp_nn_calibration/fig_imagenet_calibration/net_ILSVRC2012_img_val_mnasnet_1_0_cal_None_clu_dt.pdf;%
  EfficientNet B7;%
  xp_nn_calibration/fig_imagenet_calibration/net_ILSVRC2012_img_val_efficientnet_b7_cal_None_clu_dt.pdf;%
  RegNet y\_32gf;%
  xp_nn_calibration/fig_imagenet_calibration/net_ILSVRC2012_img_val_regnet_y_32gf_cal_None_clu_dt.pdf;%
  ViT L-16;%
  xp_nn_calibration/fig_imagenet_calibration/net_ILSVRC2012_img_val_visiontransformer_l_16_cal_None_clu_dt.pdf;%
  ConvNeXt Large;%
  xp_nn_calibration/fig_imagenet_calibration/net_ILSVRC2012_img_val_convnext_large_cal_None_clu_dt.pdf%
  }

  \caption{\textbf{Vision}: Fraction of correct predictions versus
confidence score of predicted class ($\max_k S_k$) on ImageNet-1K
(validation set) for best versions of pre-trained networks, without
post-hoc recalibration. In each bin on confidence scores, the level set is partitioned into 2 regions with a decision stump constrained to one balanced split, with a 50-50 train-test split strategy.}
  \label{fig:app:imagenet-1k:nocal:stump:best}
  \null
  \vfill
\end{figure}

\begin{figure}[p]
  \addcontentsline{toc}{subsection}{ImageNet-1K}
  \subsubsection{ImageNet-1K – Isotonic recalibration, small versions}
  \label{sec:app:imagenet-1k:cal}

  \diagrams{%
  AlexNet;%
  xp_nn_calibration/fig_imagenet_calibration/net_ILSVRC2012_img_val_alexnet_cal_isotonic_clu_dt.pdf;
  VGG 11;%
  xp_nn_calibration/fig_imagenet_calibration/net_ILSVRC2012_img_val_vgg_11_cal_isotonic_clu_dt.pdf;%
  ResNet 18;%
  xp_nn_calibration/fig_imagenet_calibration/net_ILSVRC2012_img_val_resnet_18_cal_isotonic_clu_dt.pdf;%
  DenseNet 121;%
  xp_nn_calibration/fig_imagenet_calibration/net_ILSVRC2012_img_val_densenet_121_cal_isotonic_clu_dt.pdf;%
  Inception;%
  xp_nn_calibration/fig_imagenet_calibration/net_ILSVRC2012_img_val_inception_cal_isotonic_clu_dt.pdf;%
  GoogleNet;%
  xp_nn_calibration/fig_imagenet_calibration/net_ILSVRC2012_img_val_googlenet_cal_isotonic_clu_dt.pdf;%
  ShuffleNet V2 x0.5;%
  xp_nn_calibration/fig_imagenet_calibration/net_ILSVRC2012_img_val_shufflenet_0_5_cal_isotonic_clu_dt.pdf;%
  MobileNet V2;%
  xp_nn_calibration/fig_imagenet_calibration/net_ILSVRC2012_img_val_mobilenet_v2_cal_isotonic_clu_dt.pdf;%
  ResNext 50;%
  xp_nn_calibration/fig_imagenet_calibration/net_ILSVRC2012_img_val_resnext_50_cal_isotonic_clu_dt.pdf;%
  WideResNet 50;%
  xp_nn_calibration/fig_imagenet_calibration/net_ILSVRC2012_img_val_wideresnet_50_cal_isotonic_clu_dt.pdf;%
  MNASNet 0.5;%
  xp_nn_calibration/fig_imagenet_calibration/net_ILSVRC2012_img_val_mnasnet_0_5_cal_isotonic_clu_dt.pdf;%
  EfficientNet B0;%
  xp_nn_calibration/fig_imagenet_calibration/net_ILSVRC2012_img_val_efficientnet_b0_cal_isotonic_clu_dt.pdf;%
  RegNet y\_400mf;%
  xp_nn_calibration/fig_imagenet_calibration/net_ILSVRC2012_img_val_regnet_y_400mf_cal_isotonic_clu_dt.pdf;%
  ViT B-16;%
  xp_nn_calibration/fig_imagenet_calibration/net_ILSVRC2012_img_val_visiontransformer_b_16_cal_isotonic_clu_dt.pdf;%
  ConvNeXt Tiny;%
  xp_nn_calibration/fig_imagenet_calibration/net_ILSVRC2012_img_val_convnext_tiny_cal_isotonic_clu_dt.pdf%
  }
  \caption{\textbf{Vision}: Fraction of correct predictions versus confidence score of predicted class ($\max_k S_k$) on ImageNet-1K (validation set) for small versions of pre-trained networks, with isotonic recalibration. In each bin on confidence scores, the level set is partitioned into 2 regions with a decision stump constrained to one balanced split, with a 50-50 train-test split strategy.}
  \label{fig:app:imagenet-1k:cal:stump}
\end{figure}

\begin{figure}[p]
  \subsubsection{ImageNet-1K – Isotonic recalibration, best versions}
  \label{sec:app:imagenet-1k:cal:best}

  \diagramsbest{%
  VGG 19 BN;%
  xp_nn_calibration/fig_imagenet_calibration/net_ILSVRC2012_img_val_vgg_19_bn_cal_isotonic_clu_dt.pdf;%
  ResNet 152;%
  xp_nn_calibration/fig_imagenet_calibration/net_ILSVRC2012_img_val_resnet_152_cal_isotonic_clu_dt.pdf;%
  DenseNet 161;%
  xp_nn_calibration/fig_imagenet_calibration/net_ILSVRC2012_img_val_densenet_161_cal_isotonic_clu_dt.pdf;%
  ShuffleNet V2 x1.0;%
  xp_nn_calibration/fig_imagenet_calibration/net_ILSVRC2012_img_val_shufflenet_1_0_cal_isotonic_clu_dt.pdf;%
  MobileNet V3 Large;%
  xp_nn_calibration/fig_imagenet_calibration/net_ILSVRC2012_img_val_mobilenet_v3L_cal_isotonic_clu_dt.pdf;%
  ResNext 101;%
  xp_nn_calibration/fig_imagenet_calibration/net_ILSVRC2012_img_val_resnext_101_cal_isotonic_clu_dt.pdf;%
  WideResNet 101;%
  xp_nn_calibration/fig_imagenet_calibration/net_ILSVRC2012_img_val_wideresnet_101_cal_isotonic_clu_dt.pdf;%
  MNASNet 1.0;%
  xp_nn_calibration/fig_imagenet_calibration/net_ILSVRC2012_img_val_mnasnet_1_0_cal_isotonic_clu_dt.pdf;%
  EfficientNet B7;%
  xp_nn_calibration/fig_imagenet_calibration/net_ILSVRC2012_img_val_efficientnet_b7_cal_isotonic_clu_dt.pdf;%
  RegNet y\_32gf;%
  xp_nn_calibration/fig_imagenet_calibration/net_ILSVRC2012_img_val_regnet_y_32gf_cal_isotonic_clu_dt.pdf;%
  ViT L-16;%
  xp_nn_calibration/fig_imagenet_calibration/net_ILSVRC2012_img_val_visiontransformer_l_16_cal_isotonic_clu_dt.pdf;%
  ConvNeXt Large;%
  xp_nn_calibration/fig_imagenet_calibration/net_ILSVRC2012_img_val_convnext_large_cal_isotonic_clu_dt.pdf%
  }

  \caption{\textbf{Vision}: Fraction of correct predictions versus confidence score of predicted class ($\max_k S_k$) on ImageNet-1K (validation set) for best versions of pre-trained networks, with isotonic recalibration. In each bin on confidence scores, the level set is partitioned into 2 regions with a decision stump constrained to one balanced split, with a 50-50 train-test split strategy.}
  \label{fig:app:imagenet-1k:cal:stump:best}
  \null
  \vfill
\end{figure}

\end{document}